\RequirePackage{fix-cm}
\documentclass[smallextended]{svjour3}       % onecolumn (second format)
\smartqed  % flush right qed marks, e.g. at end of proof
\usepackage{graphicx}
\usepackage{graphicx}
\usepackage{natbib}
\usepackage{hyperref}       % hyperlinks
\usepackage{url}            % simple URL typesetting
\usepackage{amsfonts}       % blackboard math symbols
\usepackage{tabularx}
\usepackage{amsmath}

\usepackage{amsthm}
\usepackage{blindtext}
\usepackage{enumitem}
\usepackage{multirow}
\usepackage{graphicx}
\usepackage{algorithm}% http://ctan.org/pkg/algorithms
\usepackage{algpseudocode}% http://ctan.org/pkg/algorithmicx
\usepackage{amsfonts}
\usepackage{xcolor}
\usepackage{booktabs}       % professional-quality tables
\usepackage{amsfonts}       % blackboard math symbols
\usepackage{nicefrac}       % compact symbols for 1/2, etc.
\usepackage{microtype}      % microtypography
\usepackage{adjustbox}
\usepackage{tabularx}
\usepackage{array}
\usepackage{amsmath}
\usepackage{blindtext}
\usepackage{amsfonts}       % blackboard math symbols
\usepackage{nicefrac}       % compact symbols for 1/2, etc.
\usepackage{microtype}      % microtypography
\usepackage{adjustbox}
\usepackage{tabularx}
\usepackage{array}
\usepackage{amsmath}
\usepackage{amsthm}
\usepackage{blindtext}
\usepackage{mathtools}
\usepackage{enumitem}
\usepackage{enumitem}
\usepackage{algorithm}% http://ctan.org/pkg/algorithms
\usepackage{algpseudocode}% http://ctan.org/pkg/algorithmicx
\usepackage{amsfonts}
\usepackage{comment}
\theoremstyle{definition}
\newtheorem{assumption}{Assumption}
\usepackage{subfigure}
\graphicspath{{figure/}}
  
\usepackage{url}
\begin{document}

\title{Differentially Private Expectation Maximization  Algorithm with Statistical Guarantees %\thanks{Grants or other notes
}

\titlerunning{DP EM Algorithm}        % if too long for running head

\author{Di Wang  \and
        Jiahao Ding  \footnote{The first two authors contributed equally.}   \and 
        Lijie Hu \and 
        Zejun Xie\and 
        Miao Pan \and 
        Jinhui Xu %etc. 
}


\institute{Di Wang \at
                  Division of Computer, Electrical and Mathematical Sciences and Engineering\\
             King Abdullah University of Science and Technology \\
            Thuwal, Saudi Arabia, 23955-6900 \\ 
              \email{di.wang@kaust.edu.sa}           \\
                          Corresponding author%  
           \and
           Jiahao Ding \at
            Department of  Electrical and Computer Engineering\\
    University of Houston\\
             \email{jding7@uh.edu}
             \and 
             Lijie Hu \at
                  Division of Computer, Electrical and Mathematical Sciences and Engineering\\
             King Abdullah University of Science and Technology \\
            Thuwal, Saudi Arabia, 23955-6900 \\ 
              \email{lijie.hu@kaust.edu.sa}         
            \and 
            Zejun Xie \at 
            Department of Computer Science\\ 
Rutgers University \\
                      \email{scoji1751@gmail.com}   
             \and 
             Miao Pan \at 
                         Department of  Electrical and Computer Engineering\\
    University of Houston\\
             \email{mpan2@uh.edu}
                    \and 
             Jinhui Xu \at 
                Department of Computer Science and Engineering\\
             State University of New York at Buffalo\\
               \email{jinhui@buffalo.edu}   \\
}

\date{Received: date / Accepted: date}

\maketitle
\begin{abstract}
    (Gradient) Expectation Maximization (EM) is a widely used algorithm for estimating the maximum likelihood of mixture models or incomplete data problems. A major challenge facing this popular technique is how to effectively preserve the privacy of sensitive data.
Previous research on this problem has already lead to the discovery of some Differentially Private (DP) algorithms for (Gradient) EM. However, unlike in the non-private case, existing techniques are not yet able to provide  
finite sample statistical guarantees. To address this issue,  we propose in this paper the first DP version of Gradient EM algorithm with statistical guarantees. Specifically, we first propose a new mechanism for privately estimating the mean of a heavy-tailed distribution, which significantly improves a previous result in \citep{wangicml2020}, and it could be extended to the local DP model, which has not been studied before. Next, we apply our general framework to three canonical models: Gaussian Mixture Model (GMM), Mixture of Regressions Model (MRM) and Linear Regression with Missing Covariates (RMC). Specifically, for GMM in the DP model, our estimation error is near optimal in some cases. For the other two models, we provide the first result on finite sample statistical guarantees.  Our theory is supported by thorough numerical experiments on both real-world data and synthetic data. 

\end{abstract}

     % basic 
\section{Introduction}

%\vspace{-0.2in}
As one of the most popular techniques for estimating the maximum likelihood
%estimator 
of mixture models or incomplete data problems, Expectation Maximization (EM) algorithm has been  widely applied to many areas such as genomics  \citep{laird2010algorithm}, finance \citep{faria2013financial}, and crowdsourcing \citep{dawid1979maximum}.  EM algorithm is well-known for its convergence to an empirically good local estimator \citep{wu1983convergence}. Recent studies have further revealed that 
%But only until recently, it was shown that 
it can also provide finite sample statistical guarantees
%for its performance have not been established until recent studies 
\citep{balakrishnan2017statistical,zhu2017high,wang2015high,yi2015regularized}. Specifically, \citep{balakrishnan2017statistical} showed that  classical EM and its gradient ascent variant (Gradient EM) are capable of achieving the first local convergence (theory) and finite sample statistical rate of convergence.
%for the classical EM and its gradient ascent variant (gradient EM) were established in \citep{balakrishnan2017statistical}. 
They also provided a (near) optimal minimax rate for some canonical statistical models such as Gaussian mixture model (GMM), mixture of regressions model (MRM) and linear regression with missing covariates (RMC). 

The wide applications of EM also present some new challenges to this method. Particularly, due to the existence of sensitive data and their  distributed nature in many applications like social science, biomedicine, and genomics, it is often challenging to preserve the privacy of such data as they are  
%On the other side, due to the presence of sensitive
%data (especially those in social science, biomedicine and genomics) and %their distributed nature, such data are 
extremely difficult to aggregate and learn from. Consider a case where
health records are scattered across multiple hospitals (or even
countries), it is not possible to process the whole dataset in a
central server due to privacy and ownership concerns. A better solution is to use some differentially private mechanisms to
conduct the aggregation and learning tasks. Differential Privacy (DP) \citep{dwork2006calibrating} is a commonly-accepted criterion that provides provable protection against identification
and is resilient to arbitrary auxiliary information that might be
available to attackers. 

Thus, to be able to use (Gradient) EM algorithm to learn from these sensitive data, it is urgent to design some DP versions of the (gradient) EM algorithm. \citep{park2017dp}  proposed the first DP EM algorithm which 
mainly focuses on the practical behaviors of the method. Their algorithm needs quite a few 
%lots of 
assumptions on the model and the data, which make it difficult to extend to some canonical models mentioned above. Furthermore, unlike the aforementioned non-private case, 
%as we mentioned above, 
their algorithm does not provide  
%there is no 
any finite sample statistical guarantee on the solution
%on their algorithm 
(see Related Work section for detailed comparison). Thus, it is still unknown \textbf{whether there exists any DP variant of  the (gradient) EM algorithm  that has finite sample statistical guarantees}. 
%as in the non-private case}.

To answer this question,
%address the aforementioned issue, 
 we propose in this paper the first  $(\epsilon, \delta)$-DP 
 %variant of 
 (Gradient) EM algorithm with finite sample statistical guarantees. Specifically, 
\begin{itemize}
    \item We first show that, given an appropriate initialization 
    $\beta^{\text{init}}$ ({\em i.e.,} $\|\beta^{\text{init}}-\beta^*\|_2 \leq \kappa \|\beta^*\|_2$ for some constant $\kappa\in (0,1)$), if the model satisfies some additional assumptions and the number of sample $n$ is large enough, the output $\beta^{priv}$ of our DP (Gradient) EM algorithm is guaranteed to have a bounded estimation error, $\|\beta^{priv}-\beta^*\|_2 \leq \tilde{O}(\frac{d\sqrt{\tau}}{\sqrt{n\epsilon}})$, with high probability, %in we omitted other factors, 
    where $d$ is the dimensionality and $\tau$ is an upper bound of the second-order moment of  each coordinate of the gradient function.  To get the result, we propose a new mechanism for privately estimating the mean of a heavy-tailed distribution, which is based on a finer analysis of the mechanism given by \citep{wangicml2020}. Moreover, our mechanism could be easily extended to the local privacy model, which is the first result on the problem. Thus, we believe our mechanism could be used in other machine learning problems. 
    \item We then apply our general framework to  the three canonical models: GMM, MRM and RMC. Our private estimator achieves an estimation error that is upper bounded by 
    %Specifically, the estimation error of our private estimator is upper bounded by 
    $\tilde{O}(\frac{d}{\sqrt{n\epsilon}})$, $\tilde{O}(\frac{d^{\frac{3}{2}}}{\sqrt{n\epsilon}})$ and $\tilde{O}(\frac{d^{\frac{3}{2}}}{\sqrt{n\epsilon}})$ for GMM, MRM and RMC, respectively.  We note that they are the first statistical guarantees for MRM and RMC in the Differential Privacy model, and the error bound for GMM is near optimal in some cases. We also conduct thorough experiments on the these three models.  Experimental results on these models are consistent with our theoretical analysis. 
\end{itemize} %\vspace{-0.1in}

\section{Related Work}
%\vspace{-0.2in}
As we mentioned previously, designing DP version of EM algorithm is still not well studied. To our best knowledge, the only work on DP EM algorithm is given by \citep{park2017dp}. However, their result is incomparable with ours for the following reasons. Firstly, our work aims to achieve finite sample statistical guarantees for the DP EM algorithm, while  \citep{park2017dp} mainly focuses on designing practical DP EM algorithm that does not provide any statistical guarantees. % their algorithm. 
Particularly, \citep{park2017dp} assumed that datasets are pre-processed such that the  $\ell_2$-norm of each data record is less than $1$. This means that their algorithm will likely introduce additional bias on the statistical guarantees.  Secondly, \citep{park2017dp} studied only the exponential family so that noise can be directly added to the sufficient statistics. However, most of the latent variable models  do not satisfy such an assumption. This includes  the MRM and RMC models to be considered in this paper. 

In this paper, we implement our general framework on three specific models, and DP GMM is the only one that has been studied previously. Specifically, \citep{nissim2007smooth} provided the first result for the general $k$-GMM based on the sample-and-aggregate framework. Later on, \citep{kamath2019differentially}  improved the result by a factor of $\sqrt{d}/\epsilon$, and also claimed that their sample complexity is near optimal. Compared with their result, our proposed algorithm ensures that  when $\epsilon$ is some constant,  it has the same sample complexity. Also, although their algorithm has polynomial time complexity, it is actually not very practical and thus no practical study has been conducted. Moreover, their algorithm is heavily dependent on a previous clustering algorithm; it is unclear whether it can be extended to other mixture models. From these two perspectives, our framework is more general and practical.
%\vspace{-0.2in}
\section{Preliminaries}
%\vspace{-0.15in}

Let $Y$ and $Z$ be two random variables taking values in the sample spaces $\mathcal{Y}$ and $\mathcal{Z}$, respectively. Suppose that the pair $(Y, Z)$ has a joint density function $f_{\beta^*}$ that belongs to some parameterized family $\{f_{\beta^*}|\beta^* \in \Omega \}$. Rather than considering the whole pair of $(Y, Z)$, we observe only component $Y$. Thus,  component $Z$ can be viewed as the missing or latent structure. We assume that the term $h_{\beta}(y)$ is the margin distribution over the latent variable $Z$, {\em i.e.,} $h_\beta(y)=\int_{\mathcal{Z}} f_\beta(y, z) dz.$ Let $k_{\beta}(z|y)$ be the density of $Z$ conditional on the observed variable $Y=y$, that is, $k_\beta(z|y)=\frac{f_\beta(y, z)}{h_\beta(y)}.$

Given $n$ observations $y_1, y_2, \cdots, y_n$ of $Y$, the EM algorithm is to maximize the log-likelihood $\max_{\beta\in \Omega }\ell_n(\beta) = \sum_{i=1}^n\log h_\beta(y_i).$
Due to the unobserved latent variable $Z$, it is often difficult to directly evaluate $\ell_n(\beta)$. Thus, we %turn to 
consider the lower bound of $\ell_n(\beta)$ . By Jensen's inequality, we have
\begin{align}
%\vspace{-0.1in}
    \frac{1}{n}[\ell_n(\beta)-\ell_n(\beta')]
  & \geq \frac{1}{n}\sum_{i=1}^n\int_{\mathcal{Z}}k_{\beta'}(z|y_i)\log f_\beta(y_i, z)dz \nonumber \\
   &- \frac{1}{n}\sum_{i=1}^n\int_{\mathcal{Z}}k_{\beta'}(z|y_i)\log {f_{\beta'}(y_i, z)}dz.   \label{eq:1}
\end{align}
Let $Q_n(\beta; \beta')=\frac{1}{n}\sum_{i=1}^n q_i(\beta;\beta') $, where 
\begin{equation}\label{eq:2}
%\vspace{-0.1in}
   q_{i}(\beta; \beta')=\int_{\mathcal{Z}}k_{\beta'}(z|y_i)\log f_\beta(y_i, z)dz.  \footnote{We use $ q(\beta; \beta')$ for general sample $y$.}
\end{equation}
Also, it is convenient to let   $Q(\beta; \beta')$ denote the expectation of $Q_n(\beta; \beta')$ w.r.t $\{y_i\}_{i=1}^n$, that is,
\begin{equation}\label{eq:3}
    Q(\beta; \beta')= \mathbb{E}_{y\sim h_{\beta^*}}\int_{\mathcal{Z}}k_{\beta'}(z|y)\log f_\beta(y, z)dz.
\end{equation}
We can see that the second term on the right hand side of (\ref{eq:1}) is independent on $\beta$. Thus, given some fixed $\beta'$, we can maximize the lower bound function $Q_n(\beta; \beta')$ over $\beta$ to obtain sufficiently large $\ell_n(\beta)-\ell_n (\beta')$. Thus, in the $t$-th iteration of the standard EM algorithm, we can evaluate $Q_n(\cdot; \beta^t)$ at the E-step and then perform the operation of $\beta^{t+1}=\max_{\beta\in \Omega }Q_n(\beta; \beta^t)$ at the M-step. See \citep{mclachlan2007algorithm} for more details.

In addition to the exact maximization implementation of the M-step, we add a gradient ascent implementation of the M-step, which performs an approximate maximization via a gradient descent step. \\

\noindent \textbf{Gradient EM Algorithm \citep{balakrishnan2017statistical}} When $Q_n(\cdot; \beta^t)$ is differentiable, the update of $\beta^t$ to $\beta^{t+1}$ consists of the following two steps. 
\begin{itemize}
    \item E-step: Evaluate the functions in (\ref{eq:2}) to compute $Q_n(\cdot; \beta^t)$.
    \item M-step: Update $\beta^{t+1}=\beta^t+\eta \nabla Q_n(\beta^t; \beta^t)$, where $\nabla$ is the derivative of $Q_n$ w.r.t the first component and $\eta$ is the step size. 
\end{itemize}
Next, we give some examples that use the gradient EM algorithm. Note that they are the typical examples for studying the statistical property of EM algorithm \citep{wang2015high,balakrishnan2017statistical,yi2015regularized,zhu2017high}.\\

\noindent \textbf{Gaussian Mixture Model (GMM)} Let $y_1, \cdots, y_n$ be $n$ i.i.d samples from $Y\in \mathbb{R}^d$ with 
\begin{equation}\label{eq:4}
    Y = Z\cdot \beta^*+V,
\end{equation}
where $Z$ is a Rademacher random variable ({\em i.e.,} $\mathbb{P}(Z=+1)= \mathbb{P}(Z=-1)=\frac{1}{2}$),  and $V\sim \mathcal{N}(0, \sigma^2 I_d)$ is independent of $Z$ for some known standard deviation $\sigma$.  In GMM, we have 
\begin{equation}\label{aeq:1}
    \nabla q(\beta;\beta)=[2w_\beta(y)-1]\cdot y-\beta,
\end{equation}
where $w_\beta(y)=\frac{1}{1+\exp(-\langle \beta, y\rangle/\sigma^2)}$. \\

\noindent \textbf{Mixture of (Linear) Regressions Model (MRM)} Let  $(x_1, y_1), (x_2, y_2),$ $\cdots, (x_n, y_n)$ be $n$ samples i.i.d sampled from $Y\in \mathbb{R}$ and $X\in \mathbb{R}^d$  with 
\begin{equation}\label{eq:5}
    Y= Z\langle \beta^*, X \rangle +V, 
\end{equation}
where $X\sim \mathcal{N}(0, I_d)$, $V\sim \mathcal{N}(0, \sigma^2)$, $Z$ is a Rademacher random variable, and $X, V, Z$ are independent. In this case, we have 
\begin{equation}\label{aeq:2}
    \nabla q(\beta;\beta)=(2w_\beta(x, y)-1) \cdot y \cdot x-x x^T\cdot \beta,
\end{equation}
where $w_\beta(x, y)=\frac{1}{1+\exp(-y\langle \beta ,x \rangle/\sigma^2)}$. \\

\noindent \textbf{Linear Regression with Missing Covariates (RMC)} We assume that $Y\in \mathbb{R}$ and $X\in \mathbb{R}^d$ satisfy 
\begin{equation}\label{eq:6}
    Y= \langle X, \beta^* \rangle +V,
\end{equation}
where $X\sim \mathcal{N}(0, I_d)$ and $V\sim \mathcal{N}(0, \sigma^2)$ are independent.  Let $x_1, x_2, \cdots, x_n$ be $n$ observations of $X$ with each coordinate of $x_i$ missing (unobserved) independently with probability $p_m\in[0,1)$.  In this case, we have \begin{equation}\label{aeq:3}
    \nabla q(\beta; \beta)= y\cdot m_\beta(x^{\text{obs}},y)-K_\beta(x^{\text{obs}}, y)\beta,
\end{equation}
where the functions $m_\beta(x^{\text{obs}},y)\in \mathbb{R}^d$ and $K_\beta(x^{\text{obs}}, y)\in \mathbb{R}^{d\times d}$ are defined as:   
\begin{equation}\label{aeq:4}
    m_\beta(x^{\text{obs}},y)= z \odot x+\frac{y-\langle \beta, z\odot x\rangle }{\sigma^2+\|(1-z)\odot \beta\|_2^2}(1-z)\odot \beta 
\end{equation}
and 
\begin{multline}\label{aeq:100} 
    K_\beta(x^{\text{obs}}, y)=\text{diag}(1-z)+  m_\beta(x^{\text{obs}},y)\cdot [  m_\beta(x^{\text{obs}},y)]^T 
   \\  -[(1-z)\odot   m_\beta(x^{\text{obs}},y)]\cdot [(1-z)\odot   m_\beta(x^{\text{obs}},y)]^T,
\end{multline}
where vector $z \in \mathbb{R}^d$ is defined as $z_{j}=1$ if $x_{j}$ is observed and $z_{j}=0$ is $x_{j}$ is missing, and $\odot$ denotes the Hadamard product of matrices. 

Next, we provide several definitions on the required properties of functions $Q_n(\cdot; \cdot)$ and $Q(\cdot; \cdot)$. Note that some of them have been used in  previous studies on the statistical guarantees of EM algorithm  \citep{balakrishnan2017statistical,wang2015high,zhu2017high}.

\begin{definition}\label{def:0}
 Function $Q(\cdot; \beta^*)$ is self-consistent if  
    $\beta^*=\arg\max_{\beta\in \Omega }Q(\beta; \beta^*).$
That is, $\beta^*$ maximizes the lower bound of the log likelihood function.
\end{definition}

\begin{definition}[Lipschitz-Gradient-2($\gamma, \mathcal{B}$)]\label{def:1}
$Q(\cdot; \cdot)$ is called Lipschitz-Gradient-2($\gamma, \mathcal{B}$), if for the underlying parameter $\beta^*$ and any $\beta\in \mathcal{B}$ for some set $\mathcal{B}$, the following holds 
\begin{equation}\label{eq:7}
    \|\nabla Q(\beta; \beta^*)-\nabla Q(\beta; \beta)\|_2\leq \gamma \|\beta-\beta^*\|_2.
\end{equation} 
\end{definition}%\vspace{-0.1in}

We note that there are some differences between the definition of Lipschitz-Gradient-2 and the Lipschitz continuity condition in the convex optimization literature \citep{nesterov2013introductory}. Firstly, in (\ref{eq:7}), the gradient is w.r.t the second component, while the Lipschitz continuity is w.r.t the first component. Secondly, the property holds only for fixed $\beta^*$ and any $\beta$, while the Lipschitz continuity is for all $\beta, \beta'\in \mathcal{B}$.

\begin{definition}[$\mu$-smooth]
$Q(\cdot; \beta^*)$ is $\mu$-smooth, that is if
for any $\beta, \beta'\in \mathcal{B}$, 
%and the underlying parameter $\beta^*$, $Q(\cdot; \beta^*)$ is %$\mu$-smooth, {\em i.e.,} 
$ Q(\beta;\beta^*)\geq Q(\beta'; \beta^*)+(\beta-\beta')^T\nabla Q(\beta';\beta^*)-\frac{\mu}{2}\|\beta'-\beta\|_2^2.$
\end{definition}

\begin{definition}[$\upsilon$-strongly concave] \label{def:3} $Q(\cdot; \beta^*)$ is $\upsilon$-strongly concave, that is if 
for any $\beta, \beta'\in \mathcal{B}$,  $Q(\beta;\beta^*)\leq  Q(\beta'; \beta^*)+(\beta-\beta')^T\nabla Q(\beta';\beta^*)-\frac{\upsilon}{2}\|\beta'-\beta\|_2^2.$
\end{definition}

In the following we will propose the assumptions that will be used throughout the whole paper. Note that these assumptions are commonly used in other works on statistical analysis of EM algorithm such as \citep{balakrishnan2017computationally,zhu2017high,wang2015high}. 
\begin{assumption}\label{assumption:1}
We assume that function $Q(\cdot; \cdot)$ in (\ref{eq:3}) is self-consistent, Lipschitz-Gradient-2($\gamma, \mathcal{B}$), $\mu$-smooth,  $\upsilon$-strongly concave over some set $\mathcal{B}$. Moreover, we assume that $\forall j\in[d]$ and $ \beta\in \mathcal{B}$, there is some known upper bound $\tau$ on the second-order moment of the $j$-coordinate of $\nabla q(\beta, \beta)$, {\em i.e.,} $\mathbb{E}_{y} (\nabla_j q(\beta, \beta))^2\leq \tau$ and for each $i\in [n]$,  $\nabla_j q_i(\beta, \beta)$ is independent with others. 
\end{assumption}
\begin{definition}[Differential Privacy
\citep{dwork2006calibrating}]\label{def:5}
	Given a data universe $\mathcal{X}$, we say that two datasets $D,D'\subseteq \mathcal{X}$ are neighbors if they differ by only one entry, which is denoted as $D \sim D'$. A randomized algorithm $\mathcal{A}$ is $(\epsilon,\delta)$-differentially private (DP) if for all neighboring datasets $D,D'$ and for all events $S$ in the output space of $\mathcal{A}$, we have 
	$\mathbb{P}(\mathcal{A}(D)\in S)\leq e^{\epsilon} \mathbb{P}(\mathcal{A}(D')\in S)+\delta.$
\end{definition}
\begin{definition}[Local Differential Privacy]
    A randomized algorithm $\mathcal{A}$ is $(\epsilon, \delta)$-local differential privacy (LDP) if for any $x, x'\in \mathcal{X}$ and for all events $S$ in the output space of $\mathcal{A}$, we have 	$\mathbb{P}(\mathcal{A}(x)\in S)\leq e^{\epsilon} \mathbb{P}(\mathcal{A}(x')\in S)+\delta.$
\end{definition}
	\begin{definition}[Gaussian Mechanism]\label{def:6}
		Given a function $q : \mathcal{X}^n\rightarrow \mathbb{R}^p$, the Gaussian Mechanism is defined as:
		$\mathcal{M}_G(D,q,\epsilon)=q(D)+ Y,$
		where Y is drawn from a Gaussian Distribution $\mathcal{N}(0,\sigma^2I_p)$ with $\sigma\geq \frac{\sqrt{2\ln(1.25/\delta)}\Delta_2(q)}{\epsilon}$.  $\Delta_2(q)$ is the $\ell_2$-sensitivity of the function $q$, {\em i.e.,}
		$\Delta_2(q)=\sup_{D\sim D'}||q(D)-q(D')||_2.$
		Gaussian Mechanism preserves $(\epsilon,\delta)$-differentially private.
	\end{definition}

	Due to the similarity with the Gradient Descent algorithm and the simplicity of illustrating our idea compared with the original EM algorithm, and since it has been claimed that even in the non-private case Gradient EM algorithm is more flexible than EM algorithm \citep{balakrishnan2017statistical},  in this paper, we will mainly focus on DP versions of Gradient EM algorithm. We note that it is straightforward to extend our idea to DP EM algorithm.
	%See Appendix \ref{sec:dpem} for the statistical guarantees of the DP  EM algorithm. 
	%and the reason why our DP Gradient EM algorithm has priority to DP EM algorithm. 
%\vspace{-0.2in}
\section{Main Method}
%\vspace{-0.1in}
\subsection{Main Difficulty}
%\vspace{-0.2in}
In the previous section, we introduced the Gradient EM algorithm, which updates the estimator via the gradient $ \nabla Q_n(\beta^t; \beta^t)$. It is notable that this idea is quite similar to the Gradient Descent algorithm. Moreover, we know that there are several DP versions of the (Stochastic) Gradient Descent algorithm such as \citep{bassily2014private,wang2017differentially,song2013stochastic,wang2019differentially,lee2018concentrated}. The key idea of DP Gradient Descent is adding some randomized noise such as Gaussian noise to preserve DP property in each iteration, and by the composition theorem of DP (\citep{dwork2014algorithmic}), the whole algorithm will still be DP. Thus, motivated by this, to design a DP variant of Gradient EM algorithm, the most direct way is adding some Gaussian noise to the gradient $ \nabla Q_n(\beta^t; \beta^t)$ in each iteration and updating the parameter. 

However, it is notable that  we cannot add Gaussian noise directly to the gradient in the Gradient EM algorithm. The main reason is that all previous DP Gradient Descent algorithms need to assume that each component of the gradient (which correspond to the function $\nabla q_i$ in (\ref{eq:2})) is bounded, or the loss function is $O(1)$-Lipschitz, such as Logistic Regression, so that its $\ell_2$-norm sensitivity is bounded and thus  the Gaussian mechanism can be used. However, in the Gradient EM algorithm, each component ($ \nabla q_i(\beta^t; \beta^t)$ in (\ref{eq:2})) is  unbounded in most of the cases. For example, we can easily show the following fact. 
\begin{theorem}\label{thm:1}
Consider the GMM in (\ref{eq:4}), there is a case with fixed $\beta$, such that for each constant $c$, with {\bf positive probability} w.r.t $y$ we have  $\|\nabla q(\beta; \beta)\|_2\geq c. $
\end{theorem}
Thus, to design a DP (Gradient) EM algorithm, the major difficulty lies in  how to process the gradient to make its sensitivity bounded. Two main approaches are used in previous work: (1) \citep{park2017dp} assumed that datasets are pre-processed such that the $\ell_2$ norm of each sample is bounded by $1$. However, as mentioned previously, our goal is to achieve 
%instead focusing on the practical performance of EM in \citep{park2017dp}, 
%since in this paper we focused on 
the statistical guarantees for the DP (Gradient) EM algorithm. If a similar approach is adopted in our algorithm, the 
%and we know that 
(manual) normalization can easily destroy many statistical properties of the data and force the private estimator to introduce additional bias, making it  
%Thus, if here we adopted the approach  in \citep{park2017dp}, our private estimator will introduces additional bias and this will make it 
inconsistent.\footnote{An estimator $\beta_n$ is consistent if $\lim_{n\to \infty} \|\beta_n-\beta^*\|_2=0$.} (2) Instead of normalizing the datasets,  \citep{abadi2016deep} first clipped the gradient to ensure that the $\ell_2$-norm of each component of the gradient is bounded by the threshold $C$, and then added Gaussian noise (see Algorithm \ref{alg:1} for more details). However, such an approach may cause two issues. First, in general  clipping gradient could introduce additional bias even in statistical estimation, which has also been pointed out in \citep{song2020characterizing}. Second, the threshold $C$ heavily affects the convergence speed and selecting the best $C$ is quite difficult (see Experimental section for more details). Due to these two reasons, it is hard to study the statistical guarantees of Algorithm \ref{alg:1}. Thus, we need a new approach to pre-process the gradient to ensure that it has not only bounded $\ell_2$-norm but also consistent statistical guarantee. 

\begin{algorithm}[!h]
	\caption{Clipped DP Gradient EM} \label{alg:1}
	$\mathbf{Input}$: $D=\{y_i\}_{i=1}^n\subset \mathbb{R}^d$, privacy parameters $\epsilon, \delta$; $Q_n(\cdot; \cdot)$ and its $q(\cdot; \cdot)$, initial parameter $\beta^0$, gradient norm $C$, step size $\eta$ and the number of iterations $T$.
	\begin{algorithmic}[1]
 \For{$t=1,2,\cdots, T$}
 \State For each $i\in [n]$, evaluate the function in (\ref{eq:2}) to compute $q_i(\beta ; \beta^{t-1})$. 
 \State Clip gradient: $$\nabla\bar{q}_i(\beta^{t-1}; \beta^{t-1})= \frac{\nabla q_i(\beta^{t-1}; \beta^{t-1})}{\max\{1, \frac{\|\nabla q_i(\beta^{t-1}; \beta^{t-1})\|_2}{C}\}}. $$
 \State Update $\beta^{t}=\beta^t+\eta (\nabla  \bar{Q}_n(\beta^{t-1}; \beta^{t-1})+\mathcal{N}(0, C^2\sigma^2I_d),$ where $\nabla \bar{Q}_n(\beta^{t-1}; \beta^{t-1}))=\frac{1}{n}\sum_{i=1}^n\nabla\bar{q}_i(\beta^{t-1}; \beta^{t-1})$ and $\sigma^2=c\frac{T\log \frac{1}{\delta}}{n^2\epsilon^2}$ for some constant $c$. 
 \EndFor
 \State Return $\beta^T$
	\end{algorithmic}
\end{algorithm}
%\vspace{-0.2in}
\subsection{Our Method}
%\vspace{-0.2in}
In this section, we will propose our method to overcome the aforementioned    difficulties. Since our method is motivated by a robust and private mean estimator for heavy-tailed distributions, which was given in \citep{wangicml2020}, and it is derived from the robust mean estimator  in \citep{holland2019a}. To be self-contained, we first review their estimator. We now consider a $1$-dimensional random variable $x$ and assume that $x_1, x_2, \cdots, x_n$ are i.i.d. sampled from $x$. The estimator consists of three steps:

\noindent \textbf{Scaling and Truncation} For each sample $x_i$, we first re-scale it by dividing $s$ (which will be specified later). Then, the re-scaled one was passed through a  soft truncation function $\phi$. Finally, we put the truncated mean back to the original scale. That is, 
\begin{equation}\label{eq:9}
%\vspace{-0.05in}
    \frac{s}{n}\sum_{i=1}^n \phi(\frac{x_i}{s})\approx \mathbb{E}X.
\end{equation}
Here, we use the function given in \citep{catoni2017dimension},
\begin{equation}\label{eq:10}
    \phi(x)= \begin{cases} x-\frac{x^3}{6}, & -\sqrt{2}\leq x\leq \sqrt{2} \\
    \frac{2\sqrt{2}}{3}, & x>\sqrt{2} \\
    -\frac{2\sqrt{2}}{3}, & x<-\sqrt{2}.   
    \end{cases}
\end{equation}
A key property for $\phi$ is that $\phi$ is bounded, that is,  $|\phi(x)|\leq \frac{2\sqrt{2}}{3}$.
%%\vspace{-0.15in}

\noindent \textbf{Noise Multiplication} Let $\eta_1, \eta_2, \cdots, \eta_n$ be random noise generated from a common distribution $\eta\sim \chi$ with $\mathbb{E}\eta =0$. We multiply each data $x_i$ by a factor of $1+\eta_i$, and then perform the scaling and truncation step on the term $x_i(1+\eta_i)$.  
%through the scaling and truncation step, 
That is, 
\begin{equation}\label{eq:11}
%\vspace{-0.1in}
    \tilde{x}(\eta) =\frac{s}{n}\sum_{i=1}^n \phi(\frac{x_i+\eta_i x_i}{s}).
\end{equation}
%%\vspace{-0.2in}
\noindent \textbf{Noise Smoothing} In this final step, we smooth the multiplicative noise by taking the expectation w.r.t. the distributions. That is, 
\begin{equation}\label{eq:12}
%\vspace{-0.1in}
    \hat{x}=\mathbb{E}  \tilde{x}(\eta) = \frac{s}{n}\sum_{i=1}^n \int \phi(\frac{x_i+\eta_i x_i}{s})d \chi(\eta_i).
\end{equation}
Computing the explicit form of each integral in (\ref{eq:12}) depends on the function $\phi(\cdot)$ and the distribution $\chi$. Fortunately, \citep{catoni2017dimension} showed that when $\phi$ is in (\ref{eq:10}) and $\chi\sim \mathcal{N}(0, \frac{1}{\beta})$ (where $\beta$ will be specified later), we have for any $a$ and  $b>0$
\begin{equation}\label{eq:13}
%\vspace{-0.1in}
    \mathbb{E}_{\eta}\phi(a+b\sqrt{\beta}\eta)=a(1-\frac{b^2}{2})-\frac{a^3}{6}+\hat{C}(a,b),
\end{equation}
where $\hat{C}(a, b)$ is a correction form which is easy to implement and it has the following explicit form.  \\

\noindent{ \bf Explicit Form of (\ref{eq:13})}
We first define the following notations:
\begin{align*}
   &V_- := \frac{\sqrt{2}-a}{b}, V_{+}=\frac{\sqrt{2}+a}{b} \\
   &F_{-}:= \Phi(-V_-), F_{+}:=\Phi(-V_+) \\
   &E_{-}:= \exp(-\frac{V^2_-}{2}), E_{+}:=\exp(-\frac{V^2_{+}}{2}),
\end{align*}
where $\Phi$ denotes the CDF of the standard Gaussian distribution. Then \begin{equation*}
    \hat{C}(a,b)=T_1+T_2+\cdots+T_5, 
\end{equation*}
where 
\begin{align*}
    &T_1:= \frac{2\sqrt{2}}{3}(F_{-}-F_{+}) \\
    & T_2:= -(a-\frac{a^3}{6})(F_{-}+F_{+}) \\ 
    & T_3:=\frac{b}{\sqrt{2\pi}}(1-\frac{a^2}{2})(E_{+}-E_{-})\\ 
    &T_4 : = \frac{ab^2}{2}\left(F_{+}+F_{-}+\frac{1}{\sqrt{2\pi}}(V_{+}E_{+}+V_{-}E_{-})\right) \\
    & T_5:= \frac{b^3}{6\sqrt{2\pi}}\left((2+V_{-}^2)E_{-}-(2+V_{+}^2)E_{+}\right). 
\end{align*}
%\citep{holland2019a} showed the following estimation error  for  the mean estimator $\hat{x}$ after these three steps. 
 The estimation error of the mean estimator $\hat{x}$ after these three steps is given as following. 
\begin{lemma}[Lemma 5 in \citep{holland2019a}] \label{lemma:1}
Let  $x_1, x_2, \cdots, x_n$ be i.i.d. samples from  distribution $x\sim \mu$. Assume that there is some known upper bound on the second-order moment, {\em i.e.,} $\mathbb{E}_\mu x^2\leq \tau $. For a given failure probability $\zeta$, if  set  $\beta= 2\log \frac{1}{\zeta}$ and $s=\sqrt{\frac{n\tau }{2\log\frac{1}{\zeta}}}$, then with probability at least $1-\zeta$ we have
$|\hat{x}-\mathbb{E}x|\leq O(\sqrt{\frac{\tau \log \frac{1}{\zeta}}{n}}). $ \end{lemma}
%%\vspace{-0.2in}
To obtain an $(\epsilon,\delta)$-DP estimator, the key observation is that the bounded function $\phi$ in (\ref{eq:10})  also makes the integral form of  (\ref{eq:12}) bounded by $\frac{2\sqrt{2}}{3}$. Thus, we know that the $\ell_2$-norm sensitivity is $\frac{s}{n}\frac{4\sqrt{2}}{3}$. Hence, the query 
\begin{equation}\label{eq:15}
    \mathcal{A}(D)=\hat{x}+ Z, Z\sim \mathcal{N}(0, \sigma^2), \sigma^2=O(\frac{s^2\log \frac{1}{\delta}}{\epsilon^2n^2})
\end{equation}
will be $(\epsilon, \delta)$-DP, which leads to  
%we can easily get 
the following result. 

\begin{lemma}[Theorem 6 in \citep{wangicml2020}]\label{lemma:2}
Under the assumptions in Lemma \ref{lemma:1}, with probability at least $1-\zeta$ the following holds
%\vspace{-0.1in}
\begin{equation}\label{eq:16}
    |\mathcal{A}(D)-\mathbb{E}(x)|\leq O(\sqrt{\frac{\tau \log \frac{1}{\delta}\log\frac{1}{\zeta}}{n\epsilon^2}}).
\end{equation}
\end{lemma}
{Although in Lemma \ref{lemma:2} we just need to assume that $x$ has bounded second order moment instead of bounded norm,  there are still other two problems}: First, Lemma \ref{lemma:2} is directly followed by Lemma \ref{lemma:1} with the same parameter $s$ and $\beta$. However, due to the noise we add,  is it possible that we can further improve the result by some other specific $s$ and $\beta$? Second, by using the previous parameters we cannot extend to the local DP model since it will have a huge error (we can easily see that in the local DP setting, the mechanism is equivalent to (\ref{eq:15}) with $\sigma^2=O(\frac{s^2\log \frac{1}{\delta}}{n\epsilon^2})=O(\frac{\tau}{\epsilon^2})$, which could be considered as a constant error since it is not decayed to zero when $n$ increases). Thus, can we extend the method to the local DP model? In the following we provide affirmative answer of these two questions through finer analysis of the mechanism (\ref{eq:15}). Our analysis is started from a Legendre transform of the mapping given by \citep{catoni2004statistical},  see Appendix \ref{sec:proof} for details.
\begin{theorem}\label{thm:new}
Let  $x_1, x_2, \cdots, x_n$ be i.i.d. samples from  distribution $x\sim \mu$. Assume that there is some known upper bound on the second-order moment, {\em i.e.,} $\mathbb{E}_\mu x^2\leq \tau $. For a given failure probability $\zeta$, if  set  $\beta=\sqrt{\log \frac{1}{\zeta}}$ and $s=\frac{\sqrt{n\epsilon\tau}}{\log \frac{1}{\zeta}\log^{1/4}\frac{1}{\delta}}$, then with probability at least $1-\zeta$ mechanism (\ref{eq:15}) satisfies 
\begin{equation}\label{eq:new}
   |\mathcal{A}(D)-\mathbb{E}x|\leq O(\sqrt{\frac{\tau \log^{1/2} \frac{1}{\delta}\log \frac{1}{\zeta}}{n\epsilon}}). 
\end{equation}
\end{theorem}

Compared with (\ref{eq:16}), we can see the error bound in  (\ref{eq:new}) improves a factor of $O(\frac{1}{\sqrt{\epsilon}})$. We will also see that, by using a similar analysis, we can have a local DP version of (\ref{eq:15}) with an error bound of $O(\sqrt{\frac{\tau \log^{1/2} \frac{1}{\delta}\log \frac{1}{\zeta}}{\sqrt{n}\epsilon}})$. To our best knowledge, this is the first result on private mean estimation of heavy-tailed distribution in the local DP model. Thus, all of our results on DP Gradient EM algorithm can be easily extended to the local model via Theorem \ref{thm:new2}.

\begin{theorem}\label{thm:new2}
Let  $x_1, x_2, \cdots, x_n$ be i.i.d. samples from  distribution $x\sim \mu$. Assume that there is some known upper bound on the second-order moment, {\em i.e.,} $\mathbb{E}_\mu x^2\leq \tau $. Consider the following mechanism: 
\begin{equation}
    \hat{x}_i=  s  \int \phi(\frac{x_i+\eta_i x_i}{s})d \chi(\eta_i)+Z, Z\sim \mathcal{N}(0, \sigma^2), \sigma^2=O(\frac{s^2\log \frac{1}{\delta}}{\epsilon^2})
\end{equation}
and the coordinator (server) output $\mathcal{A}(D)=\frac{1}{n}\sum_{i=1}^n \hat{x}_i$. Then we can see the whole algorithm is $(\epsilon, \delta)$-LDP. Moreover, for a given failure probability $\zeta$, if  set  $\beta=\sqrt{\log \frac{1}{\zeta}}$ and $s=\frac{\sqrt[4]{n}\sqrt{\epsilon\tau}}{\log \frac{1}{\zeta}\log^{1/4}\frac{1}{\delta}}$, then with probability at least $1-\zeta$, $\mathcal{A}(D)$ satisfies 
\begin{equation}\label{eq:new9}
   |\mathcal{A}(D)-\mathbb{E}x|\leq O(\sqrt{\frac{\tau \log^{1/2} \frac{1}{\delta}\log \frac{1}{\zeta}}{\sqrt{n}\epsilon}}). 
\end{equation}
\end{theorem}

Inspired by the previous private 1-dimensional mean estimation, we propose our method (Algorithm \ref{alg:2}). In Algorithm \ref{alg:2}, the key idea is that, in the $t$-th iteration of Gradient EM algorithm, we first apply the previous private estimator to each coordinate of the gradient $\nabla Q_n(\beta^{t-1}; \beta^{t-1})$, and then perform the M-step. We can easily show that Algorithm \ref{alg:2} is $(\epsilon, \delta)$-DP. 
\begin{theorem}[Privacy guarantee]\label{theorem:2}
For any $0< \epsilon, \delta<1 $,  Algorithm \ref{alg:2} is $(\epsilon, \delta)$-DP. 
\end{theorem}
In the following, we will show the statistical guarantee for the models under Assumption \ref{assumption:1}, if the initial parameter $\beta^0$ is closed to the underlying parameter $\beta^*$ enough.

\begin{algorithm*}[!ht]
	\caption{DP Gradient EM Algorithm} \label{alg:2}
	$\mathbf{Input}$: $D=\{y_i\}_{i=1}^n\subset \mathbb{R}^d$, privacy parameters $\epsilon, \delta$, $Q(\cdot; \cdot)$ and its $q_i(\cdot;  \cdot)$,  initial parameter $\beta^0\in \mathcal{B}$, $\tau$ which satisfies Assumption \ref{assumption:1}, the number of iterations $T$ (to be specified later), step size $\eta$ and failure probability $\zeta>0$. 
	\begin{algorithmic}[1]
    \State Let $\tilde{\epsilon}= \sqrt{\log \frac{1}{\delta}+\epsilon}-\sqrt{\log \frac{1}{\delta}}$, $s={\frac{\sqrt{m\tau\tilde{\epsilon}}}{2\log \frac{d}{\zeta}}}$, $\beta=\sqrt{\log \frac{d}{\zeta}}$. Partite the data $D$ into $T$ subsets $\{D_i\}_{i=1}^T$ with $|D_i|=m= \frac{n}{T}$.
    \For {$t=1, 2, \cdots, T$}
    \State For each $j\in [d]$, calculate the robust gradient by using (\ref{eq:12}) and add Gaussian noise over the dataset $D_t$, that is 
    \begin{small}
    \begin{multline} \label{eq:17}
        g_j^{t-1}(\beta^{t-1})= \frac{1}{m}\sum_{i=1}^m \left(\nabla_jq_i(\beta^{t-1}, \beta^{t-1})\big(1-\frac{\nabla^2_jq_i(\beta^{t-1}, \beta^{t-1})}{2s^2\beta}\big)- \frac{\nabla^3_jq_i(\beta^{t-1}, \beta^{t-1})}{6s^2}\right)\\+\frac{s}{n}\sum_{i=1}^n\hat{C}\left(\frac{\nabla_jq_i(\beta^{t-1}, \beta^{t-1})}{s}, \frac{|\nabla_jq_i(\beta^{t-1}, \beta^{t-1})|}{s\sqrt{\beta}}\right)+ Z_{j}^{t-1}, 
    \end{multline}
    \end{small}
    where $y_i\in D_t$ for $i\in [m]$, $Z_{j}^{t-1}\sim \mathcal{N}(0, \sigma^2)$ with $\sigma^2= \frac{16s^2 d}{9 m^2\tilde{\epsilon}^2}=\frac{4dT\tau}{9n\beta^2\tilde{\epsilon}}$.
    \State Let vector $\tilde{\nabla}Q_n(\beta^{t-1})\in \mathbb{R}^d$  denote  $\tilde{\nabla}Q_n(\beta^{t-1})=(g_1^{t-1}(\beta^{t-1}), g_2^{t-1}(\beta^{t-1}), \cdots, g_d^{t-1}(\beta^{t-1}))$. 
    \State Update
        $\beta^{t}=\beta^{t-1}+\eta \tilde{\nabla}Q_n(\beta^{t-1}). $
    \EndFor
	\end{algorithmic} 
\end{algorithm*}
\begin{theorem}\label{theorem:3} Let the parameter set $\mathcal{B}=\{\beta: \|\beta-\beta^*\|_2\leq R \}$ for $R=\kappa \|\beta^*\|_2$ for some constant $\kappa \in (0, 1)$.  Assume that  Assumption \ref{assumption:1} holds for parameters $\gamma, \mathcal{B}, \mu, v, \tau $ satisfying the condition of $1-2\frac{v-\gamma}{v+\mu} \in (0, 1)$. Also, assume that $\|\beta^0-\beta^*\|_2\leq \frac{R}{2}$, $n$ is large enough so that 
\begin{equation}\label{neweq:1}
%\vspace{-0.05in}
\tilde{\Omega}( (\frac{1}{v-\gamma})^2 \frac{d^2\tau T \sqrt{\log \frac{1}{\delta}} \log \frac{1}{\zeta} }{\epsilon R^2})\leq n.  
\end{equation}
Then, with probability at least $1-\zeta$,  we have, for all $t\in [T]$, $\beta^t\in \mathcal{B}$. If it holds and if taking $T=O(\frac{\mu+v}{v-\gamma}\log n )$ and $\eta=\frac{2}{\mu+v}$,  we have 
\begin{equation}\label{eq:18}
%\vspace{-0.05in}
    \|\beta^T-\beta^*\|_2\leq \tilde{O}\big(R\sqrt{\frac{v+\mu}{(v-\gamma)^3}} \frac{d\sqrt[4]{\log \frac{1}{\delta}}\log \frac{1}{\zeta}\sqrt{\tau}}{\sqrt{n \epsilon}} \big), 
\end{equation}
where the $\tilde{O}$-term and $\tilde{\Omega}$-term  omit $\log d$, $\log n$ and other factors (see Appendix for the explicit form of the result). 

\end{theorem}
\begin{remark}
There are several points that need to note. Firstly, the assumptions of the parameter set $\beta$ and the initial parameter $\beta^0$ are commonly used in other papers on statistical guarantees of (Gradient) EM algorithm such as \citep{balakrishnan2017computationally,zhu2017high,wang2015high}. Even though Theorem \ref{theorem:3} requires that the initial estimator be close enough to the optimal one, our experiments show that the algorithm actually performs quite well for any random initialization. Secondly, in (\ref{neweq:1}) we need to assume that $n\propto \frac{1}{R^2}$, where $R$ is the radius of $\mathcal{B}$. This is due to that in Algorithm \ref{alg:2}, we need to keep each $\beta^t\in \mathcal{B}$ under perturbation. When $R$ is small,  we have to let the noise be small enough, which means that $n$ should be large enough. Finally, for specific models, $R, v, \mu, \gamma$ are constants, this means that the error in (\ref{eq:18}) is $\tilde{O}(\frac{d\sqrt{\tau}}{\sqrt{n\epsilon}})$. However, here $\tau$ depends on the model, which may also depend on $d$ and $\|\beta^*\|_2$. 
\end{remark}

\paragraph{Comparing with Previous Results}\label{sec:compare}
We can see the main idea of our algorithm is motivated by a result of estimating the mean of a $d$-dimensional heavy-tailed distributions in DP model. It is notable that recently \citep{kamath2019differentially}  and \citep{brunel2020propose} also studied estimating the mean of heavy-tailed distributions differentially privately. Next we will provide a detailed comparison with these work. 

\citep{brunel2020propose} derived concentration inequalities for differentially private median
and mean estimators building on the Propose-test-release (PTR) mechanism. Specifically, for the $1$-dimensional mean estimation problem, they showed that if the data samples sampled from some distribution with bounded third-order moment, then there is an $(\epsilon, \delta)$-DP algorithm whose output  $M(X)$ satisfies $|M(D)-\mathbb{E}(x)|^2\leq O(\frac{\log \frac{1}{\delta'}\log \frac{1}{\delta} }{n\epsilon^2})$ with probability at least $1-\delta'$. We can see our result in Theorem \ref{thm:new} is much more better than theirs. Moreover we can see in Theorem \ref{thm:new} we just need the bounded second-order moment assumption while \cite{brunel2020propose} needs bounded third-order moment. Moreover, there is no experimental study of their algorithm. Thus, from this view, our method is better than theirs.

Recently, \cite{kamath2019differentially} studied the private mean estimation problem under finite $\theta$-th order moment assumption with $\theta\in [2, \infty)$.  Specifically, if the data distribution only has finite second-order moment, then they showed that if $\|\mathbb{E}(x)\|_2\leq R$ for some known constant $R$, then there is an $(\epsilon, \delta)$-DP algorithm whose output $M(X)$ satisfies $\|M(D)-\mathbb{E}(x)\|^2_2\leq \tilde{O}(\frac{d\log \frac{1}{\delta} }{n\epsilon})$ with probability at least $0.7$ (see Theorem 4.7 in \cite{kamath2019differentially} for details). Combing with this result and our proofs we can get an improved upper bound of  $\frac{ \sqrt{d \tau \log n\log \frac{d}{\zeta}}}{\sqrt{\beta n\tilde{\epsilon}}})$ (we omit the proof here). Thus, from this perspective, our bounds are larger. However, there are one critical issue that forbid using the result of \citep{kamath2019differentially}; That is, we can see that these two bounds hold with probability at least $1-0.7\times T$, where $T$ is the iteration number, and $T=O(\log n)$ in our algorithm. That is when in the large scale case or when the condition number $\frac{\beta}{\alpha}$ is large, the probability of success will be closed to or even less than zero, which is meaningless. Compared to this, our results hold with probability $1-\delta'$ for any $\delta'\in (0, 1)$. Moreover, the algorithm in \citep{kamath2019differentially} is quite complicated and impractical, and it is unclear whether we can extend their method to the local DP model. Thus, our method is more practical and more general. 
\section{Implications for Some Specific Models}
%\vspace{-0.1in}
In this section, we apply our framework ({\em i.e.,} Algorithm \ref{alg:2}) to the models  mentioned in the Preliminaries section. To obtain results for these models, we only need to find the corresponding $\mathcal{B}, \gamma, k, R, v, \mu, \tau$ to ensure that Assumption \ref{assumption:1} and the assumptions in Theorem \ref{theorem:3} hold. 
%\vspace{-0.1in}
\subsection{Gaussian Mixture Model}
%\vspace{-0.1in}
The following lemma ensures the properties of Lipschitz-Gradient-2($\gamma, \mathcal{B}$), smoothness, strongly concave and self-consistency for  model (\ref{eq:4}).
\begin{lemma}[\citep{balakrishnan2017statistical,yi2015regularized}]\label{lemma:3}
If 
%Suppose we have 
$\frac{\|\beta^*\|_2}{\sigma}\geq r$, where $r$ is a sufficiently large constant denoting the minimum signal-to-noise ratio (SNR), then there exists an absolute constant $C>0$ such that the properties of self-consistent, Lipschitz-Gradient-2($\gamma, \mathcal{B})$, $\mu$-smoothness and $\upsilon$-strongly concave hold for function $Q(\cdot; \cdot)$  with
    $\gamma=\exp(-Cr^2), \mu=\upsilon=1, R=k\|\beta^*\|_2, k=\frac{1}{4}, \text{ and } \mathcal{B}=\{\beta:\|\beta-\beta^*\|_2\leq R\}.$
\end{lemma}
We can show the following  second-order moment bound for $\nabla_j q(\beta, \beta)$.
\begin{lemma}\label{lemma:4}
With the same notations as 
%notations of 
in Lemma \ref{lemma:3}, for each $\beta\in \mathcal{B}$, the $j$-the coordinate of $\nabla q(\beta; \beta)$ ({\em i.e.,} $\nabla_j q(\beta; \beta)$) satisfies the following inequality  
\begin{equation*}
    \mathbb{E}_{y} (\nabla_j q(\beta; \beta))^2\leq O((\|\beta^*\|^2_{\infty}+\sigma^2)). 
\end{equation*}
Also, for fixed $j\in [d]$, each $\nabla_j q_i(\beta;\beta)$, where $i\in [n]$, is independent with others.
\end{lemma} %\vspace{-0.1in}
Combining with Lemma \ref{lemma:3}, \ref{lemma:4} and Theorem \ref{theorem:3} we have the following statistical guarantee for GMM. 
\begin{theorem}\label{theorem:4}
With the same notations as 
%notations of 
in Lemma \ref{lemma:3}, in Algorithm \ref{alg:2}  assume that $\|\beta^0-\beta^*\|_2\leq \frac{1}{8}\|\beta^*\|_2$ and  $n$ is large enough so that
\begin{equation}\label{eq:19}
\tilde{\Omega}(  \frac{d^2\sqrt{\|\beta^*\|^2_{\infty}+\sigma^2} \sqrt{\log \frac{1}{\delta}} \log \frac{1}{\zeta} }{\epsilon \|\beta^*\|_2^2})\leq n.
\end{equation} 
Moreover, if take $T=O(\log n )$ and $\eta=O(1)$, then we have with probability at least $1-\zeta$
\begin{equation}\label{eq:20}
    \|\beta^T-\beta^*\|_2\leq \tilde{O}\big(\|\beta^*\|_2\frac{d \sqrt[4]{\log \frac{1}{\delta}}\log \frac{1}{\zeta}\sqrt{\|\beta^*\|^2_{\infty}+\sigma^2}}{\sqrt{n \epsilon}} \big), 
\end{equation}
where the $\tilde{O}, \tilde{\Omega}$ terms omit logarithmic and other factors.
\end{theorem}
\begin{remark}
Note that if we assume that $\sigma, \|\beta^*\|_2=O(1)$, then the error in (\ref{eq:20}) is upper bounded by $\tilde{O}(\frac{d}{\sqrt{n\epsilon}})$. This means that to achieve the error of $\alpha\in (0,1)$, the sample complexity is $\tilde{O}(\frac{d^2}{\alpha^2\epsilon})$. It is notable that for GMM, the near optimal rate is  $\tilde{O}(d^2(\frac{1}{\alpha^2}+\frac{1}{\alpha\epsilon})$ \citep{kamath2019differentially}.
%\footnote{Note that although \citep{kamath2019differentially} used TV distance, while we use the Euclidean distance, we can easily transfer our result to a result based on  TV distance via Pinsker's inequality and the KL diatance between two Gaussian distributions.} 
Thus when $\epsilon$ is some constant, our result matches their near optimal rate. However, as mentioned in previous section, their algorithm is too complicated to be practical and it is difficult to extend their method to other Mixture models.
\iffalse 
Also, we assume that the SNR is large, which is reasonable since it has been shown that for Gaussian Mixture Model with low SNR, the variance of noise makes it harder for the algorithm to converge \citep{ma2000asymptotic}, which is the same for MRM. 
\fi 
\end{remark}
%\vspace{-0.2in}
\subsection{Mixture of Regressions Model}
%\vspace{-0.1in}
\iffalse
The following lemma, which was given in \citep{balakrishnan2017statistical,yi2015regularized}, shows the properties of  Lipschitz-Gradient-2($\gamma, \mathcal{B}$), smoothness and strongly concave for model (\ref{eq:5}). 
\fi
\begin{lemma}[\citep{balakrishnan2017statistical,yi2015regularized}]\label{lemma:5}
If $\frac{\|\beta^*\|_2}{\sigma}\geq r$, where $r$ is a sufficiently large constant denoting %that the denotes 
the required minimal signal-to-noise ratio (SNR), then function $Q(\cdot; \cdot)$ of the Mixture of Regressions Model has the properties  of self-consistent, Lipschitz-Gradient-2($\gamma, \mathcal{B})$, $\mu$-smoothness, and $\upsilon$-strongly with 
    $\gamma\in(0,\frac{1}{4}), \mu=\upsilon=1, \mathcal{B}=\{\beta: \|\beta-\beta^*\|_2\leq R\}, R=k\|\beta^*\|_2$, and $k=\frac{1}{32}. $
\end{lemma}
\begin{lemma}\label{lemma:6}
With the same 
notations as in Lemma \ref{lemma:5}, for each $\beta\in \mathcal{B}$, the $j$-the coordinate of $\nabla q_i(\beta; \beta)$, {\em i.e.,} $\nabla_j q(\beta; \beta)$ satisfies the following inequality 
\begin{equation*}
   \mathbb{E}_{y} (\nabla_j q(\beta; \beta))^2\leq O(\max\{(\|\beta^*\|^2_2+\sigma^2)^2,  d\|\beta^*\|^2_2\}). 
\end{equation*}
Also, for fixed $j\in [d]$, each $\nabla_j q_i(\beta;\beta)$ is independent with others for $i\in [n]$.  
\end{lemma}
% Combining with Lemma \ref{lemma:5}, \ref{lemma:6} and Theorem \ref{theorem:3} we have the following statistical guarantee for MRM. 
\begin{theorem}\label{theorem:5}
With the same notations as 
%notations of 
in Lemma \ref{lemma:5}, in Algorithm \ref{alg:2}  assume that  $\|\beta^0-\beta^*\|_2\leq \frac{1}{64}\|\beta^*\|_2$ and $n$ is large enough so that
\begin{equation}
\tilde{\Omega}(  \frac{d^2\max\{(\|\beta^*\|^2_2+\sigma^2)^2,  d\|\beta^*\|^2_2\} \sqrt{\log \frac{1}{\delta}} \log \frac{1}{\zeta} }{\epsilon \|\beta^*\|_2^2})\leq n. \nonumber 
\end{equation}  
Moreover, if take $T=O(\log n )$ and $\eta=O(1)$,  then we have, with probability at least $1-\zeta$,
\begin{small}
\begin{equation}\label{eq:21}
    \|\beta^T-\beta^*\|_2\leq \tilde{O}\big(\frac{d\|\beta^*\|_2 \sqrt[4]{\log \frac{1}{\delta}}\sqrt{\max\{\|\beta^*\|^2_2+\sigma^2,  d\|\beta^*\|^2_2\}}}{\sqrt{n \epsilon}} \big), 
\end{equation}
\end{small}
where the $\tilde{O}$-term and $\tilde{\Omega}$-term  omit logarithmic factors.
\end{theorem}
\begin{remark}
If we assume that  $\|\beta^*\|$ and $\sigma=O(1)$, then the error in (\ref{eq:21}) is upper bounded by $\tilde{O}(\frac{d^{\frac{3}{2}}}{\sqrt{n\epsilon}})$, which has an additional factor of $\sqrt{d}$ compared with the bound in (\ref{eq:20}) for GMM. We note that this is the first statistical result for MRM in the  DP model. 
%Motivated by the result in GMM, we conjecture it is near optimal when $\epsilon=O(1)$.
\end{remark}

%\vspace{-0.1in}
\subsection{Linear Regression with Missing Covariates}

\begin{lemma}[\citep{balakrishnan2017statistical,yi2015regularized}]\label{lemma:7}
%Suppose 
If $\frac{\|\beta^*\|_2}{\sigma}\leq r$  and 
%Assuming that we have 
    $p_m<\frac{1}{1+2b+2b^2}$,
where  $r$ is a constant denoting the required maximum signal-to-noise ratio (SNR) and   $b=r^2(1+k)^2$ for some constant $k\in (0,1)$, then  function $Q(\cdot; \cdot)$ of the linear regression with missing covariates has the properties of self-consistent, Lipschitz-Gradient-2($\gamma, \mathcal{B})$, $\mu$-smoothness and $\upsilon$-strongly with 
\begin{align}
    &\gamma =\frac{b+p_m(1+2b+2b^2)}{1+b}<1, \mu=\upsilon=1,\nonumber \\ &\mathcal{B}=\{\beta:\|\beta-\beta^*\|_2\leq R\}, \text{ where } R=k\|\beta^*\|_2. \nonumber
\end{align}
\end{lemma}

\begin{lemma}\label{lemma:8}
With the same assumptions as in Lemma \ref{lemma:7}, for each $\beta\in \mathcal{B}$ and $j\in [d]$ ,  $\nabla_j q(\beta; \beta)$ satisfies
%%\vspace{-0.1in}
\begin{align}
\mathbb{E}(\nabla_j q(\beta; \beta))^2\leq  O((\sqrt{d}\|\beta^*\|_2+ \sigma^2+\|\beta^*\|_2^2)^2).
\end{align}
Also, for fixed $j\in [d]$, each $\nabla_j q_i(\beta;\beta)$, where $i\in [n]$, is independent with others.
\end{lemma}
\begin{theorem}\label{theorem:6}
With the same notations as 
%notations of 
in Lemma \ref{lemma:7}, in Algorithm \ref{alg:2}  assume that  $\|\beta^0-\beta^*\|_2\leq \frac{k}{2}\|\beta^*\|_2$ and $n$ is large enough so that
\begin{equation}
\tilde{\Omega}(  \frac{d^2(\sqrt{d}\|\beta^*\|_2+ \sigma^2+\|\beta^*\|_2^2)^2\log \frac{1}{\delta} \log \frac{1}{\zeta} }{\epsilon^2 \|\beta^*\|_2^2})\leq n. \nonumber 
\end{equation}  
Moreover, if  take $T=O(\log n )$ and $\eta=O(1)$,  then we have, with probability at least $1-2T\zeta$, 
\begin{equation}
    \|\beta^T-\beta^*\|_2\leq \tilde{O}\big(\frac{d\log \frac{1}{\delta}\log \frac{1}{\zeta}\|\beta^*\|_2(\sqrt{d}\|\beta^*\|_2+ \sigma^2+\|\beta^*\|_2^2)}{\sqrt{n \epsilon^2}} \big),  \nonumber
\end{equation}
where the $\tilde{O}, \tilde{\Omega}$ terms omit logarithmic and other factors.
\end{theorem}
Note that unlike the previous two models, we assume here that SNR is upper bounded by some constant which is unavoidable as pointed out in \citep{loh2011high}.

\section{Experiments}\label{sec:exp}

In this section, we evaluate the performance of Algorithm \ref{alg:2} on three canonical models: GMM, MRM, and RMC. Since in the paper we mainly focus on the statistical setting and its theoretical behaviors, { we evaluate our algorithm on both the synthetic data and the real world datasets\footnote{\url{http://archive.ics.uci.edu/ml/datasets/Adult}, \url{http://international.ipums.org}}: ADULT, IPUMS-BR and IPUMS-US.}\\
%Note that previous papers on the statistical guarantees of EM algorithm all evaluating their algorithms on synthetic data only such as \citep{balakrishnan2017statistical,yi2015regularized,zhu2017high}. Thus, evaluating experiments on synthetic data only is sufficient and reasonable  for the paper.

\noindent{\bf Baseline Methods}\quad We compare our approach against two baseline algorithms. One is the Gradient EM algorithm \citep{balakrishnan2017statistical}, namely, EM, as our non-private baseline method. The other is clipped DP Gradient EM (Algorithm \ref{alg:1}), namely, clipped, as our private baseline method. As we mentioned previously, previous DP EM method in \citep{park2017dp} needs strong assumptions on the model and the data itself to ensure DP property, which do not hold for our models. Thus we can not compare with their method.\\

\noindent{\bf Experimental Settings }\quad 
For each of these models, we generate synthesized datasets according to the underlying distribution. We also utilize $\|\beta - \beta^*\|_2$ to measure the estimation error. Instead of choosing the initial parameter $\beta^0$ that is close to the optimal one, we consider random initialization. As we will see later, even if we select random initial parameter, the performance of our private estimator is good enough. We set signal-to-noise ratio $\frac{\|\beta^*\|_2}{\sigma} = 3$.
For the privacy parameters, we choose $\epsilon = \{0.2,0.5,1\}$ and $\delta = \frac{1}{n^{1.1}}$. 
%For Algorithm \ref{alg:1}, we consider the clipping threshold $C = \{0.1,0.5,1,2,5,10\}$.
%All experiments are repeated for 50 independent trials and the average results are reported.
We also conduct experiments on three real world datasets: ADULT, IPUMS-BR and IPUMS-US. The ADULT dataset includes of 48,842 data samples. The target is to predict whether the annual income is more than \$50k or not. The IPUMS-BR and  IPUMS-US are from IPUMS-International, which include 38,000 and 40,000 data samples of census microdata. The goal is to predict whether the monthly income of an individual is above \$300 or not. To fit the real dataset into the GMM (\ref{eq:4}), we process the data as following. First, based the target we divide the whole dataset into two clusters and take the same amount of data record for each cluster, and then for each part we calculate the covariance matrix and its maximal eigenvalue  as the $\sigma$ in (\ref{eq:4}). To get the $\beta^*$, we first calculate the mean of each cluster, then we calculate their midpoint. Next we transit all the data records along this midpoint. After transition, the mean a cluster will be the $\beta^*$. \\

\noindent{\bf Experimental Results}\quad 
Firstly, we will show that the performance of Algorithm \ref{alg:1} is heavily affected by the clipping threshold $C$. As shown in Figure \ref{fig:clipped_em}, we conduct the algorithm on three canonical models with fixed data size $n$, dimension data $d$, and privacy budget $\epsilon$. If $C$ is set to be a small value (e.g., 0.1), it significantly reduces the adding noise in each iteration but at the same time it leads much information loss in gradient estimation. Conversely, if $C$ is set too high (e.g., 5 or 10), the noise variance becomes high, resulting in introducing too much noise to the estimation.
Thus, selecting the optimal $C$ is quite difficult since too large or too small values of $C$ has a negative effect on the performance of Algorithm \ref{alg:1}.
Even for $C = 1$ that achieves lowest estimation error among other threshold values, the estimation error does not decay as the number of iterations increases, whereas under the same privacy guarantee, our proposed algorithm achieves the same convergence behavior as EM, and thoroughly outperforms Algorithm \ref{alg:1}.
For fair comparison, we fixed $C = 1$ for Algorithm \ref{alg:1} in the  following experiments.

In Figure \ref{fig:comp_gmm}, \ref{fig:comp_mrm} and \ref{fig:comp_rmc}, we show the estimation error $\|\beta - \beta^*\|_2$ of all algorithms on three canonical models over iteration $t$ with different  
privacy budget $\epsilon$, data dimension $d$ and sample size $n$. We can see that the estimation error decreases and will converge to some  error when the iteration number is larger, which has been shown in our previous theoretical analysis. 
We can also see that with a fixed iteration number,  the estimation error of our proposed algorithm in each of the three models becomes smaller with a larger $\epsilon, n$ or a smaller $d$, which are consistent with our theoretical results. In these figures, our algorithm exhibits nearly the same convergence behavior as the non-private baseline method and outperforms Algorithm \ref{alg:1}.

In Figure \ref{fig:comp_gmm2}, \ref{fig:comp_mrm2} and \ref{fig:comp_rmc2}, we conduct more experiments on these three canonical models for  synthetic data with different privacy level $\epsilon$, dimensionality $d$ and sample size $n$. In these algorithms we choose the best iteration number $T$ in Theorem \ref{thm:new} and compute the estimation error on $\beta \coloneqq  \beta^T$. We can see that the estimation error of our proposed algorithm in each of the three models decreases when $\epsilon$ increases, $d$ decreases or $n$ increases, which are consistent with our theoretical results,  since in the previous section we showed that the estimation error is upper bounded by $\tilde{O}(\frac{d\sqrt{\tau}}{\sqrt{n\epsilon}})$. 

Finally, we conduct experimental results on real-world datasets. { We  present the estimation error of different algorithms on GMM model. As shown in Figure \ref{fig:real}, we can observe that our proposed algorithm still outperforms the baseline algorithms under different privacy budgets.} Moreover, even in the case where $\epsilon$ is low, the estimation error is still quite close the non-private EM algorithm as the iteration number grows, which shows the effectiveness of our algorithm.

\begin{figure*}[!htbp]\centering
 \subfigure[GMM, $n=1000,d = 20,\epsilon = 0.2$\label{clipped_gmm}]
 {\includegraphics[trim=0.05in 0 0.8in 0.7in,clip,width=2.2in]{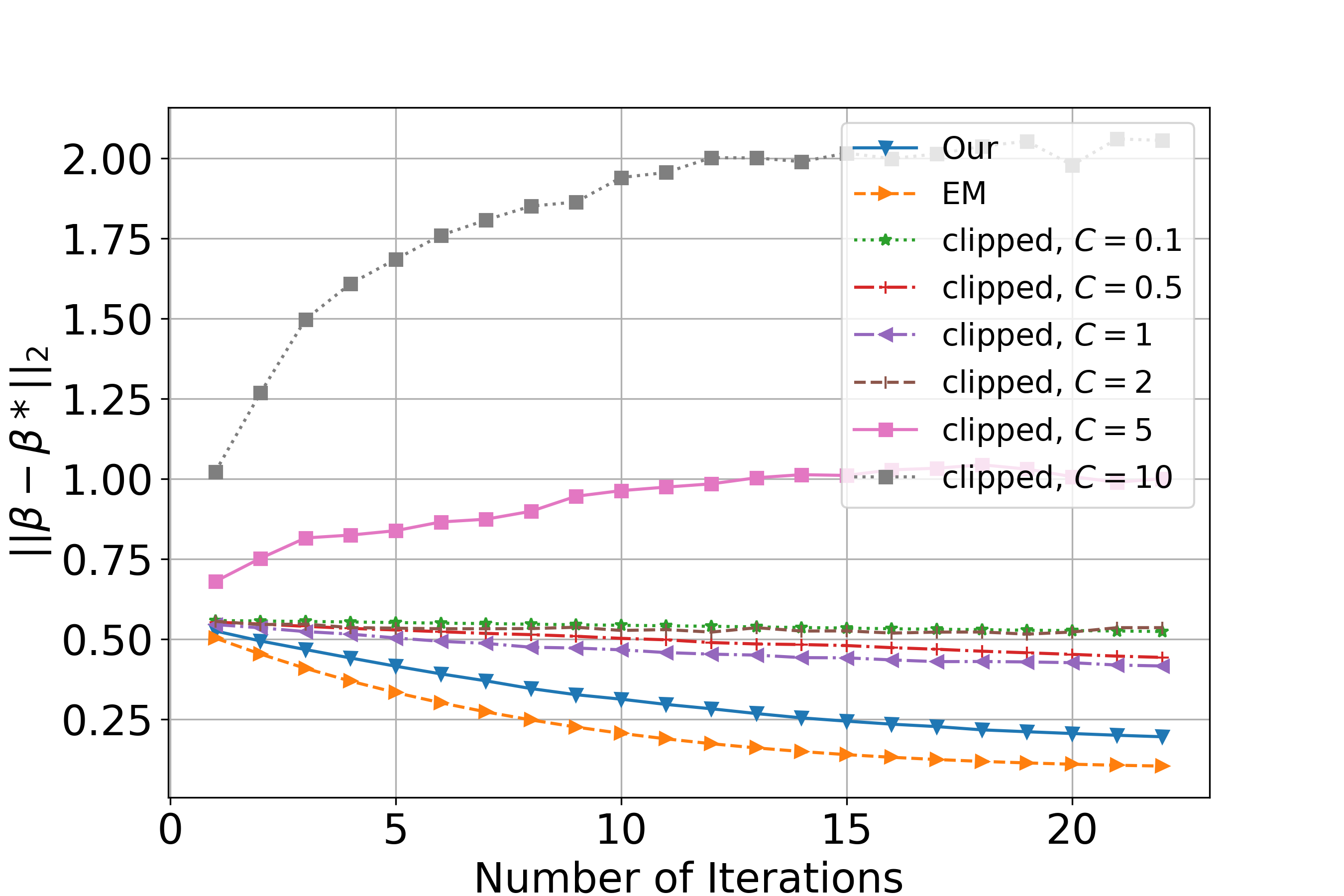}}
\subfigure[MRM, $n=1000,d = 20,\epsilon = 0.2$ \label{clipped_mrm}]
  {\includegraphics[trim=0.05in 0 0.8in 0.7in,clip,width=2.2in]{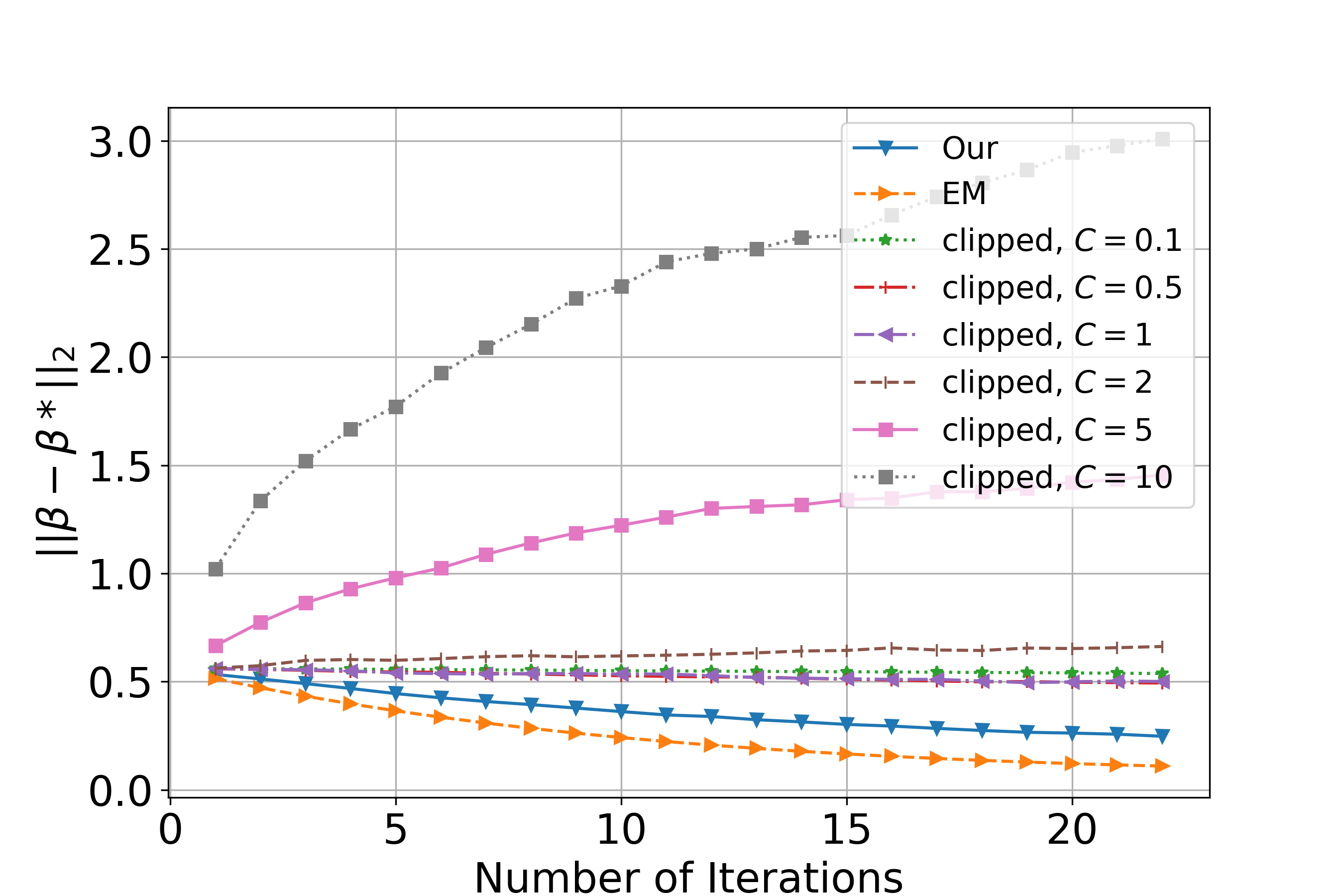}}
     \subfigure[RMC, $n=1000,d = 20,\epsilon = 0.2$\label{clipped_rmc}]
  {\includegraphics[trim=0.05in 0 0.8in 0.7in,clip,width=2.2in]{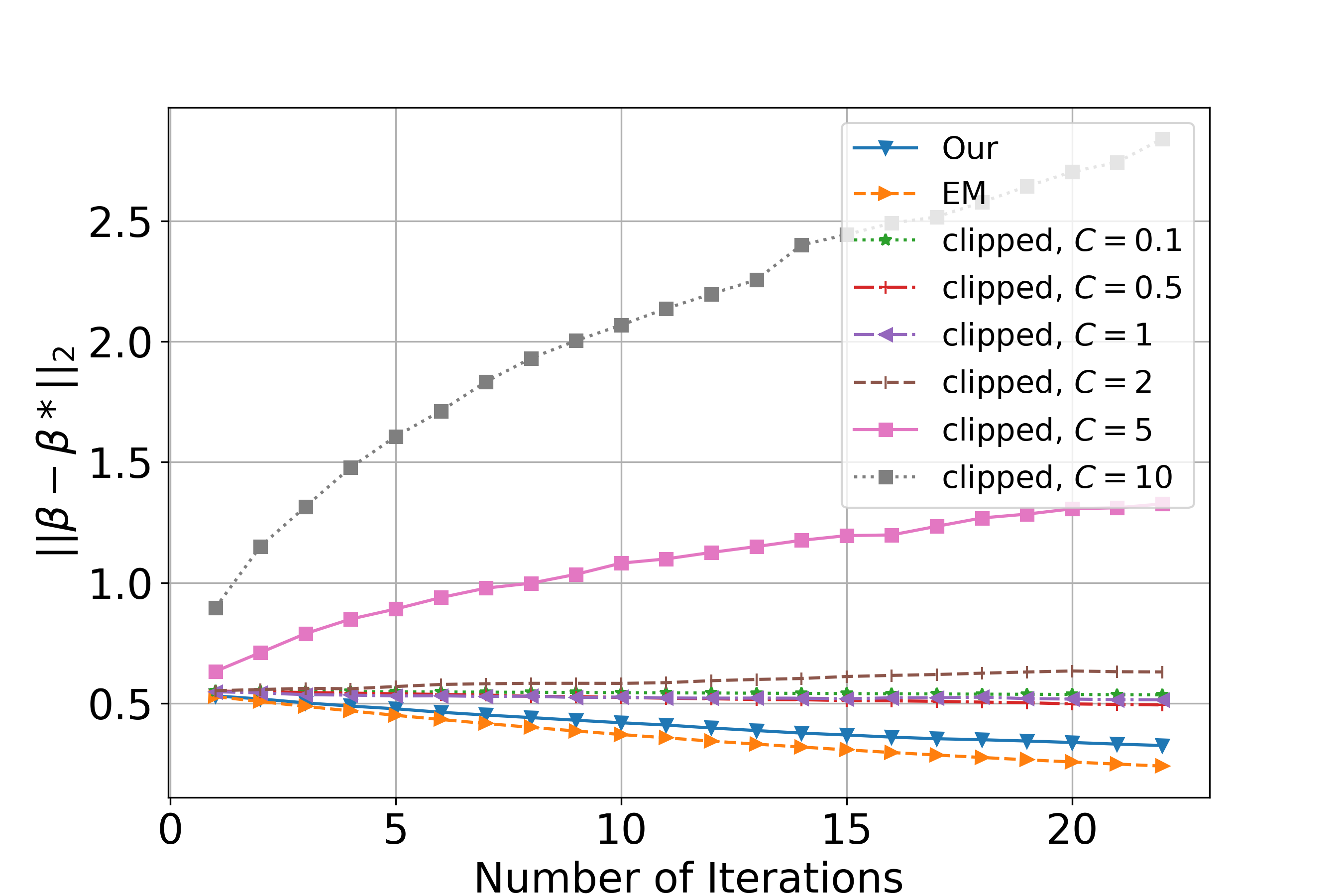}}\vspace{-0.1in}
 \caption{Estimation error of Algorithm 1 (clipped) v.s. iteration $t$ under different clipping threshold $C$} \label{fig:clipped_em}\vspace{-0.1in}
\end{figure*}

\begin{figure*}[!htbp]\centering
 \subfigure[$n=2000,d = 10$\label{Gmm_budget}]
 {\includegraphics[trim=0.05in 0 0.8in 0.7in,clip,width=2.2in]{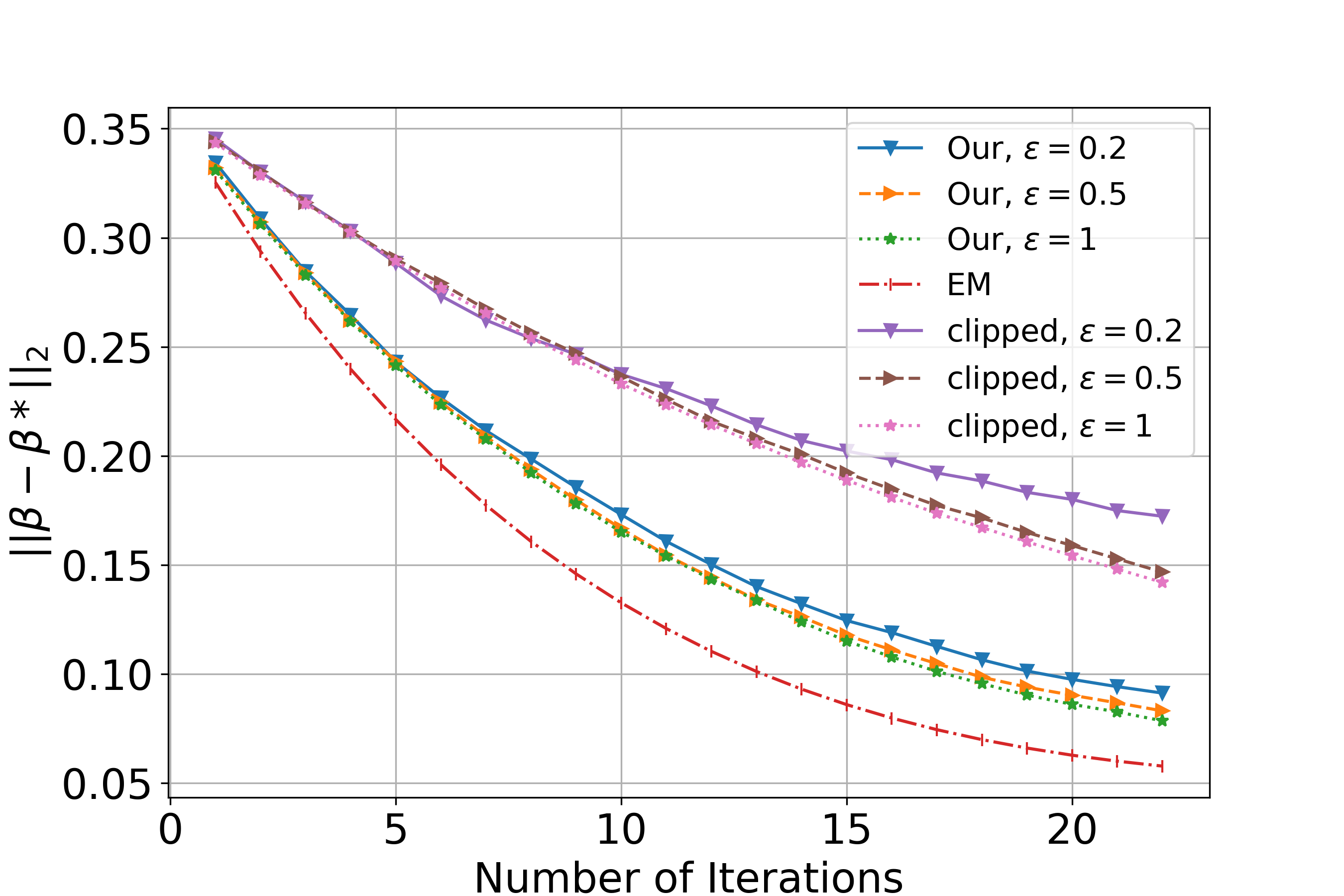}}
\subfigure[$n=2000, \epsilon = 0.5$ \label{Gmm_dim}]
  {\includegraphics[trim=0.05in 0 0.8in 0.7in,clip,width=2.2in]{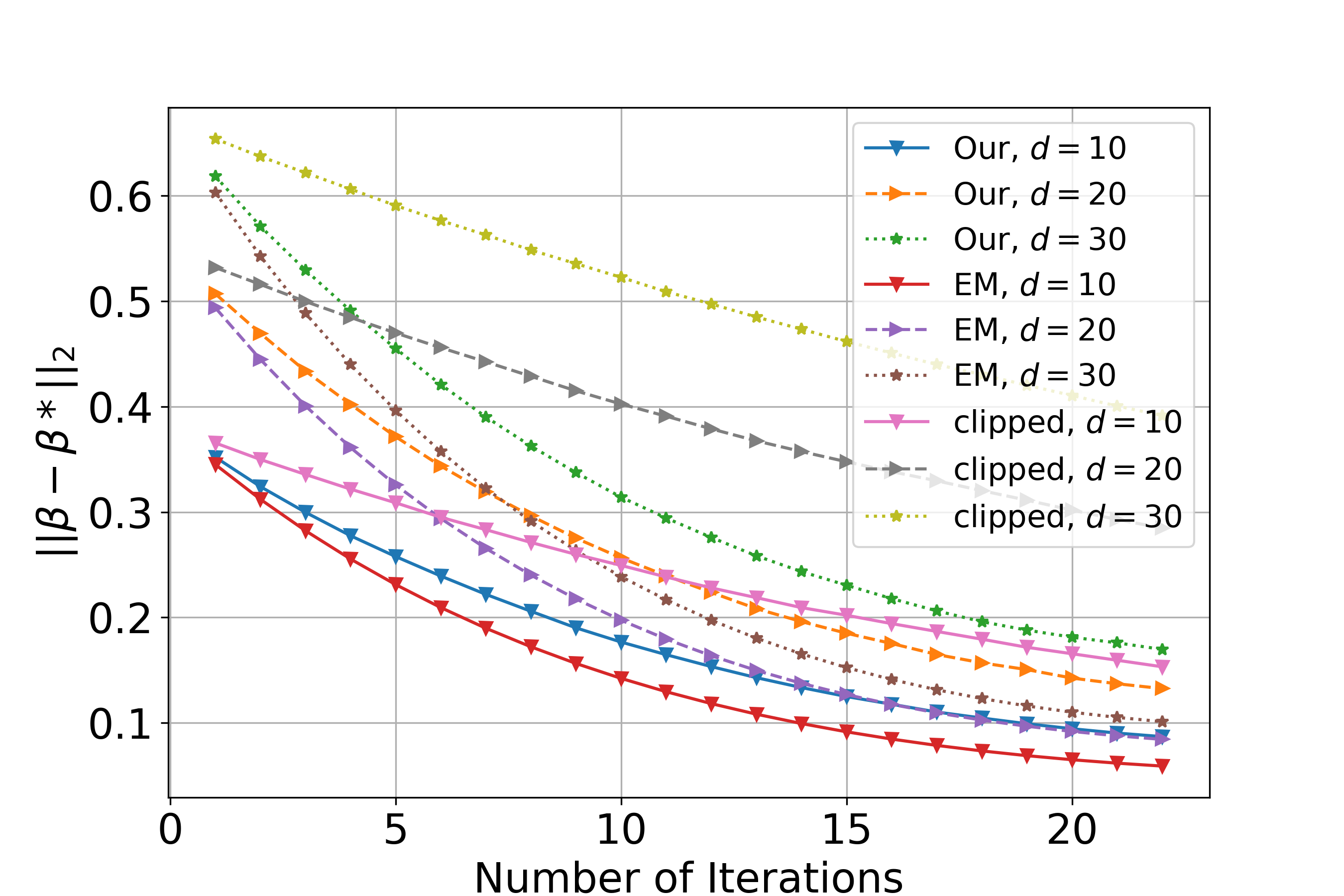}}
     \subfigure[$d=10, \epsilon =0.5$\label{Gmm_samples}]
  {\includegraphics[trim=0.05in 0 0.8in 0.7in,clip,width=2.2in]{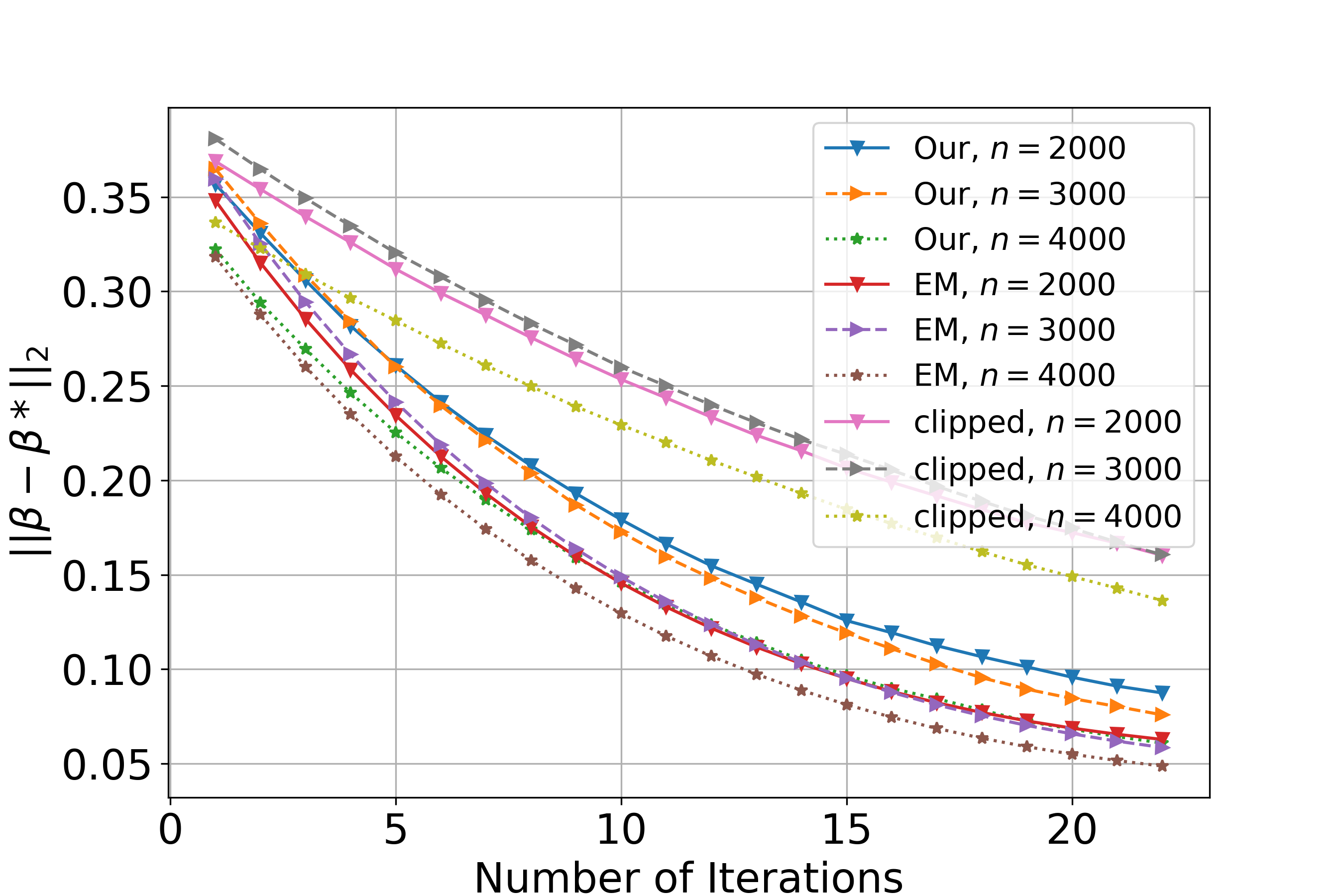}} \vspace{-0.1in}
 \caption{ Estimation error of GMM w.r.t privacy budget $\epsilon$, data dimension $d$, data size $n$ and iteration $t$} \label{fig:comp_gmm}\vspace{-0.1in}
\end{figure*}

\begin{figure*}[!htbp]\centering
 \subfigure[$n=2000,d = 10$\label{mrm_budget}]
 {\includegraphics[trim=0.05in 0 0.8in 0.7in,clip,width=2.2in]{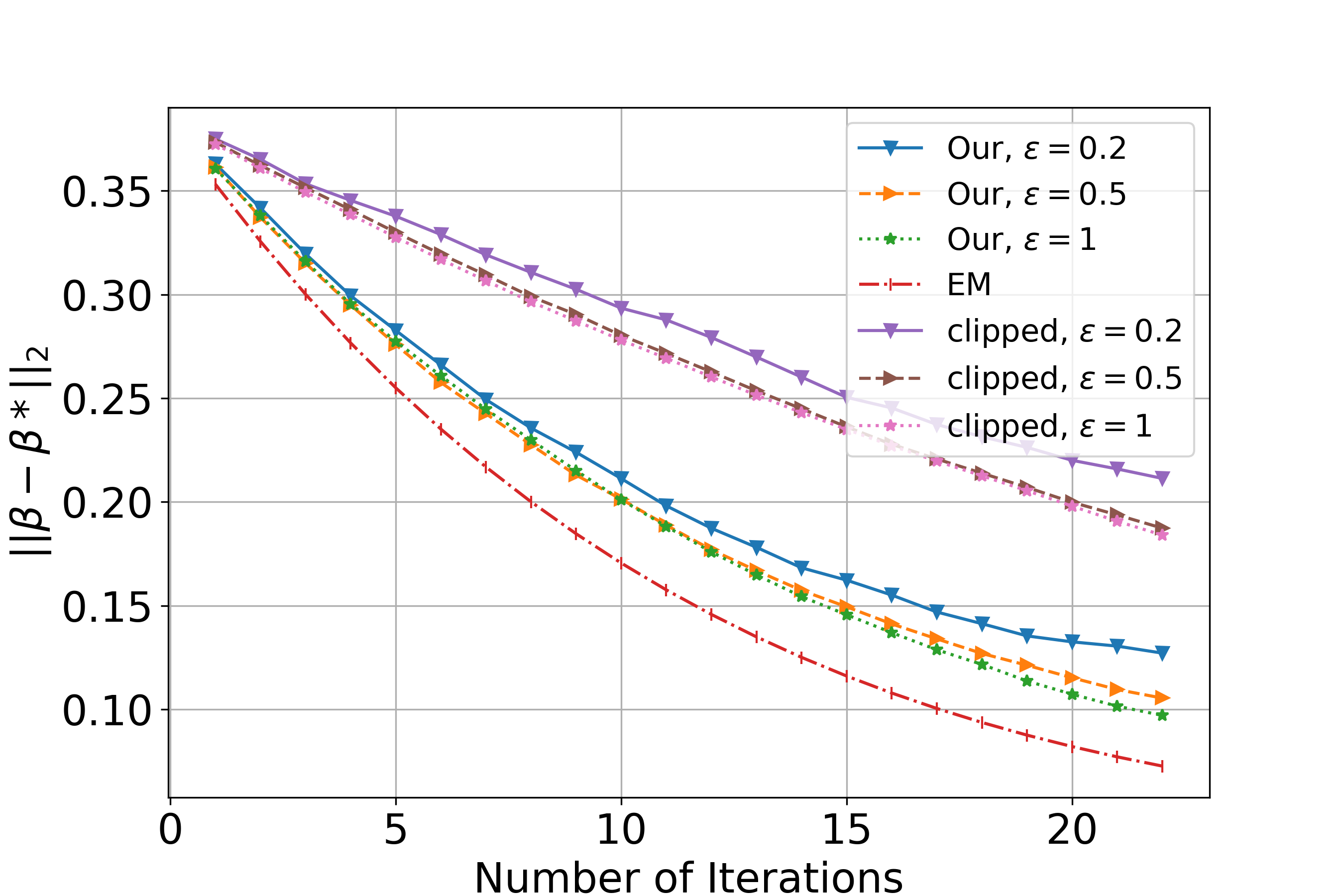}}
\subfigure[$n=2000, \epsilon = 0.5$ \label{mrm_dim}]
  {\includegraphics[trim=0.05in 0 0.8in 0.6in,clip,width=2.2in]{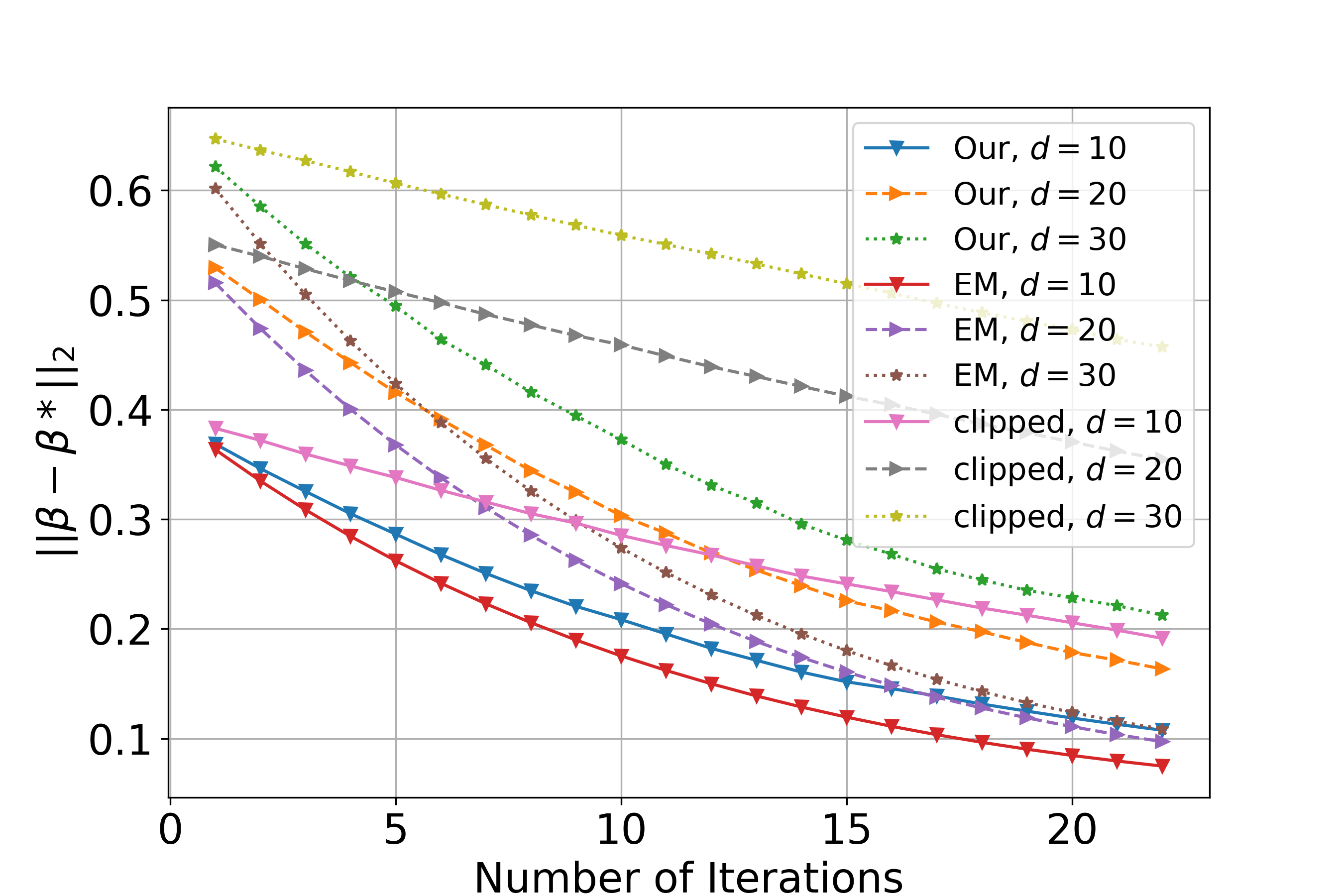}}
     \subfigure[$d=10, \epsilon =0.5$\label{mrm_samples}]
  {\includegraphics[trim=0.05in 0 0.8in 0.6in,clip,width=2.2in]{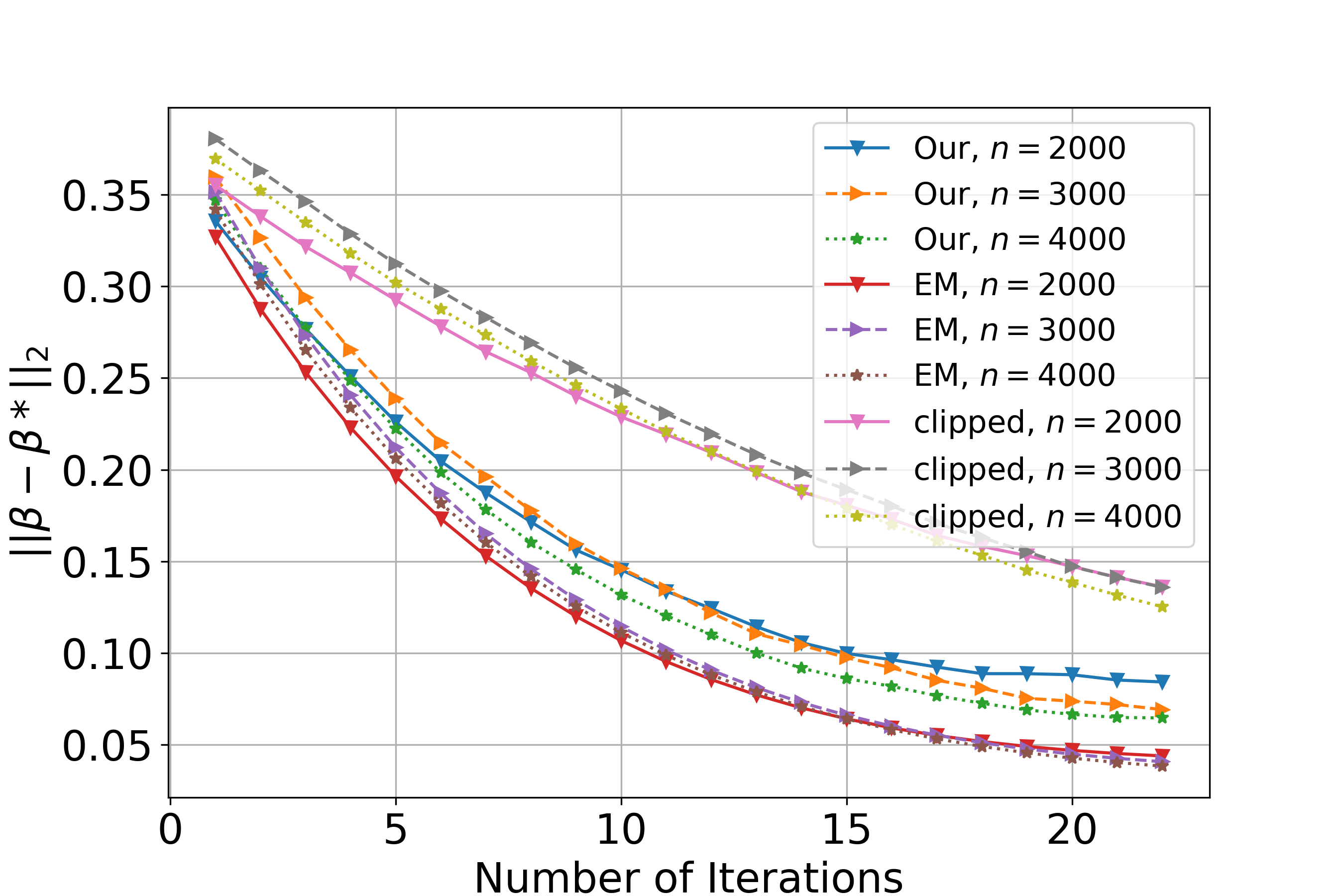}}\vspace{-0.1in}
 \caption{Estimation error of MRM w.r.t privacy budget $\epsilon$, data dimension $d$, data size $n$ and iteration $t$} \label{fig:comp_mrm}\vspace{-0.1in}
\end{figure*}
\begin{figure*}[!htbp]\centering
 \subfigure[$n=2000,d = 10$\label{rmc_budget}]
 {\includegraphics[trim=0.05in 0 0.8in 0.7in,clip,width=2.2in]{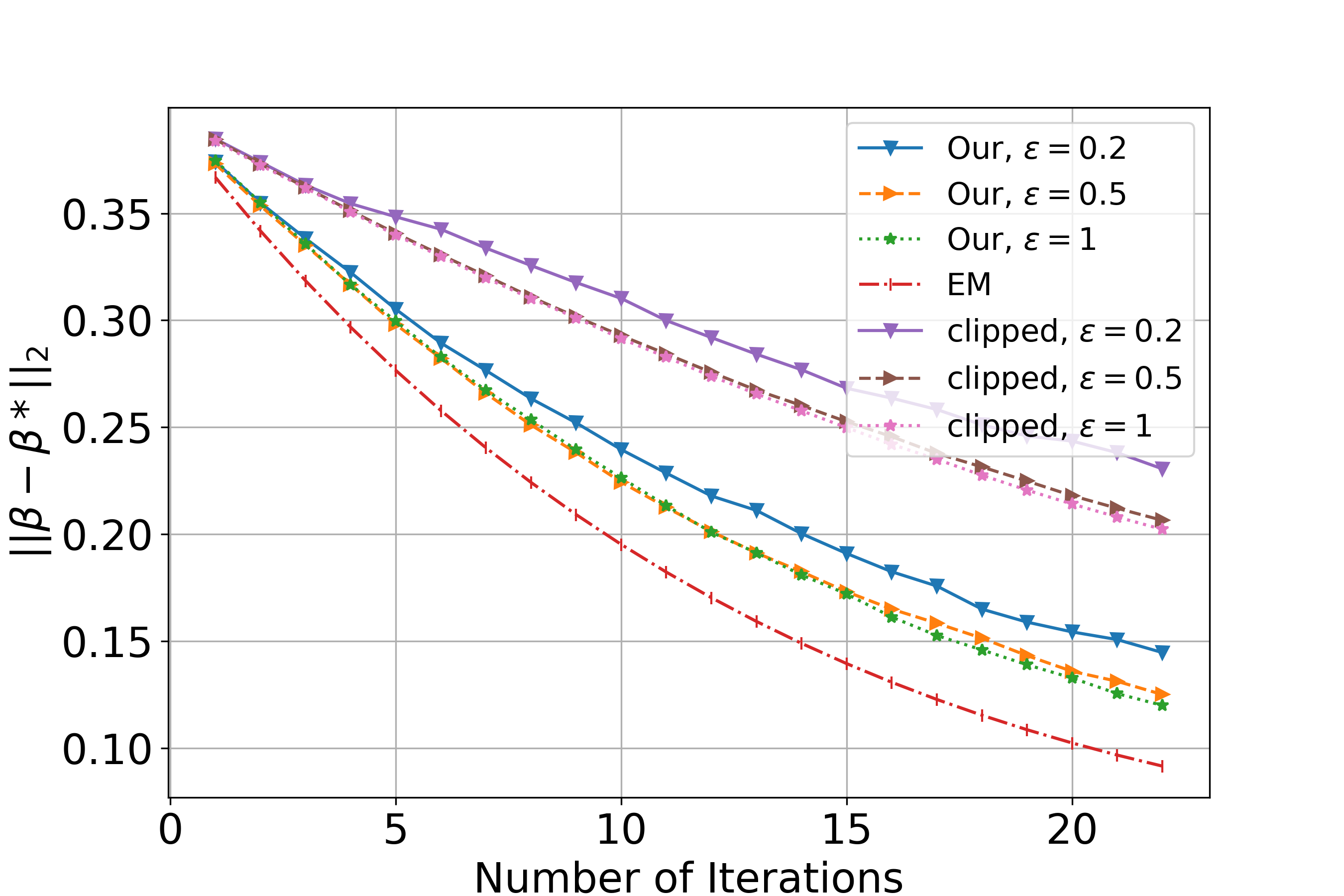}}
\subfigure[$n=2000, \epsilon = 0.5$ \label{rmc_dim}]
  {\includegraphics[trim=0.05in 0 0.8in 0.6in,clip,width=2.2in]{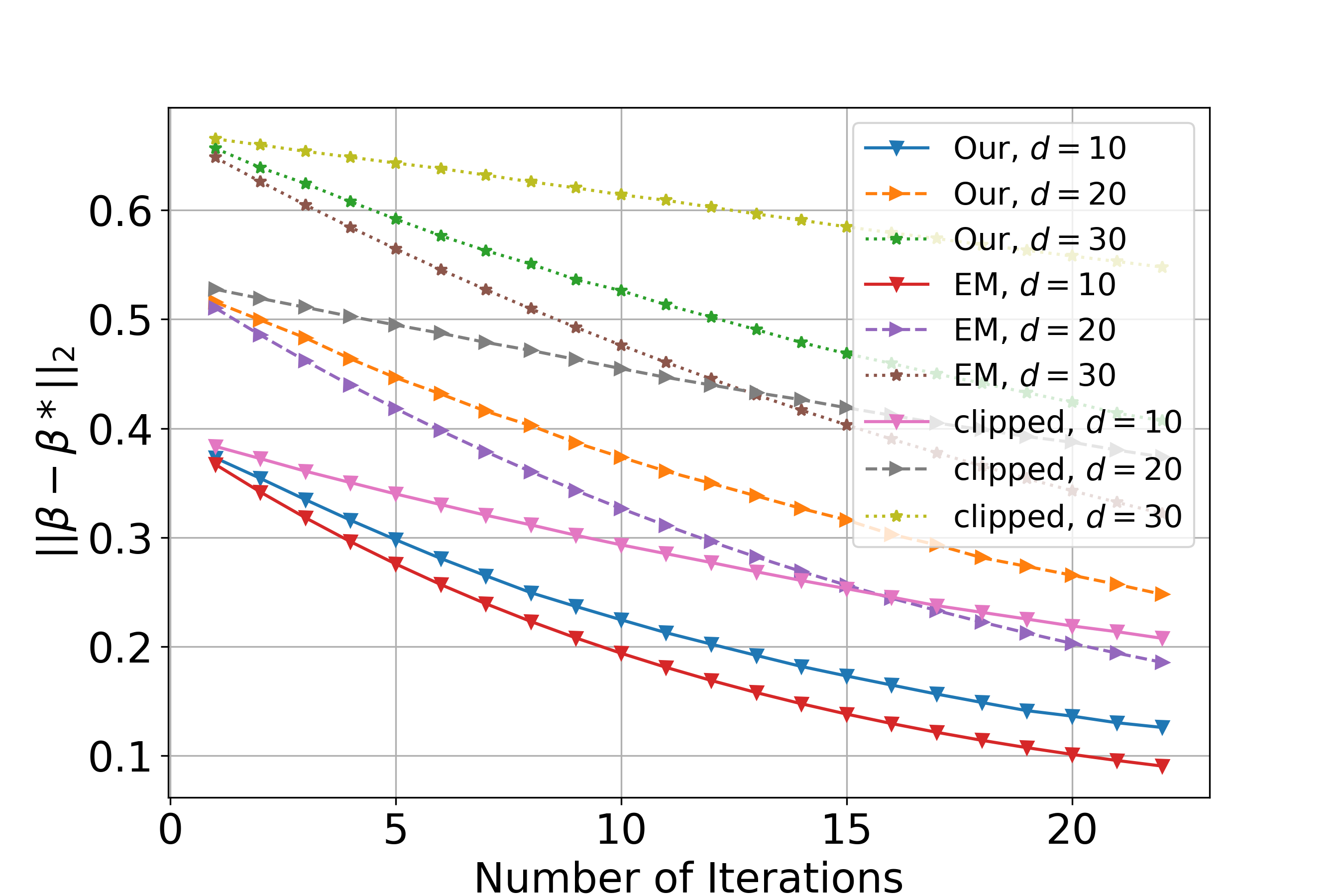}}
     \subfigure[$d=10, \epsilon =0.5$\label{rmc_samples}]
  {\includegraphics[trim=0.05in 0 0.8in 0.7in,clip,width=2.2in]{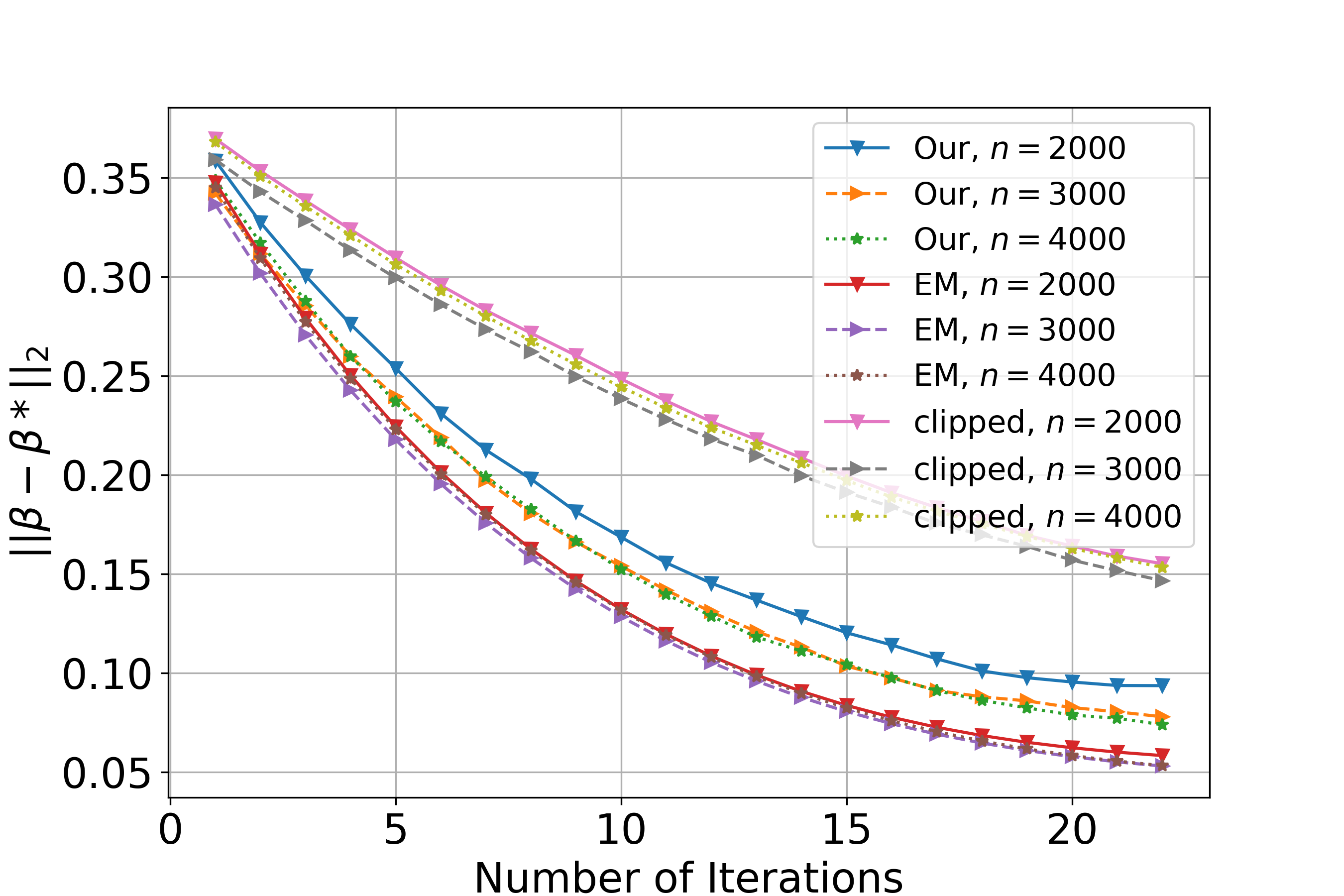}}\vspace{-0.1in}
 \caption{Estimation error of RMC w.r.t privacy budget $\epsilon$, data dimension $d$, data size $n$ and iteration $t$} \label{fig:comp_rmc}\vspace{-0.1in}
\end{figure*}

\begin{figure*}[!htbp]\centering
 \subfigure[$n=2000,d = 10$\label{Gmm_budget2}]
 {\includegraphics[trim=0.05in 0 0.8in 0.7in,clip,width=2.2in]{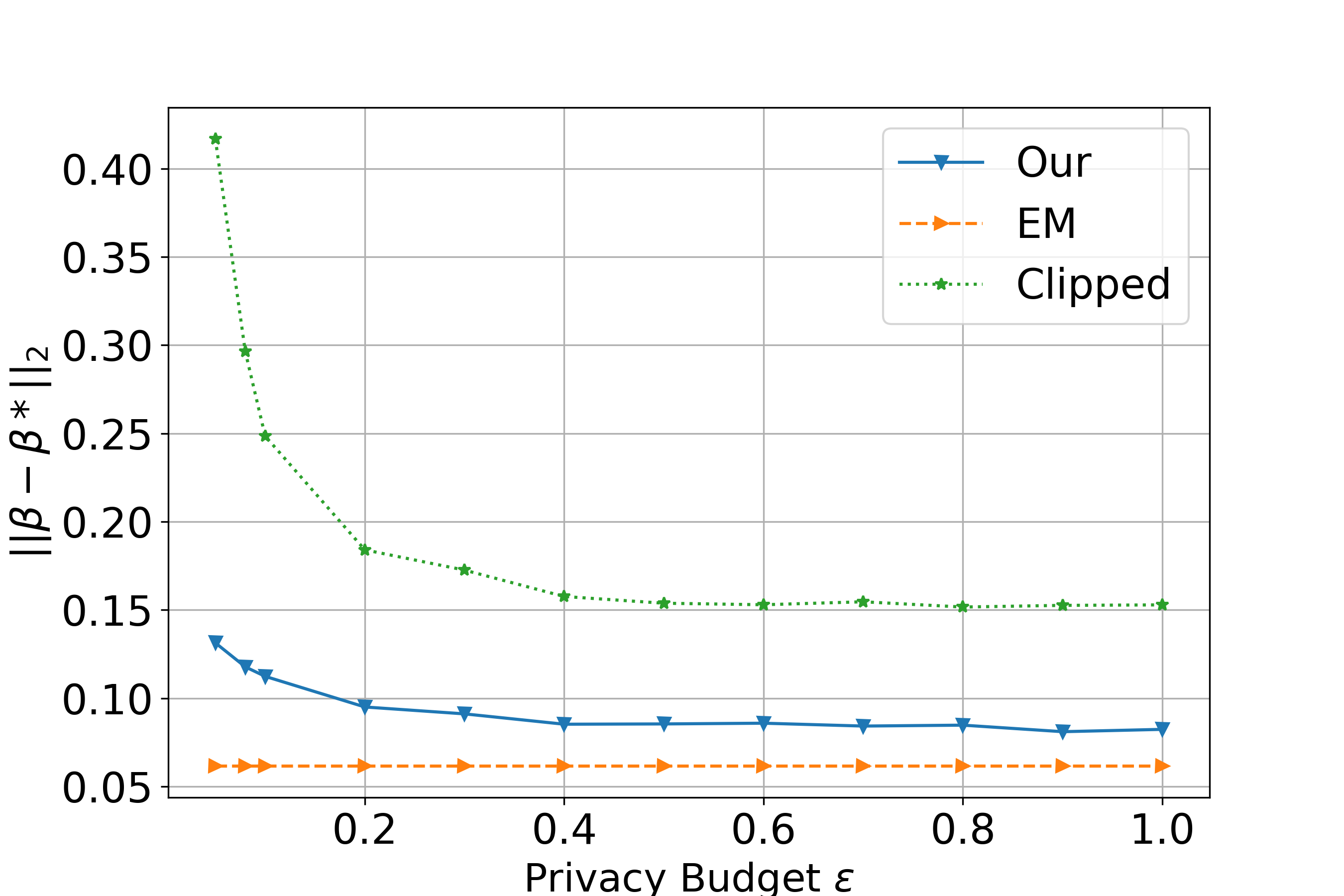}}
\subfigure[$n=2000, \epsilon = 0.5$ \label{Gmm_dim2}]
  {\includegraphics[trim=0.05in 0 0.8in 0.7in,clip,width=2.2in]{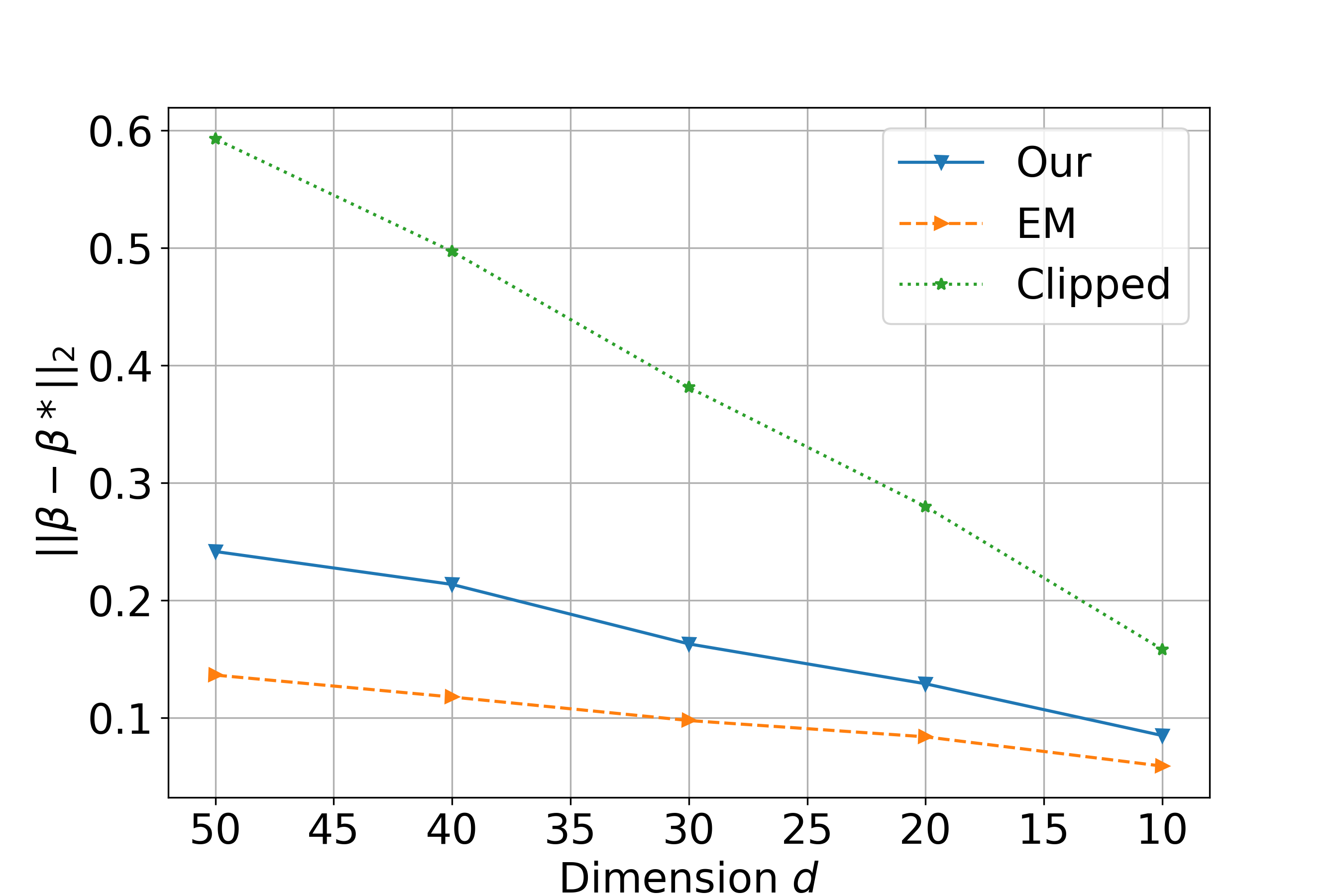}}
     \subfigure[$d=10, \epsilon =0.5$\label{Gmm_samples2}]
  {\includegraphics[trim=0.05in 0 0.8in 0.7in,clip,width=2.2in]{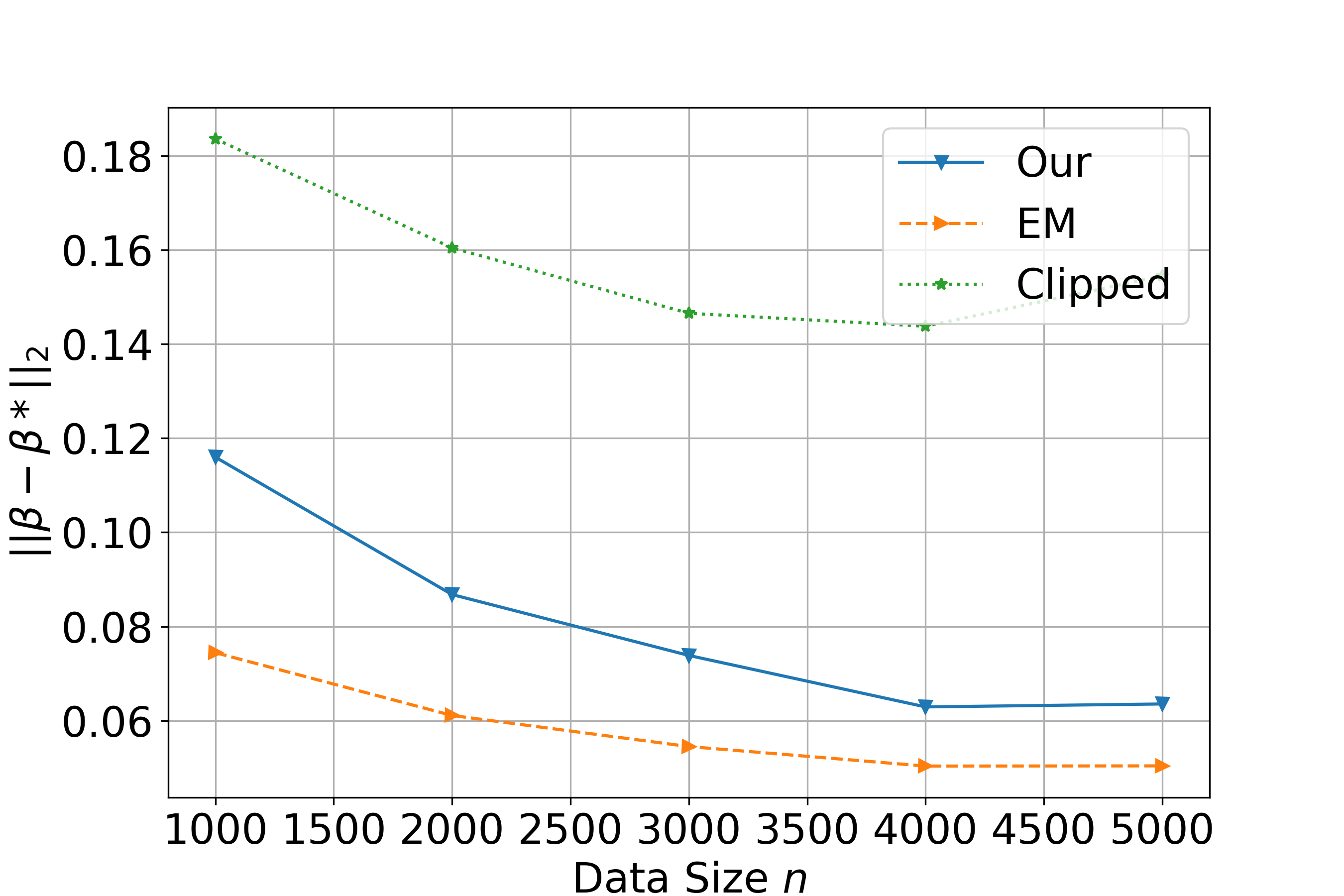}}
 \caption{Estimation error of GMM w.r.t privacy budget $\epsilon$, data dimension $d$ and data size $n$} \label{fig:comp_gmm2}
\end{figure*}
\begin{figure*}[!htbp]\centering
 \subfigure[$n=2000,d = 10$\label{mrm_budget2}]
 {\includegraphics[trim=0.05in 0 0.8in 0.7in,clip,width=2.2in]{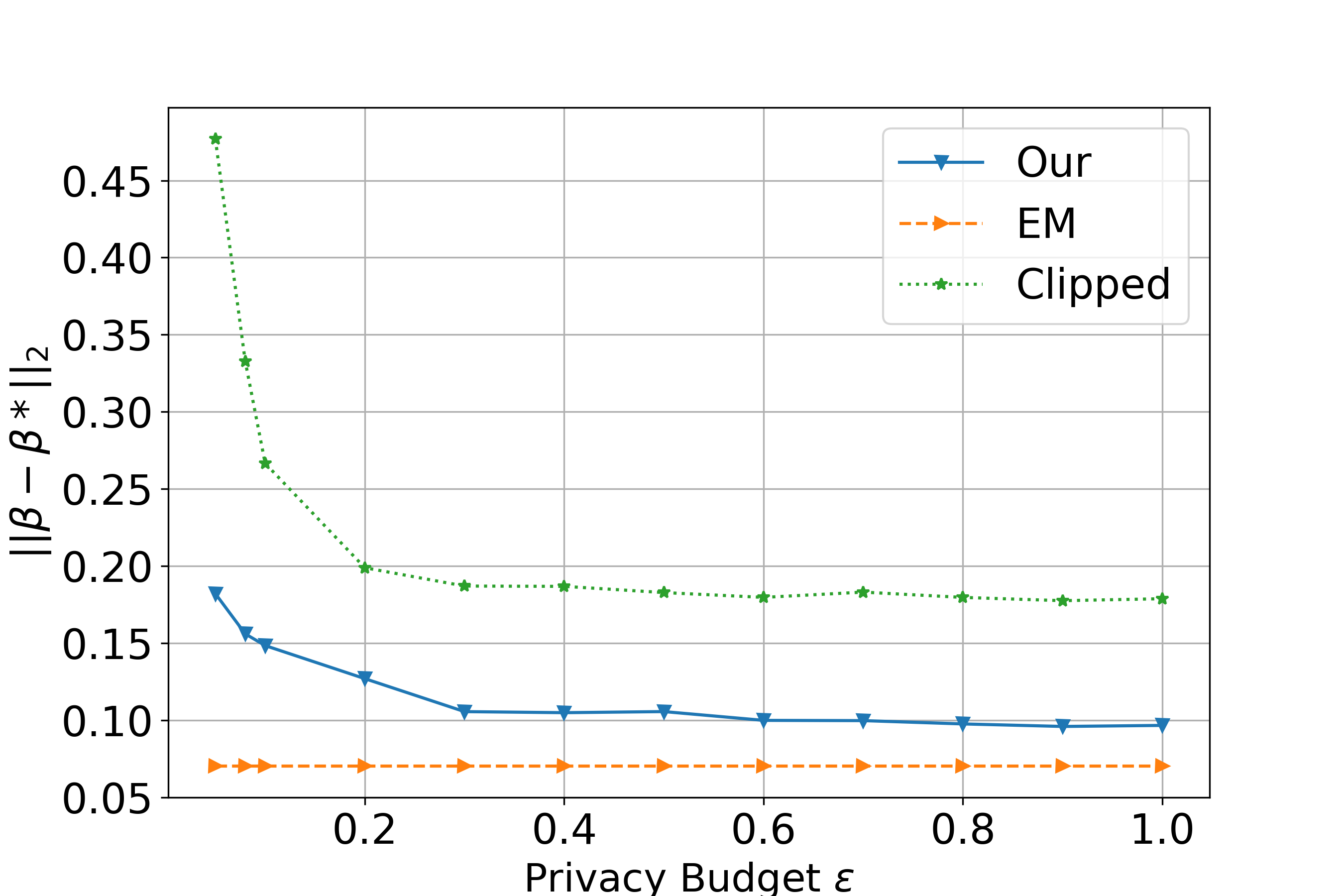}}
\subfigure[$n=2000, \epsilon = 0.5$ \label{mrm_dim2}]
  {\includegraphics[trim=0.05in 0 0.8in 0.7in,clip,width=2.2in]{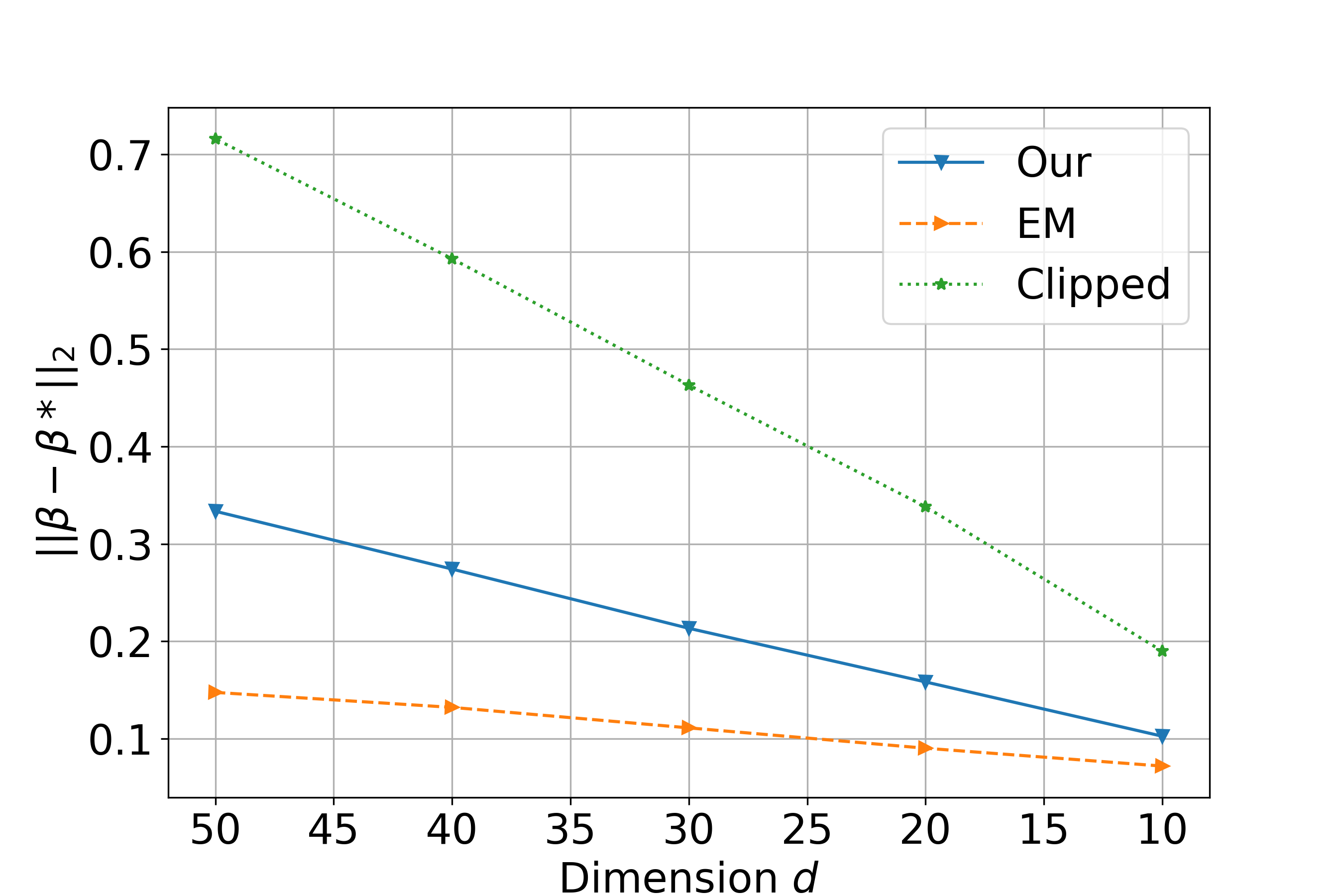}}
     \subfigure[$d=10, \epsilon =0.5$\label{mrm_samples2}]
  {\includegraphics[trim=0.05in 0 0.8in 0.7in,clip,width=2.2in]{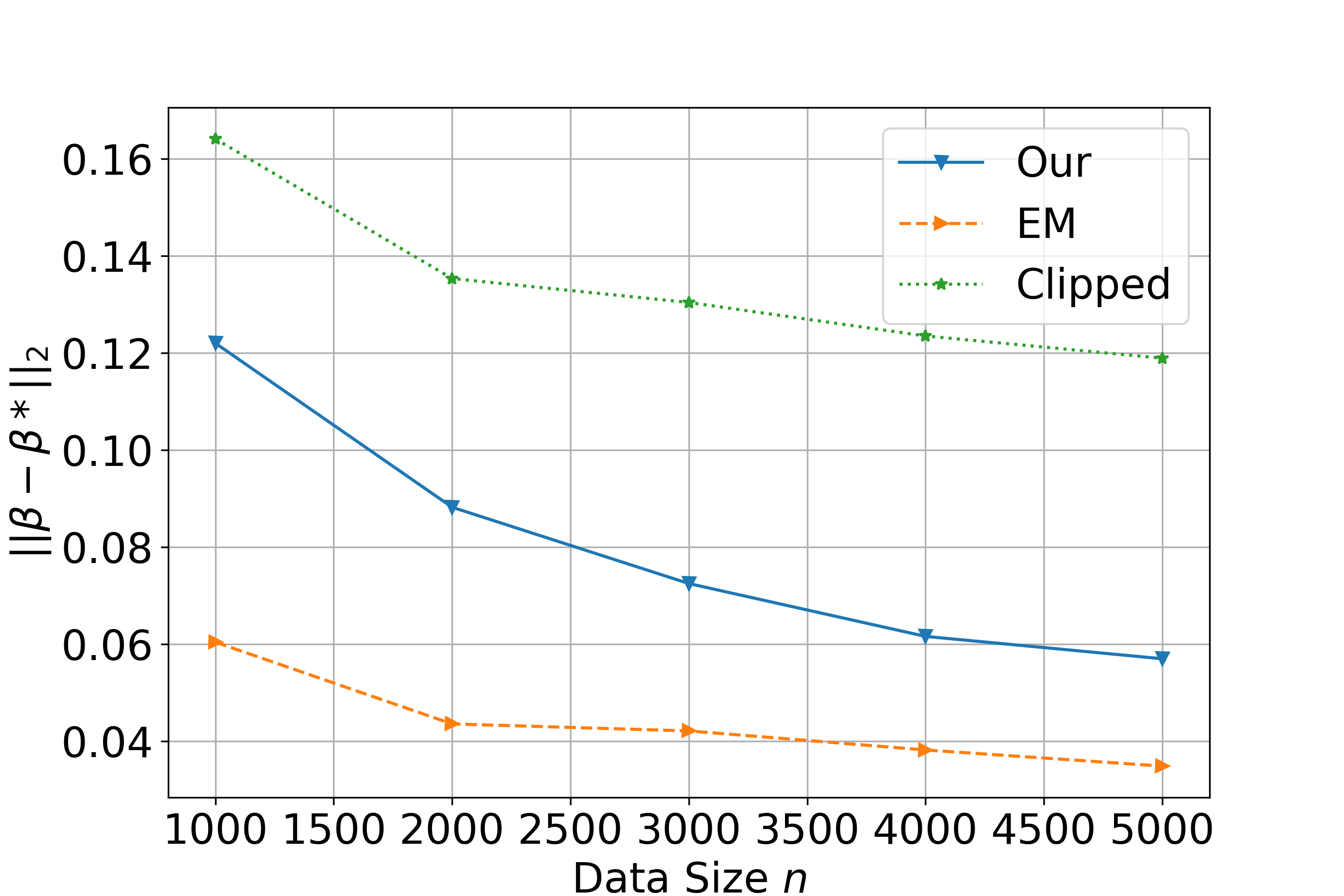}}
 \caption{Estimation error of MRM w.r.t privacy budget $\epsilon$, data dimension $d$ and data size $n$} \label{fig:comp_mrm2}
\end{figure*}

\begin{figure*} [!htbp]\centering
 \subfigure[$n=2000,d = 10$\label{rmc_budget2}]
 {\includegraphics[trim=0.05in 0 0.8in 0.7in,clip,width=2.2in]{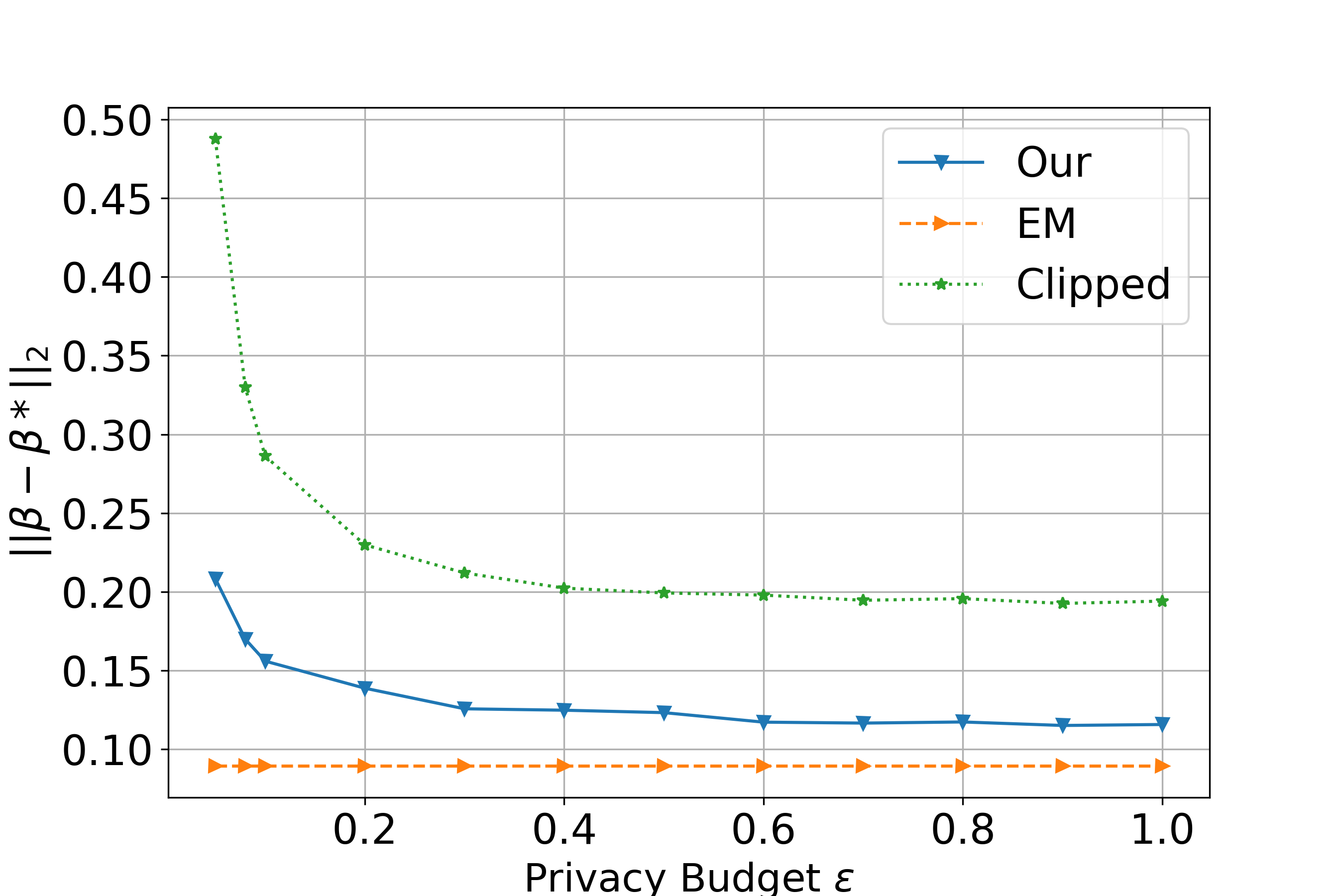}}
\subfigure[$n=2000, \epsilon = 0.5$ \label{rmc_dim2}]
  {\includegraphics[trim=0.05in 0 0.8in 0.7in,clip,width=2.2in]{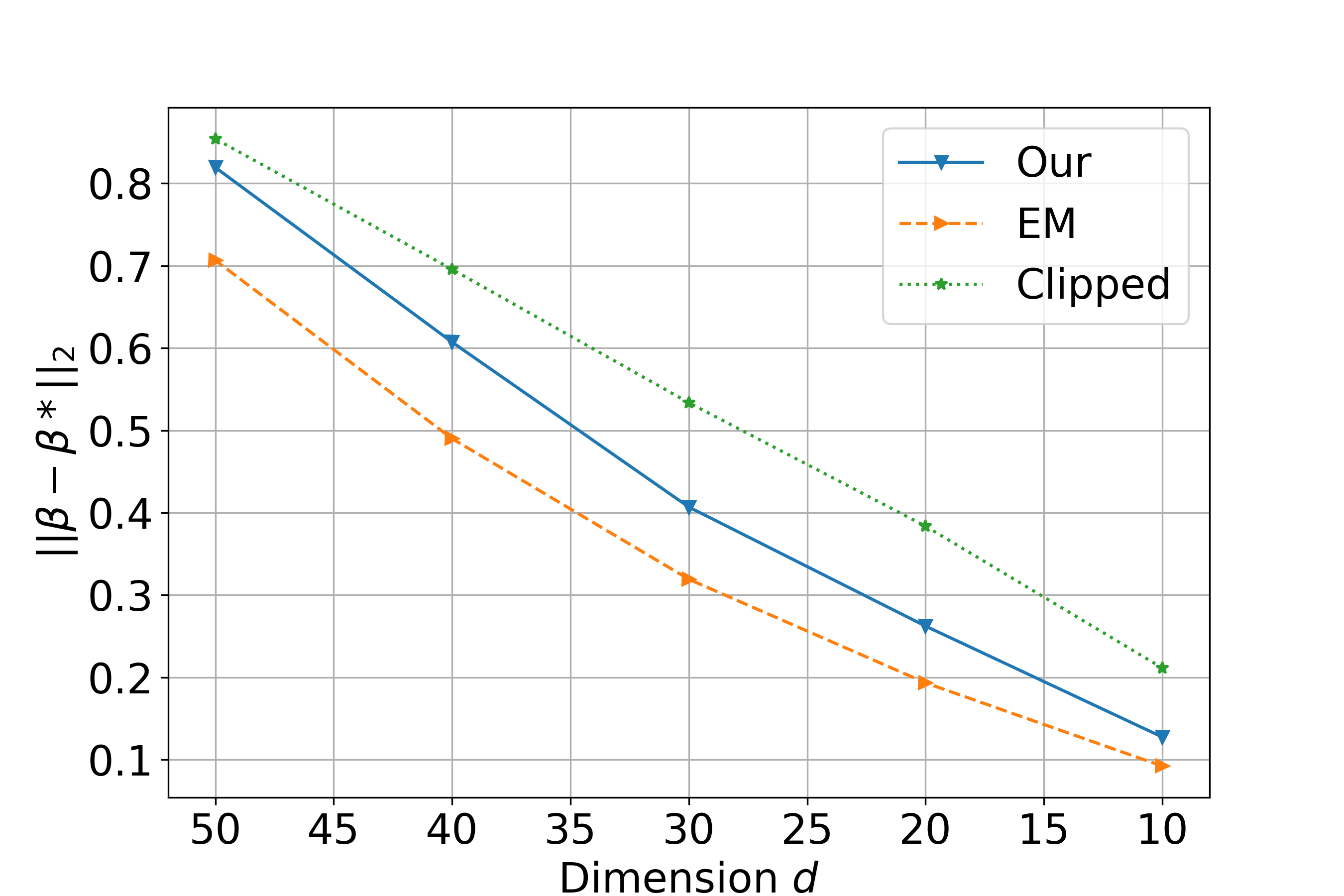}}
     \subfigure[$d=10, \epsilon =0.5$\label{rmc_samples2}]
  {\includegraphics[trim=0.05in 0 0.8in 0.7in,clip,width=2.2in]{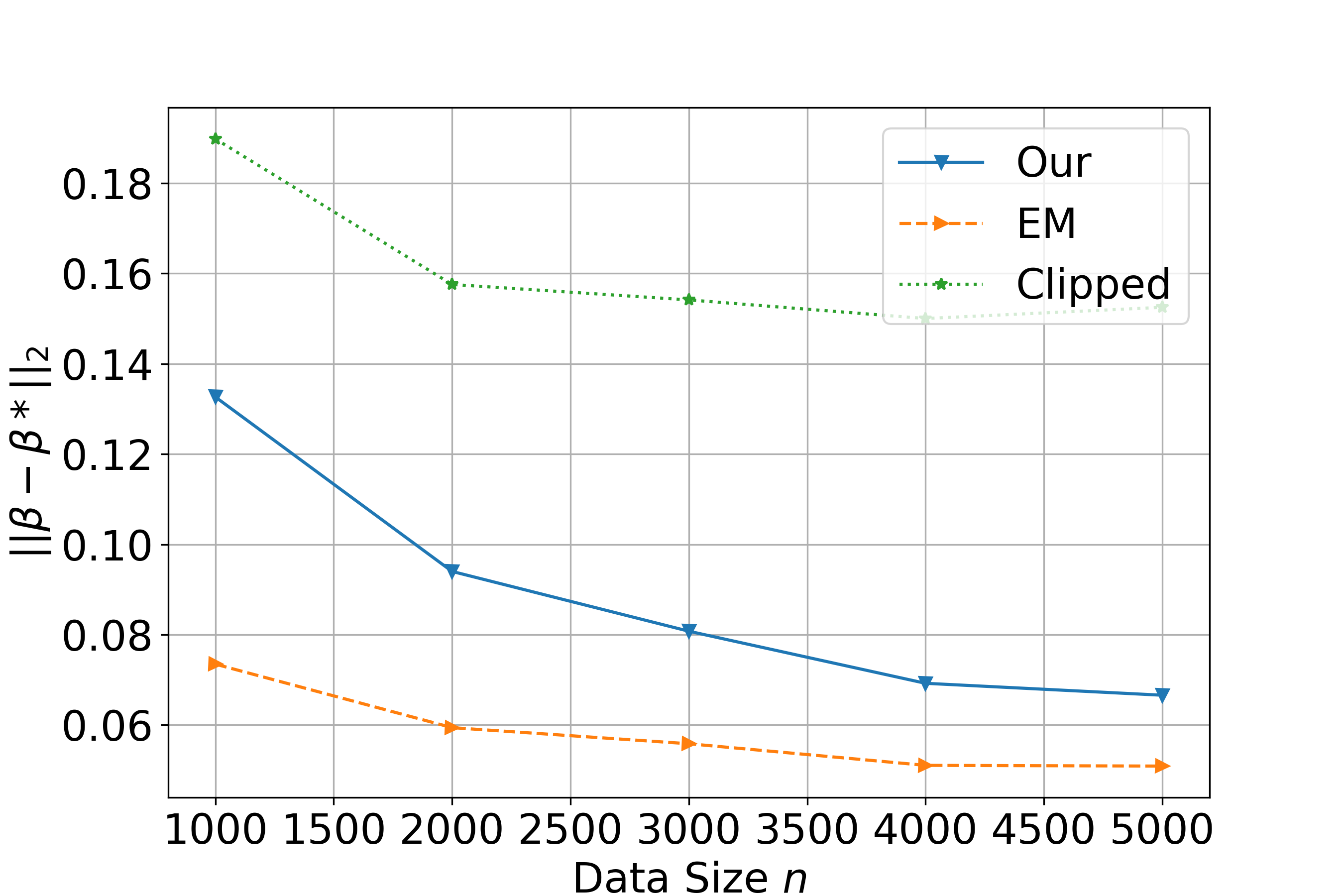}}
 \caption{Estimation error of RMC w.r.t privacy budget $\epsilon$, data dimension $d$ and data size $n$} \label{fig:comp_rmc2}
\end{figure*}

\begin{figure*}[!htbp]\centering
 \subfigure[ADULT]
 {\includegraphics[trim=0.05in 0 0.8in 0.7in,clip,width=2.2in]{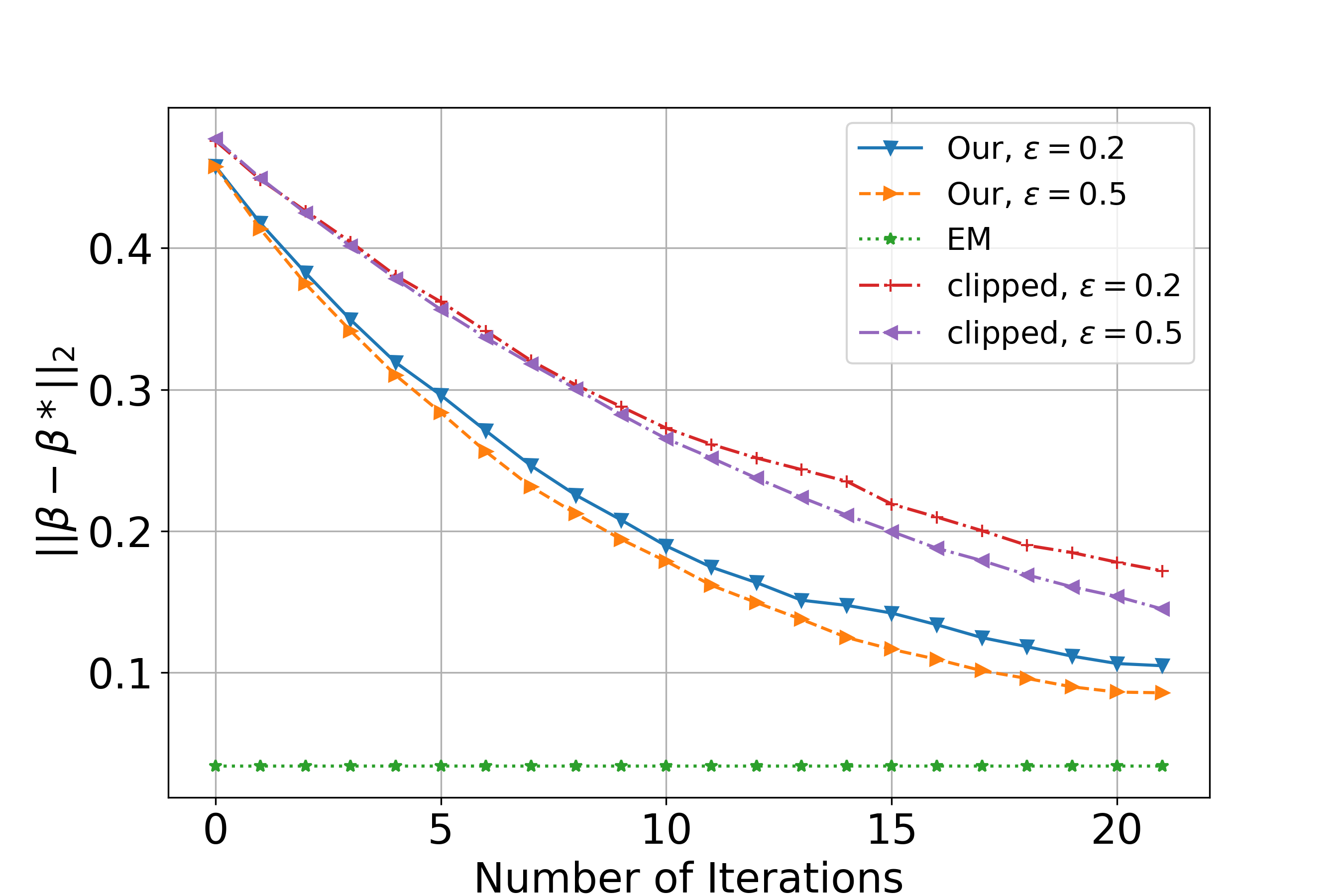}}
\subfigure[IPUMS-US ]
  {\includegraphics[trim=0.05in 0 0.8in 0.7in,clip,width=2.2in]{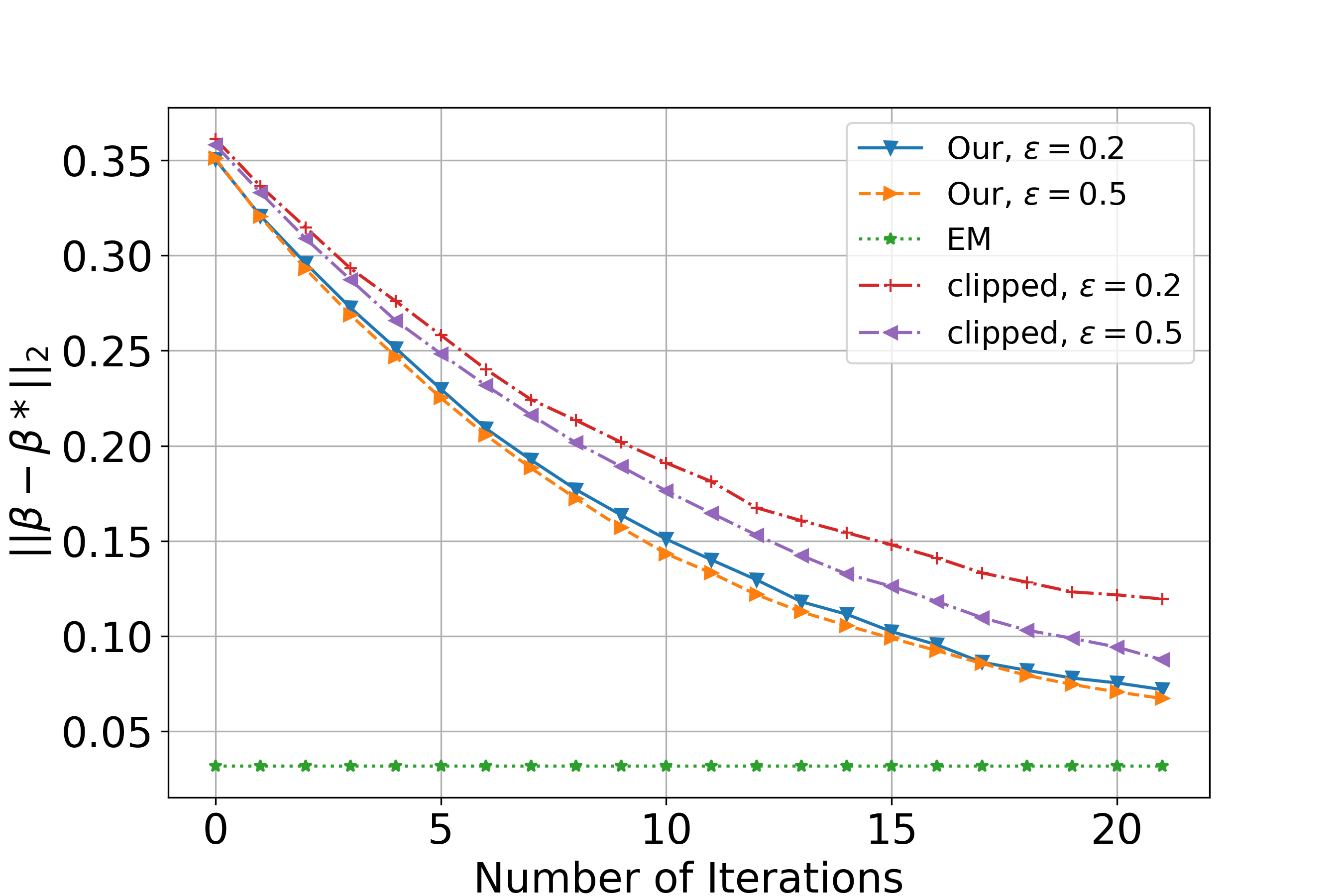}}
     \subfigure[IPUMS-BR]
  {\includegraphics[trim=0.05in 0 0.8in 0.7in,clip,width=2.2in]{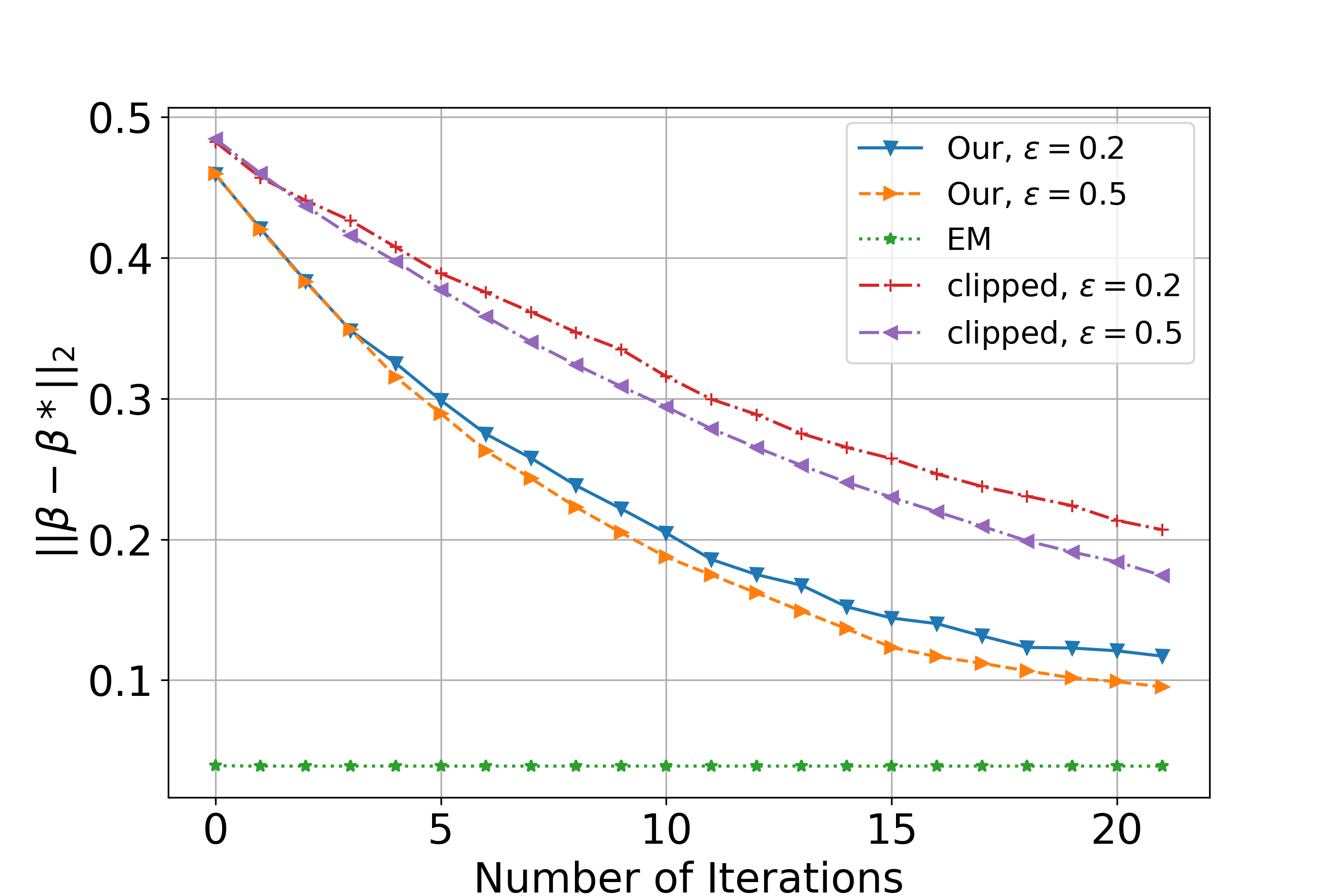}}\vspace{-0.1in}
 \caption{Estimation error of GMM over three real datasets: ADULT, IPUMS-US and IPUMS-BR } \label{fig:real}\vspace{-0.1in}
\end{figure*}

%In Figure \ref{fig:comp_gmm_alg3} and \ref{fig:comp_gmm_2_alg3} we compare DP EM (Algorithm \ref{alg:3}) with DP Gradient EM and the original EM algorithm on GMM. We can see that DP Gradient EM has lower error compared with DP EM in all the cases. 

\section{Conclusion}
In this paper, we provide the first study on the finite sample statistical guarantees of (Gradient) EM algorithm in Differential Privacy (DP) model. Due to the high sensitivity, previous DP Gradient Descent based methods cannot be directly extended to the Gradient EM algorithm. In order to address this issue,  we propose a new and improved private algorithm for estimating the mean of heavy-tailed distributions, which could also be extended to the local DP model and can be used to other machine learning problems, such as the Heavy-tailed Stochastic Convex Optimization. We also implement our algorithms to several canonical latent variable models, and some of these models has not been studied before. Finally, we conduct extensive experiments on both of the synthetic and real-world data, and these results outperform previous heuristic methods and  show the effectiveness of our algorithm. 

There are still several open problems for future research. Firstly, in Theorem \ref{theorem:3} we need to assume the initial parameter is close enough to the optimal one. We note that since the objective function could be non-convex, such assumption is unavoidable in general. However, it is still unknown how to find  such a parameter for general problems, even in the non-private case. Thus, our question is how to find such initial vectors for some specific models? Secondly, in this paper, we implement our algorithms on GMM and MRM with known variance, weights and only two components. It is still unknown whether we can use the same idea to other more complicated models, such as GMM with K components and/or unknown variance and weights. Finally, in this paper we provide the first theoretical results on MRM and RMC, it is still known unknown whether we can future improve these bounds and what is the lower bounds of these two models in DP model. 
\appendix

\section{Preliminaries}
First, we will  we recall some definitions and lemmas on the sub-exponential and sub-Gaussian random variables. See \citep{vershynin2010introduction} for details. 
\begin{definition}\label{adef:1}
For a sub-exponential random vector $X$, 
%we define 
its  sub-exponential norm $\|X\|_{\psi_1}$ is defined as  \begin{equation*}
    \|X\|_{\psi_1} = \sup_{p\geq 1} p^{-1}(\mathbb{E}|X|^p)^{\frac{1}{p}}.
\end{equation*}
\end{definition}
\begin{definition}[$\xi$-sub-exponential] \label{adef:2}
A random variable $X$ with mean $\mathbb{E}(X)$ is $\xi$-sub-exponential for $\xi>0$ if 
%it satisfies 
for all $|t|<\frac{1}{\xi}$,
    $\mathbb{E}\{\exp(t[X-\mathbb{E}(X)])\}\leq \exp(\frac{\xi^2t^2}{2}). $
\end{definition} 
\begin{lemma}\label{alemma:1}
Let $X$ be a sub-exponential random variable, then there are absolute constants $C, c>0$, such that when $|t|\leq \frac{c}{\|X\|_{\psi_1}}$ ,
\begin{equation*}
    \mathbb{E}[\exp(tX)]\leq \exp(Ct^2\|X\|_{\psi_1}^2).
\end{equation*}
\end{lemma}
\begin{lemma}\label{alemma:2}
From Definition \ref{adef:1}, \ref{adef:2} we can see that for a zero-mean sub-exponential random variable $X$, it second-order moment is bounded, {\em i.e.,} $\mathbb{E}X^2\leq O(\|X\|^2_{\psi_1})$.
\end{lemma}

\begin{lemma}[Bernstein's inequality]\label{alemma:3}
%Suppose that 
Let $X_1,\cdots, X_n$ be $n$ i.i.d realizations of $\upsilon$-sub-exponential random variable $X$ with mean $\mu$. Then, 
\begin{equation*}
    \text{Pr}(|\frac{1}{n}\sum_{i=1}^n X_i-\mu|\geq t)\leq 2\exp(-n\min(-\frac{t^2}{\upsilon^2}, \frac{t}{2\upsilon})).
\end{equation*}
\end{lemma}
\begin{definition}\label{adef:3}
A random variable $X$ is sub-Gaussian with variance $\sigma^2$ if for all $t>0$, the following holds  
\begin{equation*}
    \text{Pr}(|X-\mathbb{E}X|\geq t)\leq 2\exp(-\frac{t^2}{2\sigma^2}).
\end{equation*}
\end{definition}
\begin{definition}\label{adef:4}
For a sub-Gaussian random variable $X$, its sub-Gaussian norm $\|X\|_{\psi_2}$  is defined as 
\begin{equation*}
    \|X\|_{\psi_2}= \sup_{p\geq 1}p^{-\frac{1}{2}}(\mathbb{E}|X|^p)^{\frac{1}{p}}. 
\end{equation*}
\end{definition}
\begin{lemma}\label{alemma:4}
If $X$ is sub-Gaussian or sub-exponential, then  $\|X-\mathbb{E} X\|_{\psi_2}\leq 2\|X\|_{\psi_2}$ or $\|X-\mathbb{E} X\|_{\psi_1}\leq 2\|X\|_{\psi_1}$ holds, respectively. 
\end{lemma}
\begin{lemma}\label{alemma:5}
For two sub-Gaussian random variables $X_1, X_2$, $X_1\cdot X_2$ is a sub-exponential random variable with 
\begin{equation*}
    \|X_1\cdot X_2\|_{\psi_1}\leq C\max\{\|X_1\|_{\psi_2}^2, \|X_2\|_{\psi_2}^2\}.
\end{equation*}
\end{lemma}
\begin{lemma}\label{alemma:6}
Let $X_1, X_2, \cdots, X_k$ be $k$ independent zero-mean sub-Gaussian random variables, and $X= \sum_{j=1}^kX_j$. Then, $X$ is sub-Gaussian with  $\|X\|_{\psi_2}^2\leq C\sum_{j=1}^k\|X_j\|_{\psi_2}^2$ for some absolute constant $C>0$.
\end{lemma}
Next, we provide some symmetrization results of random variables, which will be used in our proofs. See \citep{boucheron2013concentration} for details. 
\begin{lemma}\label{alemma:7}
Let $y_1, y_2, \cdots, y_n$ be the $n$ independent realizations of the random vector $Y\in \mathcal{Y}$, and $\mathcal{F}$ be a function class defined on $\mathcal{Y}$. For any increasing convex function $\phi(\cdot)$, the following holds  
\begin{equation*}
    \mathbb{E}\{\phi[\sup_{f\in \mathcal{F}}|\sum_{i=1}^n f(y_i)-\mathbb{E}(f(Y))|]\}\leq \mathbb{E}\{\phi[\sup_{f\in \mathcal{F}}|\sum_{i=1}^n \epsilon_i
f(y_i)|]\},
\end{equation*}
where $\epsilon_1, \cdots, \epsilon_n$ are i.i.d Rademacher random variables that are independent of $y_1, \cdots, y_n$. 
\end{lemma}

\begin{lemma}\label{alemma:8}
Let $y_1, \cdots, y_n$ be $n$ independent realization of the random vector $Z\in \mathcal{Z}$ and $\mathcal{F}$ be a function class defined on $\mathcal{Z}$. 
If 
%We consider the 
Lipschitz functions $\{\phi_i(\cdot)\}_{i=1}^n$ satisfy the following for all $v, v'\in\mathbb{R}$
\begin{equation*}
    |\phi_i(v)-\phi_i(v')|\leq L|v-v'| 
\end{equation*}
and $\phi_i(0)=0$, then for any increasing convex function $\phi(\cdot)$, the following holds 
%we have 
\begin{equation*}
    \mathbb{E}\{\phi[|\sup_{f\in \mathcal{F}}\sum_{i=1}^n\epsilon_i\phi_i(f(y_i))|]\}\leq \mathbb{E}\{\phi[2|L\sup_{f\in \mathcal{F}}\sum_{i=1}^n\epsilon_if(y_i)|]\},
\end{equation*}
where $\epsilon_1, \cdots, \epsilon_n$ are i.i.d Rademacher random variables that are independent of $y_1, \cdots, y_n$. 
\end{lemma}

\section{Omitted Proofs}\label{sec:proof}
\begin{proof}[{\bf Proof of Theorem \ref{thm:1}}]
Note that by (\ref{aeq:1}), we have  
\begin{equation*}
    \nabla q(\beta;\beta)=[\frac{2}{1+\exp(-\langle \beta, y\rangle/\sigma^2)}-1]\cdot y-\beta.
\end{equation*}
W.l.o.g,  we assume that  $\beta=(1,0,\cdots,0)^T$ and $\sigma=1$ in the GMM model. Then, we can see that for each constant $c\geq 0$, if 
\begin{align*}
    &\|\frac{y}{3}\|_2\geq c+\|\beta\|_2\\
    &\langle \beta, y \rangle \geq \ln 2\\
    & y\geq 0
\end{align*}
and denote the set of $y$  satisfying the above assumptions as $\mathcal{S}$,  we have 
\begin{align*}
   \| \nabla q(\beta;\beta)\|_2\geq  \|\frac{y}{3}\|_2- \|\beta\|_2\geq c.
\end{align*}
The above assumptions hold if $y=(\ln 2+1, 3s, a_3, a_4, \cdots, a_d)$, where $s\geq c$ and $ a_3,\cdots, a_d\geq 0$. We can easily see that $\mathbb{P}[y\in \mathcal{S}]>0$ since $y$ follows a mixture of Gaussian distributions. 
\end{proof}
\begin{proof}[{\bf Proof of  Theorem \ref{thm:new}} ]
Let $\mathcal{P}(\mathbb{R})$ denote all the probability measures on $\mathbb{R}$, with an appropriate $\sigma$-field tacitly assumed. Consider any two measures $v, v_0\in \mathcal{P}(\mathbb{R})$, and $h:\mathbb{R}\mapsto \mathbb{R}$ a $v_0$-measurable function. By \citep{catoni2004statistical}, it is proved that a Legendre transform of the mapping $v\mapsto K(v, v_0)$ takes the form of a cumulant generating function, namely 
\begin{equation}\label{eq:new1}
    \sup_v (\int h(u)dv(u)-K(v, v_0))=\log \int \exp(h(u))dv_0(u), 
\end{equation}
where the supremum is taken over $v\in \mathcal{P}(\mathbb{R})$. Following \citep{catoni2017dimension} here we use the Kullback
divergence for the Legendre transform of the mapping, so we define 
\begin{equation}
     K(v, v_0)=\int \log (\frac{d v}{d v_0} )dv
\end{equation}
iff $v_0\ll v$, and  $ K(v, v_0)=+\infty$ otherwise.

This identity is a technical tool and the choice of $h$ and $v_0$ are parameters that can be adjusted to fit the application. 
In actually setting these parameters, we will follow the technique given by \citep{catoni2017dimension}, which is later adapted by \citep{holland2019a}. Note the term 
\begin{equation*}
    \phi(\frac{x_i+\eta_i x_i}{s})
\end{equation*}
depends on two terms, namely the data $x_i$ and the artificial noise $\eta_i$ (if we fix $s$). Thus, for convenience we denote 
\begin{equation*}
    f(\eta, x):=\phi(\frac{x+\eta x}{s}).
\end{equation*}
By the definition of $\phi$ we can see that $f:\mathbb{R}^2\mapsto \mathbb{R}$ is measurable and bounded. Next, we denote that 
\begin{equation*}
    h(\eta)=\sum_{i=1}^n f(\eta, x_i)-c(\eta), 
\end{equation*}
where $c(\epsilon)$ is a term to be determined shortly. Take $h(\eta)$ into (\ref{eq:new1}) we have 
\begin{align*}
    B:&=  \sup_v (\int h(u)dv(u)-K(v, v_0))\\
    &= \log \int \exp(\sum_{i=1}^n f(\eta, x_i)-c(\eta))dv(\eta), 
\end{align*}
Taking the exponential of this B and then taking expectation with respect to the sample, we
have 
\begin{align*}
    \mathbb{E}\exp(B)&=\mathbb{E}\int (\frac{\exp(\sum_{i=1}^n f(\eta, x_i))}{\exp(c(\eta))})v(\eta) \\
    &= \int \frac{\Pi_{i=1}^n \mathbb{E}\exp( f(\eta, x_i))}{\exp(c(\eta))})v(\eta)
\end{align*}
The first equality comes from simple log/exp manipulations, and the second equality from
taking the integration over the sample inside the integration with respect to $v$, valid via Fubini’s theorem. By setting 
\begin{equation*}
    c(\eta)=n \log \mathbb{E} \exp (f(\eta, x)). 
\end{equation*}
 With this preparation done, we can start on the high-probability upper bound of interest:
\begin{align*}
    P(B\geq \log \frac{1}{\zeta})&= P(\exp(B)\geq \frac{1}{\zeta}) \\
    &= \mathbb{E}\mathbb{I}(\exp(B){\zeta} \geq 1) \\
    &= \leq \mathbb{E} \exp(B){\zeta}=\zeta, 
\end{align*}
where the last equality is due to $\mathbb{E} \exp(B)=1$ by setting $ c(\eta)=n \log \mathbb{E} f(\eta, x)$. Note that since our setting of $c(\eta)$ is such that $c(\cdot)$ is $v$-measurable (via the measurability of $f$), the resulting $h$ is indeed measurable w.r.t $v$. Thus by the definition of $B$ we have with probability at least $1-\zeta$
\begin{equation*}
    \sup_v (\int h(u)dv(u)-K(v, v_0))\leq \log \frac{1}{\zeta}. 
\end{equation*}
    Take the implicit form of $B$ via $h(\eta)$ and $c(\eta)$ and divide by $n$ form both side we have 
    \begin{equation}\label{eq:new2}
        \frac{1}{n}\int \sum_{i=1}^n f(\eta, x_i)dv(\eta)\leq \int \log \mathbb{E}\exp (f(\eta,x))d v(\eta) + \frac{K(v, v_0)+\log \frac{1}{\zeta}}{n}. 
    \end{equation}
It is notable that by definition $\hat{x}$ in (\ref{eq:15}) satisfies 
\begin{align*}
    \hat{x}&=  \frac{s}{n}\sum_{i=1}^n \int \phi(\frac{x_i+\eta_i x_i}{s})d \chi(\eta_i)=\frac{s}{n}\int \sum_{i=1}^n f(\eta, x_i)dv(\eta) \\
    &\leq s\int \log \mathbb{E}\exp(f(\eta,x))d v(\eta)+ \frac{sK(v, v_0)+s\log \frac{1}{\zeta}}{n}.
\end{align*}

In the following we will bound the term of $int \log \mathbb{E}\exp(f(\eta,x))d v(\eta) $ and $K(v, v_0)$. Starting with the first term, recall the definition of the truncation function $\phi(\cdot)$ in (\ref{eq:10}) we can it satisfies that for all $u\in \mathbb{R}$
\begin{equation}\label{eq:newlow}
    -\log (1-u + \frac{u^2}{2})\leq \phi(u)\leq \log(1+u+\frac{u^2}{2}). 
\end{equation}
Thus we have 
\begin{align}
    &\int \log \mathbb{E}f(\eta,x)d v(\eta) \nonumber \\
    &= \int \log \mathbb{E} \exp(\phi(\frac{(1+\eta)x}{s}) d v(\eta) \nonumber \\
    & \leq \int \log (1+\frac{(1+\eta)\mathbb{E}x }{s}+ \frac{(1+\eta)^2\mathbb{E}x^2 }{2s^2}) dv(\eta)  \nonumber \\
    &\leq \int (\frac{(1+\eta)\mathbb{E}x }{s}+ \frac{(1+\eta)^2\mathbb{E}x^2 }{2s^2}) dv(\eta)  \nonumber \\
    &= \frac{\mathbb{E}x}{s} (1+\mathbb{E}_\eta \eta)+ \frac{\mathbb{E} x^2}{2s^2} \mathbb{E}_\eta (1+\eta)^2 \nonumber \\
    &= \frac{\mathbb{E}x}{s}+ \frac{\mathbb{E} x^2}{2s^2} (1+\frac{1}{\beta}). \label{eq:new3}
\end{align}
Where the last equality is due to $\eta \sim \mathcal{N}(0, \frac{1}{\beta})$.

For term $K(v, v_0)$, it is critical to select an appropriate measure $v_0$ to easily calculate the KL-divergence. Here we set $v_0\sim \mathcal{N}(1, \frac{1}{\beta})$. In this case we have 
\begin{align}
    K(v, v_0) &= \int_{-\infty}^{+\infty} \log (\exp(\frac{\beta(u-1)^2}{2}-\frac{\beta u^2}{2}))\sqrt{\frac{\beta}{2\pi}}\exp(-\frac{\beta u^2}{2})du \nonumber \\
    &= \int_{-\infty}^{+\infty} \frac{(1-2u)\beta }{2}\sqrt{\frac{\beta}{2\pi}}\exp(-\frac{\beta u^2}{2})du \nonumber \\
    &=\frac{\beta}{2}. \label{eq:new4}
\end{align}
Thus, combining with (\ref{eq:new3}) and (\ref{eq:new4}) we have with probability at least $1-\zeta$ 
\begin{equation}
     \hat{x}\leq \mathbb{E}x + \frac{\mathbb{E}x^2}{2s}(\frac{1}{\beta}+1)+\frac{s}{n}(\frac{\beta}{2}+\log \frac{1}{\zeta}). 
\end{equation}
Thus, by the definition of $\mathcal{A}(D)$ and the concentration bound of Gaussian distribution we have with probability at least $1-2\zeta$, 
\begin{equation}\label{eq:new5}
    \mathcal{A}(D)-\mathbb{E}x \leq O(\frac{\tau}{s\beta}+\frac{s\log \frac{1}{\zeta}\beta}{n}+ \frac{s\log \frac{1}{\zeta}\sqrt{\log \frac{1}{\delta}}}{\epsilon n}). 
\end{equation}
Next,we will get a lower bound of $ \mathcal{A}(D)-\mathbb{E}x $. The proof is quite similar as in the above proof. The main difference is here we set 
\begin{equation*}
    f(\eta, x):=-\phi(\frac{x+\eta x}{s}). 
\end{equation*}
Thus we have 
\begin{align*}
    -\hat{x}&=  \frac{s}{n}\sum_{i=1}^n \int -\phi(\frac{x_i+\eta_i x_i}{s})d \chi(\eta_i)=\frac{s}{n}\int \sum_{i=1}^n f(\eta, x_i)dv(\eta) \\
    &\leq s\int \log \mathbb{E}\exp(f(\eta,x))d v(\eta)+ \frac{sK(v, v_0)+s\log \frac{1}{\zeta}}{n} \\ 
    &\leq s\int \log \mathbb{E}\exp(-\phi(\frac{x+\eta x}{s}))d v(\eta)+ \frac{sK(v, v_0)+s\log \frac{1}{\zeta}}{n}
\end{align*}
By (\ref{eq:newlow}) we have 
\begin{align*}
     -\hat{x} &\leq s\int \log \mathbb{E}\exp(-\phi(\frac{x+\eta x}{s}))d v(\eta)+ \frac{sK(v, v_0)+s\log \frac{1}{\zeta}}{n} \\
     &\leq s[ (-1)\int \frac{(1+\eta)\mathbb{E}x}{s}dv(\eta)+ \int \frac{(1+\eta)^2\mathbb{E}x^2}{s^2}dv(\eta)] \\
     &\qquad + \frac{sK(v, v_0)+s\log \frac{1}{\zeta}}{n}\\
     &\leq -\mathbb{E} x+ \frac{\mathbb{E}x^2}{2s}(\frac{1}{\beta}+1)+\frac{s}{n}(\frac{\beta}{2}+\log \frac{1}{\zeta}). 
\end{align*}
     Thus we have 
     \begin{equation}\label{eq:new7}
         \mathbb{E} x-\mathcal{A}(D)\leq O(\frac{\tau}{s\beta}+\frac{s\log \frac{1}{\zeta}\beta}{n}+ \frac{s\log \frac{1}{\zeta}\sqrt{\log \frac{1}{\delta}}}{\epsilon n}). 
     \end{equation}
     In total we have with probability at least $1-4\zeta$, 
     \begin{equation}
            |\mathbb{E} x-\mathcal{A}(D)| \leq O(\frac{\tau}{s\beta}+\frac{s\log \frac{1}{\zeta}\beta}{n}+ \frac{s\log \frac{1}{\zeta}\sqrt{\log \frac{1}{\delta}}}{\epsilon n}). 
     \end{equation}
     We can get the proof by setting $\beta=\sqrt{\log \frac{1}{\zeta}}$ and $s=\frac{\sqrt{n\epsilon\tau}}{\log \frac{1}{\zeta}\log^{1/4}\frac{1}{\delta}}$.
     \end{proof}

     \begin{proof}[{\bf Proof of Theorem \ref{thm:new2}}]
The proof of $(\epsilon, \delta)$ is just followed by Gaussian mechanism and here we omit it. For the estimation error, it is notable that $\mathcal{A}(D)$ is equivalent to
\begin{equation}\label{eq:new15}
    \mathcal{A}(D)=\hat{x}+ Z, Z\sim \mathcal{N}(0, \sigma^2), \sigma^2=O(\frac{s^2\log \frac{1}{\delta}}{\epsilon^2 n }). 
\end{equation}
Thus, by the same idea as in the proof of Theorem \ref{theorem:2}, we can see that we have 
  \begin{equation}
            |\mathbb{E} x-\mathcal{A}(D)| \leq O(\frac{\tau}{s\beta}+\frac{s\log \frac{1}{\zeta}\beta}{n}+ \frac{s\log \frac{1}{\zeta}\sqrt{\log \frac{1}{\delta}}}{\epsilon \sqrt{n}}). 
     \end{equation}
     
     We can get the result by setting $\beta=\sqrt{\log \frac{1}{\zeta}}$ and $s=\frac{\sqrt[4]{n}\sqrt{\epsilon\tau}}{\log \frac{1}{\zeta}\log^{1/4}\frac{1}{\delta}}$.
\end{proof}
\begin{proof}[{\bf Proof of Theorem \ref{theorem:2}}]
We first give the definition of zCDP in \citep{bun2016concentrated}.

\begin{definition}
    A randomized algorithm $\mathcal{A}: \mathcal{X}^n\mapsto \mathcal{Y}$ is $\rho$-zero Concentrated Differentially Private (zCDP) if for all neighboring datasets $D\sim D'$ and all $\alpha\in (1, \infty)$, 
    \begin{equation*}
        D_\alpha (\mathcal{A}(D)\| \mathcal{A}(D'))\leq \rho\alpha, 
    \end{equation*}
    where $ D_\alpha (P\| Q)=\frac{1}{\alpha-1}\log \mathbb{E}_{X\sim P}[(\frac{P(X)}{Q(X)})^{\alpha-1}]$ denotes the R\'{e}nyi divergence of order $\alpha$.
\end{definition}

We first convert $(\epsilon, \delta)$-DP to $\rho$-zCDP by using the following lemma 
\begin{lemma}[\citep{bun2016concentrated}]
Let $M: \mathcal{X}^n\mapsto \mathcal{Y}$ be a randomized algorithm. If $M$ is $\rho$-zCDP,  it is 
%satisfies 
$(\rho+ 2\sqrt{\rho \log \frac{1}{\delta}}, \delta)$-DP for all $\delta>0$. 
\end{lemma}

Thus, it  suffices to show that 
%we could show 
Algorithm \ref{alg:2} is $\tilde{\epsilon}^2= (\sqrt{\epsilon+\log \frac{1}{\delta}}-\sqrt{\log \frac{1}{\delta}})^2$-zCDP. The following lemma shows that adding some Gaussian noise will preserve zCDP. 
\begin{lemma}
	Given a function $q : \mathcal{X}^n\rightarrow \mathbb{R}^p$, the Gaussian Mechanism is defined as:
		$\mathcal{M}_G(D,q,\epsilon)=q(D)+ Y,$
		where $Y$ is drawn from a Gaussian Distribution $\mathcal{N}(0,\sigma^2I_p)$ is $\frac{\Delta_2^2(q)}{2\sigma^2}$-zCDP.  $\Delta_2(q)$ is the $\ell_2$-sensitivity of the function $q$, {\em i.e.,}
		$\Delta_2(q)=\sup_{D\sim D'}||q(D)-q(D')||_2.$
\end{lemma}

 By Lemma \ref{lemma:2} we know $\Delta_2(g_j^{t-1}(\beta^{t-1}))=\frac{4\sqrt{2}}{3}\frac{s}{m}$. By simple calculation we can show that in each iteration and each coordinate, outputting $g_j^{t-1}(\beta^{t-1})$ will be $\frac{\tilde{\epsilon}^2}{d}$-zCDP. Thus by the composition property of zCDP, we know that 
%in total 
it is  $\tilde{\epsilon}^2$-zCDP.
\end{proof}

\begin{proof}[{\bf Proof of Theorem \ref{theorem:3}}]
Consider $t$-th iteration,  under the assumption that $\beta^{t-1}\in \mathcal{B}$ we have 
\begin{align}
    \|\beta^t-\beta^*\|_2 & = \|\beta^{t-1}+\eta \tilde{\nabla} Q_n(\beta^{t-1})-\beta^*\|_2 \nonumber \\
    & \leq  \|\beta^{t-1}+\eta \nabla  Q(\beta^{t-1};\beta^{t-1} )-\beta^*\|_2 +\eta \|\tilde{\nabla} Q_n(\beta^{t-1})- \nabla  Q(\beta^{t-1};\beta^{t-1})\|_2  \label{aeq:5}
\end{align}
We first bound the first term of (\ref{aeq:5}). 
\begin{align}
   & \|\beta^{t-1}+\eta \nabla  Q(\beta^{t-1};\beta^{t-1} )-\beta^*\|_2 \nonumber \\&\leq  \|\beta^{t-1}+\eta \nabla  Q(\beta^{t-1};\beta^* )-\beta^*\|_2 + \eta \|\nabla Q(\beta^{t-1};\beta^{t-1} )- \nabla Q(\beta^{t-1};\beta^* )\|_2  \label{aeq:6}
\end{align}
We then consider the first term of (\ref{aeq:6}). We note that the self-consistent property in Definition \ref{def:0} implies that 
\begin{equation}
    \beta^*=\arg\max_{\beta}Q(\beta; \beta^*),
\end{equation}
which means that $\beta^*$ is a maximizer of $Q(\beta; \beta^*)$. Thus, the proof follows from  the convergence rate of the strongly convex and smooth functions $Q(\beta; \beta^*)$ in \citep{nesterov2013introductory}. For the step size $\eta=\frac{2}{\mu+\upsilon}$, we have 
\begin{equation}
    \|\beta^{t-1}+\eta\nabla Q(\beta^{t-1}; \beta^*)-\beta^*\|_2\leq (\frac{\mu-\upsilon}{\mu+\upsilon})\|\beta^{t-1}-\beta^*\|_2.
\end{equation}
Thus, by the Lipschitz-Gradient-2($\gamma, \mathcal{B}$) condition, we get the following of (\ref{aeq:6}) 
\begin{align}
    &\|\beta^{t-1}+\eta \nabla  Q(\beta^{t-1};\beta^{t-1} )-\beta^*\|_2 \nonumber \\
    &\leq  \|\beta^{t-1}+\eta \nabla  Q(\beta^{t-1};\beta^* )-\beta^*\|_2 + \eta \|\nabla Q(\beta^{t-1};\beta^{t-1} )- \nabla Q(\beta^{t-1};\beta^* )\|_2 \nonumber \\
    &\leq (\frac{\mu-\upsilon}{\mu+\upsilon})\|\beta^{t-1}-\beta^*\|_2+\eta \gamma \|\beta^{t-1}-\beta^*\|_2  \nonumber \\
    &= (1-2\frac{v-\gamma}{\mu+v})\|\beta^{t-1}-\beta^*\|_2  \label{aeq:10}
\end{align}
where the the last inequality is due to taking  $\eta=\frac{2}{\mu+v}$. 

Next we bound the second term of (\ref{aeq:5}). For convenience we denote the first sum of (\ref{eq:17}) ({\em i.e.,} the robust mean estimator ) as $\tilde{g}^{t-1}_j (\beta^{t-1})$. So we have 
\begin{align}
    \|\tilde{\nabla} Q_n(\beta^{t-1})- \nabla  Q(\beta^{t-1};\beta^{t-1})\|^2_2  &= \sum_{j=1}^d (g_j^{t-1}(\beta^{t-1})-\mathbb{E}\nabla_j q(\beta^{t-1}; \beta^{t-1}))^2   \label{aeq:7} \\
    &\leq \sum_{j=1}^d (\tilde{g}_j^{t-1}(\beta^{t-1})-\mathbb{E}\nabla_j q(\beta^{t-1}; \beta^{t-1}))^2 +  \sum_{j=1}^d |Z^{t-1}_j|^2 \label{aeq:8}
\end{align}
The first equality is due to Assumption \ref{assumption:1}.   For the second term of (\ref{aeq:8}), by the high probability concentration bound  of Gaussian random variable we have for fixed $j$ with probability at least $1-\frac{\zeta}{d}$, $|Z^{t-1}_j|^2 \leq \frac{8\tau dT\log \frac{d}{\zeta}}{9\beta n\tilde{\epsilon}^2}$. Thus with probability at least $1-\zeta$ we have 
\begin{equation*}
    \sum_{j=1}^d |Z^{t-1}_j|^2 \leq \frac{8\tau d^2T\log \frac{d}{\zeta}}{9\beta n\tilde{\epsilon}}.
\end{equation*}
For the first term of (\ref{aeq:8}), by Lemma \ref{lemma:1} and taking $\zeta=\frac{\zeta}{d}$, we have for a fixed $j\in [d]$, $(\tilde{g}_j^{t-1}(\beta^{t-1})-\mathbb{E}\nabla_j q(\beta^{t-1}; \beta^{t-1}))^2\leq O(\frac{\tau\log \frac{d}{\zeta}}{n})$. Thus, with probability at least $1-\zeta$, we have 
\begin{equation*}
  \sum_{j=1}^d (\tilde{g}_j^{t-1}(\beta^{t-1})-\mathbb{E}\nabla_j q(\beta^{t-1}; \beta^{t-1}))^2 \leq O(\frac{d\tau\log \frac{d}{\zeta}}{n}).
\end{equation*}
Hence, we have, with probability at least $1-2\zeta$, for some constant $C_2$
\begin{equation}\label{aeq:9}
     \|\tilde{\nabla} Q_n(\beta^{t-1})- \nabla  Q(\beta^{t-1};\beta^{t-1})\|_2\leq C_2 \frac{ d\sqrt{\tau T\log \frac{d}{\zeta}}}{\sqrt{\beta n\tilde{\epsilon}}}.
\end{equation}
Plugging (\ref{aeq:9}) and (\ref{aeq:10}) into (\ref{aeq:5}), we have, with probability  $1-2\zeta$ and for some constant $C_3$,  
\begin{equation} \label{aeq:11}
     \|\beta^t-\beta^*\|_2\leq (1-2\frac{v-\gamma}{\mu+v})\|\beta^{t-1}-\beta^*\|_2\\+C_3\frac{2}{\mu+v}\cdot \frac{ d\sqrt{\tau T\log \frac{d}{\zeta}}}{\sqrt{\beta n\tilde{\epsilon}}}
\end{equation}
Next, we will show that when $n$ is large enough,  if  $\|\beta^0-\beta^*\|_2\leq \frac{R}{2}$ then $\|\beta^t-\beta^*\|_2\leq \frac{R}{2}$ holds (and thus $\beta\in \mathcal{B}$) for all $t\in [T]$ if (\ref{aeq:11}) holds for all $t\in [T]$ (and this hold with probability at least $1-2T\zeta$). 

We will use induction. When $t=1$, by (\ref{aeq:11}) we have 
\begin{align*}
    \|\beta^1-\beta^*\|_2 &\leq (1-2\frac{v-\gamma}{\mu+v})\|\beta^0-\beta^*\|_2+C_3\frac{2}{\mu+v}\cdot \frac{ d\sqrt{\tau T\log \frac{d}{\zeta}}}{\sqrt{\beta n\tilde{\epsilon}}}  \\
    &\leq (1-2\frac{v-\gamma}{\mu+v})\frac{R}{2}+ C_3\frac{2}{\mu+v}\cdot \frac{ d\sqrt{\tau T\log \frac{d}{\zeta}}}{\sqrt{\beta n\tilde{\epsilon}}}. 
\end{align*}
If $C_3\frac{2}{\mu+v}\cdot \frac{ d\sqrt{\tau T\log \frac{d}{\zeta}}}{\sqrt{\beta n\tilde{\epsilon}}} \leq 2\frac{v-\gamma}{\mu+v}\cdot \frac{R}{2}$, then we can see that $\|\beta^1-\beta^*\|_2 \leq \frac{R}{2}$. This holds if 
\begin{equation*}
    C_4 (\frac{1}{v-\gamma})^2 \frac{d^2 \tau T\log \frac{d}{\zeta}}{R^2 \beta \tilde{\epsilon}} \leq n 
\end{equation*}
for some constant $C_4$.

Next, we will assume that (\ref{aeq:11}) holds for all $t\in [T]$ and $\beta\in \mathcal{B}$ for all $t\in [T]$. For convenience, we denote $\iota=1-2\frac{v-\gamma}{\mu+v}$. 
By (\ref{aeq:11}), we  have 
\begin{align*}
     \|\beta^T-\beta^*\|_2 &\leq (1-2\frac{v-\gamma}{\mu+v})^T\|\beta^0-\beta^*\|_2+C_3(1+\iota+\iota^2+\cdots)\frac{2}{\mu+v}\cdot \frac{ d\sqrt{\tau T\log \frac{d}{\zeta}}}{\sqrt{\beta n\tilde{\epsilon}}}\\
      &\leq (1-2\frac{v-\gamma}{\mu+v})^T\frac{R}{2}+ C_3\frac{1}{1-\iota}\cdot \frac{2}{\mu+v}\cdot \frac{ d\sqrt{\tau T\log \frac{d}{\zeta}}}{\sqrt{\beta n\tilde{\epsilon}}}\\
      &= (1-2\frac{v-\gamma}{\mu+v})^T\frac{R}{2}+O(\frac{1}{v-\gamma}\frac{ d\sqrt{\tau T\log \frac{d}{\zeta}}}{\sqrt{\beta n\tilde{\epsilon}}}). 
\end{align*}
Taking $T=O(\frac{\mu+v}{v-\gamma}\log \frac{n\tilde{\epsilon}}{d})$, we have, with probability at least $1-2T\zeta$,
\begin{equation*}
      \|\beta^T-\beta^*\|_2\leq \tilde{O}(R\sqrt{\frac{\mu+v}{(v-\gamma)^3}} \frac{ d\sqrt{\tau \log n\log \frac{d}{\zeta}}}{\sqrt{\beta n\tilde{\epsilon}}}). 
\end{equation*}
Since $\tilde{\epsilon}= \sqrt{\log \frac{1}{\delta}+\epsilon}-\sqrt{\log \frac{1}{\delta}} $, by using the Taylor series of the function $\sqrt{x+1}-\sqrt{x}$, we have $\tilde{\epsilon}= O(\frac{\epsilon}{\sqrt{\log \frac{1}{\delta}}})$. 
Thus, we have the proof by taking $\zeta=\frac{\zeta}{2T}$. 
\end{proof}

\begin{proof}[{\bf Proof of Lemma \ref{lemma:4}}]
To prove Lemma \ref{lemma:4}, we need a stronger lemma. 
\begin{lemma}\label{alemma:9}
The $j$-the coordinate of $\nabla q(\beta; \beta)$ is $\xi$-sub-exponential with 
%%\vspace{-0.1in}
\begin{equation}
    \xi = C_1\sqrt{\|\beta^*\|^2_{\infty}+\sigma^2},
\end{equation}
where $C_1$ is some absolute constant. Also, for fixed $j\in [d]$, each $\nabla_j q_i(\beta;\beta)$, where $i\in [n]$, is independent with others.
\end{lemma}
If Lemma \ref{alemma:9} holds, then by Lemma \ref{alemma:2} we can get Lemma \ref{lemma:4}. 
\end{proof}
\begin{proof}[Proof of Lemma \ref{alemma:9}]
From (\ref{aeq:1}) it is oblivious that each $\nabla_j q_i(\beta;\beta)$, where $i\in [n], j\in [d]$, is independent with others. Next, we prove the property of sub-exponential for each coordinate. 

Note that $$\nabla_j q(\beta;\beta))= [2w_\beta(y)-1]y_{j}-\beta_j,$$ and $$\mathbb{E}_{Y}\nabla_j q(\beta;\beta))= \mathbb{E}_{Y}(2w_\beta(Y)Y_j-Y_j)-\beta_j.$$

By the symmetrization lemma in Lemma \ref{alemma:7}, we have the following for any $t>0$
\begin{equation}
        \mathbb{E}\{\exp(t|[\nabla_j q(\beta; \beta)-\mathbb{E}\nabla_j q (\beta; \beta)]|)\}\\ \leq  \mathbb{E}\{\exp(t|\epsilon [2w_\beta(y)-1]y_{j}|)\},
\end{equation}

where $\epsilon$ is a Rademacher random variable. 

Next, we use Lemma \ref{alemma:8} with $f(y_{j})=y_{j}$, $\mathcal{F}=\{f\}$, $\phi(v)=[2w_\beta(y)-1]v$ and $\phi(v)=\exp(u\cdot v)$. It is easy to see that $\phi$ is 1-Lipschitz. Thus,  by Lemma \ref{alemma:8} we have
\begin{equation}
    \mathbb{E}\{\exp(t|\epsilon [2w_\beta(y)-1]y_{j}|)\}\leq \mathbb{E}\{\exp[2t|\epsilon y_{j}|]\}. 
\end{equation}

By the formulation of the model, we have $y_{ j}=z \beta^*_{j}+v_{j}$, where $z$ is a Rademacher random variable and $v_{j}\sim \mathcal{N}(0, \sigma^2)$. It is easy to see that $y_{j}$ is sub-Gaussian and 
\begin{equation}
    \|y_{j}\|_{\psi_2}=\|z\cdot \beta^*_j+v_{j}\|_{\psi_2}\leq C\cdot \sqrt{\|z\cdot \beta_j\|^2_{\psi_2}+\|v_{j}\|^2_{\psi_2}}\\ \leq C'\sqrt{|\beta_j^*|^2+\sigma^2},
\end{equation}
for some absolute constants $C, C'$, where the last inequality is due to the facts that $\|z_j\beta_j^*\|_{\psi_2}\leq |\beta_j^*|$ and $\|v_{i,j}\|_{\psi_2}\leq C''\sigma^2$ for some $C''>0$. 

Since $|\epsilon y_{j}|=|y_{j}|$, $\|\epsilon y_{j}\|_{\psi_2}=\|y_{j}\|_{\psi_2}$ and $\mathbb{E}(\epsilon y_{j})=0$, by Lemma 5.5 in \citep{vershynin2010introduction} we have that  for any $u'$ there exists a constant $C^{(4)}>0$ such that 
\begin{equation}
    \mathbb{E}\{\exp(u' \cdot \epsilon \cdot y_{j})\}\leq \exp(u'^2\cdot C^{(4)} \cdot (|\beta|_j^2+\sigma^2)). 
\end{equation}
Thus, for any $t>0$ we get
\begin{equation}
      \mathbb{E}\{\exp(2t\cdot |\epsilon \cdot y_{j}|)\}\leq 2\exp(t^2\cdot C^{(5)}\cdot (|\beta|_j^2+\sigma^2))
\end{equation}
for some constant $C^{(5)}$. 
Therefore, in total we have the following  for some constant $C^{(6)}>0$
\begin{multline}
         \mathbb{E}\{\exp(t|[\nabla_j q(\beta; \beta)-\mathbb{E}\nabla_j q (\beta; \beta)]|)\}\\ \leq \exp(t^2\cdot C^{(6)}\cdot (|\beta|_j^2+\sigma^2))\leq \exp(t^2\cdot C^{(6)}\cdot (\|\beta^*\|_\infty^2+\sigma^2)).
\end{multline}

Combining this with Lemma \ref{alemma:4} and the definition,  we know that  $\nabla_j q(\beta; \beta)$ is $O(\sqrt{\|\beta^*\|_\infty^2+\sigma^2})$-sub-exponential. 
\end{proof}

\begin{proof}[{\bf Proof of Lemma \ref{lemma:6}}]
Just as in the proof of Lemma \ref{lemma:4}, we will show that  $\nabla_j q(\beta; \beta)$ is sub-exponential instead.  
\begin{lemma}\label{alemma:10}
For each $\beta\in \mathcal{B}$, the $j$-the coordinate of $\nabla q(\beta; \beta)$  is $\xi$-sub-exponential with 
\begin{equation}\label{eq:55}
    \xi = C\max\{\|\beta^*\|^2_2+\sigma^2, 1, \sqrt{d}\|\beta^*\|_2\},
\end{equation}
where $C>0$ is some absolute constant. Also, for fixed $j\in [d]$, each $\nabla_j q_i(\beta;\beta)$, where $i\in [n]$, is independent with others.
\end{lemma}
\begin{proof}[Proof of Lemma \ref{alemma:10}]
From (\ref{aeq:2}) it is oblivious that for fixed $j\in [d]$, each $\nabla_j q_i(\beta;\beta)$, where $i\in [n]$, is independent with others. Next, we prove the property of sub-exponential. 

Note that $\mathbb{E}\nabla_j q(\beta; \beta) = \mathbb{E} 2w_\beta(x, y)y\cdot x_j-\beta_j$. Thus, we have 
\begin{equation}\label{aeq:56}
       \nabla_j q(\beta; \beta)-\mathbb{E}\nabla_j q(\beta; \beta)= \underbrace{2w_\beta(x, y)y x_{j}- \mathbb{E}[]2w_\beta(x,y)y x_j]}_{A}\\ +\underbrace{[xx^T\beta-\beta]_j}_{B}- \underbrace{yx_{j}}_{C}.
\end{equation}

For term A and any $t>0$, we have 
\begin{equation}\label{aeq:57}
    \mathbb{E}\{\exp(t|A|)\}\leq \mathbb{E}\{\exp[t|2\epsilon w_\beta(x, y)y x_{j}|]\}.
\end{equation}
Using Lemma \ref{alemma:8} on  $f(yx_{j})=yx_{j}$, $\mathcal{F}=f$, $\phi_i(v)= 2w_\beta(x,y)v$ and $\phi(v)=\exp(uv)$, we have 
\begin{equation}\label{aeq:58}
     \mathbb{E}\{\exp[t|2\epsilon w_\beta(x, y)y x_{j}|]\leq \mathbb{E}\{\exp[4t|\epsilon y x_{j}|]\}.
\end{equation}
Note that since $y=z\langle \beta^*,x\rangle +v$ and  $\|z\langle \beta^*,x\rangle\|_{\psi_2}=\|\langle \beta^*,x\rangle\|_{\psi_2}\leq C\|\beta^*\|_2$ and $\|v\|_{\psi_2}\leq C'\sigma$ for some constants $C, C'>0$, by Lemma \ref{alemma:6} we know that  there exists a constant $C''>0$ such that 
\begin{equation}\label{aeq:59}
    \|y\|_{\psi_2}\leq C''\sqrt{\|\beta^*\|_2^2+\sigma^2}.
\end{equation}
Thus, by Lemma \ref{alemma:5} we have 
\begin{equation}\label{aeq:60}
    \|yx_{j}\|_{\psi_1}\leq \max\{C''^2(\|\beta^*\|_2^2+\sigma^2), C'''\}\\ \leq C_4\max\{\|\beta^*\|_2^2+\sigma^2,1\}. 
\end{equation}
For term B, we have 
\begin{equation}
    \mathbb{E}\{\exp[t |B|]\}=\mathbb{E}\{\exp[t |\sum_{k=1}^d x_{j}x_k\beta_k-\beta_j|]\},
\end{equation}
where $x_j, x_k\sim \mathcal{N}(0, 1)$. Now, by Lemma \ref{alemma:5} we have  $\|x_jx_k\beta_k\|_{\psi_1}\leq |\beta_k|C^{(5)}$ for some constant $C^{(5)}>0$. Thus, we get $\|\sum_{k=1}^d x_jx_k\beta_k\|_{\psi_1}\leq C^{(5)}\|\beta\|_1$. 

Also, we know that $\|\beta\|_1\leq \sqrt{d}\|\beta\|_2$. Furthermore, we have $\|\beta\|_2\leq \|\beta^*\|_2+\|\beta^*-\beta\|_2\leq O(\|\beta^*\|_2)$, since  $\beta\in \mathcal{B}$ (by assumption). From Lemma \ref{alemma:5}, we get $\|B\|_{\psi_1}\leq C^{(6)}\sqrt{d}\|\beta^*\|_2$ with some constant $C^{(6)}>0$.

 Thus, we know that there exist some constants $C^{(7)}>0$ and $C^{(8)}>0$ such that 
 
 \begin{align*}
 	\|\nabla_j q(\beta; \beta)-\mathbb{E}\nabla_j q(\beta; \beta)\|_{\psi_1} &\leq C^{(7)}\max\{\|\beta^*\|_2^2+\sigma^2, 1\}+ C^{(8)}\sqrt{d}\|\beta^*\|_2 \\
 	&\leq C^{(9)}\max\{\|\beta^*\|_2^2+\sigma^2, 1, \sqrt{d}\|\beta^*\|_2\}.
 \end{align*}
This means that $\nabla_j q(\beta; \beta)$ is  $O(\max\{\|\beta^*\|_2^2+\sigma^2, 1, \sqrt{d}\|\beta^*\|_2\})$-sub-exponential. 

\end{proof}

\end{proof}

\begin{proof}[{\bf Proof of Lemma \ref{lemma:8}}]
Just as in the proof of Lemma \ref{lemma:4}, we will show that  $\nabla_j q(\beta; \beta)$ is sub-exponential instead.  
\begin{lemma}\label{alemma:11}
For each $\beta\in \mathcal{B}$ and $j\in [d]$ ,  $\nabla_j q(\beta; \beta)$ is $\xi$-sub-exponential with 
%%\vspace{-0.1in}
\begin{align}
    \xi&=C[(1+k)(1+kr)^2\sqrt{d}\|\beta^*\|_2 \notag +\max\{(1+kr)^2, \sigma^2+\|\beta^*\|_2^2\}]  \notag \\
    & = O(\sqrt{d}\|\beta^*\|_2+ \sigma^2+\|\beta^*\|_2^2)
\end{align}
for some constant $C>0$.   Also, for fixed $j\in [d]$, each $\nabla_j q_i(\beta;\beta)$, where $i\in [n]$, is independent with others.
\end{lemma}
\end{proof}
\begin{proof}[Proof of Lemma \ref{alemma:11}]
From (\ref{aeq:3}) it is oblivious that  for fixed $j\in [d]$, each $\nabla_j q_i(\beta;\beta)$, where $i\in [n]$, is independent with others. Next, we prove the property of sub-exponential.

For simplicity, we use notations  $\bar{m}=m_\beta(x^{\text{obs}}, y)$, $\bar{m}=\beta(x^{\text{obs}}, y)$, $\bar{K}=K_\beta(x_i^{\text{obs}}, y)$, and $\bar{K}=K_\beta(x^{\text{obs}}, y)$. Then, we have 
\begin{multline}
    \nabla q(\beta; \beta) -\mathbb{E}\nabla q(\beta; \beta) =\underbrace{ m_\beta(x^{\text{obs}}, y)y-\mathbb{E} [m_\beta(x^{\text{obs}}, y)y]}_{A} \\ +\overbrace{ \big(K_\beta(x^{\text{obs}}, y)-\mathbb{E}{K_\beta(x^{\text{obs}}, y)}\big)\beta}^{B}.
\end{multline}
For the $j$-th coordinate of $A$, we have 
\begin{equation}
    A_j=  \bar{m}_j y-\mathbb{E} [\bar{m}_j y].
\end{equation}
We note that $\bar{m}_j$ is a zero-mean sub-Gaussian random variable with $\|\bar{m}_j\|_{\psi_2}\leq C(1+kr)$ (see Lemma B.3 in \citep{wang2015high})
\begin{lemma}\label{alemma:12}
Under the assumption of Lemma 6, for each $j\in [d]$, $\bar{m}_j$ is sub-Gaussian with mean zero and $\|\bar{m}_j\|_{\psi_2}\leq C(1+kr)$.
\end{lemma}
Thus, by Lemma \ref{alemma:5} we have 
\begin{equation}
    \|\bar{m}_j y\|_{\psi_1}\leq C\max\{\|\bar{m}_j\|_{\psi_2}^2, \|y\|_{\psi_2}^2\} \\ 
    \leq C'\max\{(1+kr)^2, \sigma^2+\|\beta^*\|_2^2\},
\end{equation}
where the last inequality is due to the fact that $y=\langle \beta^*, x\rangle+v$. Thus,  $\|y\|_{\psi_2}^2\leq C_3(\|\langle \beta^*, x\rangle\|_{\psi_2}^2+\|v\|_{\psi_2}^2)$ for some $C_3$. 

For term B, we have 
\begin{equation}\label{aeq:65}
    \bar{K}_j= \underbrace{(1-z_{j})\beta_j}_{C}+\underbrace{\sum_{k=1}^d\bar{m}_j\bar{m}_k\beta_k}_{D}-\underbrace{\sum_{k=1}^d[(1-z_{j})\bar{m}_j][(1-z_{k})\bar{m}_k]\beta_k}_{E}. 
\end{equation}
For term C, we have the following (by Example 5.8 in \citep{vershynin2010introduction})
\begin{equation}
  \|  (1-z_{j})\beta_j\|_{\psi_2}\leq |\beta_j|\leq \|\beta\|_{\infty}\leq (1+k)\sqrt{s}\|\beta^*\|_2. 
\end{equation}
For term D, by Lemma \ref{alemma:12} and \ref{alemma:5} we have 
\begin{equation}
    \|\sum_{k=1}^d\bar{m}_j\bar{m}_k\beta_k\|_{\psi_1}  \leq \sum_{k=1}^d|\beta_k|\|\bar{m}_j\bar{m}_k\|_{\psi_1} 
    \leq \sum_{k=1}^d|\beta_k|C^2(1+kr)^2\leq C_4 (1+kr)^2\|\beta\|_1.
\end{equation}
Since $\beta\in \mathcal{B}$, we get $\|\beta\|_1\leq \sqrt{d}\|\beta\|_2\leq (1+k)\sqrt{d}\|\beta^*\|_2$. Thus, we have 
\begin{equation}
     \|\sum_{k=1}^d\bar{m}_j\bar{m}_k\beta_k\|_{\psi_1}\leq C_4\sqrt{s}(1+kr)^2\|\beta^*\|_2. 
\end{equation}
For term E, since $1-z\in [0,1]$, we have $\|(1-z_{j})\bar{m}_j\|_{\psi_2}\leq \|\bar{m}_j\|_{\psi_2}\leq C(1+kr)$. Hence, by Lemma \ref{alemma:5} we get  
\begin{align}
   \| \sum_{k=1}^d[(1-z_{j})\bar{m}_j][(1-z_{k})\bar{m}_k]\beta_k\|_{\psi_1}\notag 
   &\leq \sum_{k=1}^d|\beta_k| \|[(1-z_{j})\bar{m}_j][(1-z_{k})\bar{m}_k]\|_{\psi_1} \nonumber \\
   &\leq \sum_{k=1}^d|\beta_k| C(1+kr)^2\leq C_6(1+kr)^2\sqrt{s}\|\beta^*\|_2.
\end{align}
This gives us  
\begin{equation}
    \|\bar{K}_j\|_{\psi_1}\leq C_7\sqrt{s}(1+k)(1+kr)^2\|\beta^*\|_2.
\end{equation}
By Lemma \ref{alemma:4}, we get
\begin{align*}
        &\|\nabla_j q(\beta; \beta)- \mathbb{E}\nabla_j q(\beta; \beta)\|_{\psi_1} \leq 2\|\nabla_j q(\beta; \beta)\|_{\psi_1}\\ &\leq  C_8[(1+k)(1+kr)^2\sqrt{s}\|\beta^*\|_2+
     +\max\{(1+kr)^2, \sigma^2+\|\beta^*\|_2^2\}]. 
\end{align*}

\end{proof}
\bibliographystyle{spbasic}    
\bibliography{dpem}

\end{document}